\documentclass[journal]{IEEEtran}

\usepackage[english]{babel}

\usepackage{ifpdf}

\usepackage{cite} 
\usepackage{url}
\usepackage{hyperref}

\ifCLASSINFOpdf
	\usepackage[pdftex]{graphicx}
	\graphicspath{{./figures/}}
\else
	\usepackage[dvips]{graphicx}
	\graphicspath{./figures/}
\fi
\usepackage{color}
\usepackage{pgf, tikz, pgfplots}
\usetikzlibrary{shapes, arrows, automata, plotmarks}
\usetikzlibrary{calc,hobby,decorations}

\usepackage[cmex10]{amsmath}
\usepackage{amsfonts, amssymb, amsthm}
\usepackage{mathrsfs}

\usepackage{algorithm,algpseudocode}
	\algnewcommand{\LeftComment}[1]{\Statex \(\triangleright\) #1}

\usepackage{multirow}
\usepackage{rotating}
\usepackage{subcaption}
\usepackage{needspace}

\input{./latex_resources/mySymbol.sty}
\input{./latex_resources/pennColors.sty}

\def \hlam {\hat{\lam}}

\newtheorem{proposition}{\hspace{0pt}\bf Proposition}
\newtheorem{example}{\hspace{0pt}\bf Example}

\newtheorem{theorem}{\hspace{0pt}\bf Theorem}
\newtheorem{corollary}{\hspace{0pt}\bf Corollary}

\newtheorem{remark}{\hspace{0pt}\bf Remark}

\newtheorem{definition}{\hspace{0pt}\bf Definition}

\usepackage{enumitem}
\setlist[description]{topsep = 8pt, itemsep = 8pt, leftmargin = 24pt, labelwidth = 24pt, labelsep= 0pt}

\begin{document}
\title{Algebraic Neural Networks: Stability to Deformations}
\author{Alejandro~Parada-Mayorga
        and Alejandro Ribeiro
\thanks{Department of Electrical and Systems Engineering, University of Pennsylvania. Email: alejopm@seas.upenn.edu, aribeiro@seas.upenn.edu.}}

\markboth{IEEE Transactions on Signal Processing (accepted)}%
{Shell \MakeLowercase{\textit{et. al.}}: Bare Demo of IEEEtran.cls for Journals}
\maketitle
%
%
%
%


\begin{abstract}
We study algebraic neural networks (AlgNNs) with commutative algebras which unify diverse architectures such as Euclidean convolutional neural networks, graph neural networks, and group neural networks under the umbrella of algebraic signal processing. An AlgNN is a stacked layered information processing structure where each layer is conformed by an algebra, a vector space and a homomorphism between the algebra and the space of endomorphisms of the vector space. Signals are modeled as elements of the vector space and are processed by convolutional filters that are defined as the images of the elements of the algebra under the action of the homomorphism. We analyze stability of algebraic filters and AlgNNs to deformations of the homomorphism and derive conditions on filters that lead to Lipschitz stable operators. We conclude that stable algebraic filters have frequency responses -- defined as eigenvalue domain representations -- whose derivative is inversely proportional to the frequency -- defined as eigenvalue magnitudes. It follows that for a given level of discriminability, AlgNNs are more stable than algebraic filters, thereby explaining their better empirical performance. This same phenomenon has been proven for Euclidean convolutional neural networks and graph neural networks. Our analysis shows that this is a deep algebraic property shared by a number of architectures.
\end{abstract}

\begin{IEEEkeywords}
Algebraic Neural Networks, algebraic signal processing, representation theory of algebras, convolutional neural networks (CNNs), graph neural networks (GNNs), stability, Fr\'echet differentiability.
\end{IEEEkeywords}
\IEEEpeerreviewmaketitle
%
%

\section{Introduction}
The overwhelming empirical evidence that shows the goodness of using convolutional neural networks (CNNs) and graph neural networks (GNNs) in machine learning raises interest in finding reasons that explain their performance. In this context, stability analyses of the operators representing the neural networks play a central role, with insights reported for both CNNs~\cite{mallat_ginvscatt, bruna_iscn,bruna_groupinvrepconvnn, bietti2017_stabilitycnn} and GNNs~\cite{zou_stability, gamabruna_diffscattongraphs, fern2019stability}. Although independent, these results are similar in form and nature. This fact raises the question of whether they descend from a common notion of stability and motivates the search for a framework where these results can be unified. 

Stability of CNNs is rooted in the notion of Lipschitz-continuity to the action of diffeomorphisms introduced in~\cite{mallat_ginvscatt} for the analysis of translation-invariant operators acting on $L^{2}(\mathbb{R}^{n})$. Although initially derived for scattering transforms \cite{mallat_ginvscatt, bruna_iscn} stability results are readily extendable to the analysis of convolutional neural networks \cite{bruna_groupinvrepconvnn, bietti2017_stabilitycnn}. For GNNs the problem of formulating stability conditions has been considered in~\cite{zou_stability, gamabruna_diffscattongraphs, fern2019stability}. In~\cite{gamabruna_diffscattongraphs} the notion of stability on graphs is considered in depth pointing out that the generalization of the conditions stated in~\cite{mallat_ginvscatt, bruna_iscn} is not straightforward for non smooth, non Euclidean domains, and as a way to quantify stability in GNNs the notion of \textit{metric stability} is considered using a diffusion operator to measure the perturbations or changes in the graphs. In~\cite{fern2019stability} a related notion of stability is used to provide concrete results about the stability on GNNs.

However different, stability results for CNNs and GNNs have uncanny similarities. For instance, they both focus on signal perturbations that are modeled as deformations of the signal \emph{domain}, they both analyze the effect of perturbations in the frequency (spectral) domain, and they both conclude that (graph or Euclidean) convolutional filters have instabilities associated with high frequency components (large eigenvalues). In our search for underlying common principles we adopt the formalism of \textit{algebraic signal processing} (ASP) \cite{algSP0}. 

In general, signals are elements of a vector space $\ccalM$ which we \emph{could} process with any linear transformation in the algebra of endomorphisms of $\ccalM$. In practice, learning is facilitated if we introduce a suitable class of convolutional filters to restrict the type of transformations that are allowable. In ASP, convolutional \emph{algebraic filters} are defined as elements of a more restrictive algebra $\ccalA$ that are mapped into the algebra of endomorphisms of $\ccalM$ through a homomorphism $\rho$ (Section \ref{sec_alg_filters}). In the case of signals supported on a graph with $n$ nodes, the vector space is made up of vectors of length $n$ and the space of endomorphisms is made up of square matrices of matching dimension. Choosing the algebra of polynomials of a single variable $t$ and choosing a homomorphism that maps $t$ to the Laplacian matrix of the graph results in graph convolutional filters expressed as polynomials of the graph's Laplacian. This is the usual definition of a graph filter~\cite{segarragf}. In the case of signals in time the vector space is that of square summable sequences and convolutions can be written as polynomials on the time shift operator. This is the standard definition of convolution for discrete time signals~\cite[Ch. 2]{oppenheim2013discrete}.

In this paper we leverage algebraic filters to introduce \textit{algebraic neural networks} (AlgNNs) and study their stability to deformations of the signal domain. In particular, the main contributions of this paper are: 

\begin{description}

\item [{\bf (C1)}] The definition of AlgNNs as layered information processing architectures in which individual layers are made up of algebraic convolutional filters (Section \ref{sec_Algebraic_NNs}).  

\item [{\bf (C2)}] The introduction of perturbation models (Section \ref{sec:perturbandstability}) and stability properties (Section \ref{sec:stabilitytheorems}) that are analogous to the notions of perturbation and stability considered in \cite{mallat_ginvscatt, bruna_iscn, bruna_groupinvrepconvnn, bietti2017_stabilitycnn, zou_stability, gamabruna_diffscattongraphs, fern2019stability}.

\item[{\bf(C3)}] The proof of stability theorems for algebraic filters (Theorems~\ref{theorem:HvsFrechet}-\ref{theorem:uppboundDHmultg} and Corollaries~\ref{corollary:DHboundsprcases} and \ref{corollary:stabmultgenfilt}) out of which stability theorems for AlgNNs follow (Theorems~\ref{theorem:stabilityAlgNN0} and~\ref{theorem:stabilityAlgNN1}).

\end{description}

Our results are meaningful for algebras with a small number of generators (see Definition \ref{def_generators}). This includes discrete time convolutions but does not include continuous time convolutions. Thus, we do not recover results in \cite{mallat_ginvscatt, bruna_iscn, bruna_groupinvrepconvnn, bietti2017_stabilitycnn} as particular cases of our theorems. Rather our conclusions for discrete time CNNs are analogous to the conclusions that \cite{mallat_ginvscatt, bruna_iscn, bruna_groupinvrepconvnn, bietti2017_stabilitycnn} reach for continuous time CNNs. This relatively minor technicality aside, our stability results for AlgNNs recover existing results for CNNs and GNNs (Section \ref{sec:stbcases}). Our results also extend to other types of convolutional architectures like multidimensional CNNs -- as used in image processing --, group neural networks and graphon neural networks. They also apply to as of yet unknown convolutional architectures. Indeed, the universality of stability properties is among the fundamental insights of this paper:

\begin{description}

\item [{\bf(I1)}] The stability properties of convolutional filters and neural networks are universal. 

\end{description}

This holds because the stability properties of convolutional architectures can be expressed in terms of the algebraic laws that govern the signal model in each layer as encoded in the algebra $\ccalA$. To explain this statement we mention that representations of algebras admit spectral decompositions (Section \ref{sec:spectopt}). These decompositions permit the definition of Fourier transforms of signals and, more germane to our discussions, frequency representations of algebraic filters. These representations are defined as isomorphisms that map generators of the algebra $\ccalA$ into scalar variables (Definitions \ref{def_freq_rep} and \ref{def_freq_rep_multiple_generators}). As such, frequency representations are functions with as many variables as generators as needed to generate the algebra. In cases of interest, this just means a function of a few variables, each of which we call a frequency. Remarkably, frequency representations depend on the choice of algebra but do not depend on the vector space where signals live. Ultimately, this is the reason why universal stability results are possible and it further leads to the following insights:

\begin{description}


\item[{\bf(I2)}] Although perturbations are considered on filter operators, stability is determined by restrictions to certain subsets of the algebra. These restrictions are expressed in terms of filters' frequency representations (Section\ref{sec:proofofTheorems}).

\item[{\bf(I3)}] Stability requires filter frequency responses that are flat for large values of the frequency variables. This limits the discriminability of algebraic filters (Section \ref{sec:discussion}).

\item[{\bf(I4)}] AlgNNs improve the stability vs discriminability tradeoff of algebraic filters because pointwise nonlinearities move signal energy towards lower frequencies where signals can be better discriminated by filters with a given level of stability (Section \ref{sec:discussion}). 

\end{description}

Insights (I3) and (I4) are the summary messages of this paper. We know from \cite{mallat_ginvscatt, bruna_iscn, bruna_groupinvrepconvnn, bietti2017_stabilitycnn} that (I3) and (I4) explain the increased performance of CNNs relative to convolutional filters. We know from \cite{zou_stability, gamabruna_diffscattongraphs, fern2019stability} that (I3) and (I4) explain the increased performance of GNNs relative to graph filters. As per (I1) we show here that the reason why these analogous properties hold is the shared algebraic structure of CNNs and GNNs. The universality of the result implies that (I3) and (I4) also explain performance improvements of CNNs with multidimensional inputs relative to multidimensional Euclidean convolutional filters, group neural networks relative to group filters, and graphon neural networks relative to graphon filters among any number of known and unknown convolutional information processing architectures. Our results are limited to commutative algebras with a small number of representers. Further work is needed to extend our results to these more general signal models (Section~\ref{sec:conclusions}).

%
%
%


%
\begin{figure}
\centering

\def \scale {1.4}
\def \unit  { \scale cm}

\tikzstyle{set} = [ rectangle,
                    rounded corners = 0.4*\unit,
                    fill=blue!15,
                    inner sep=0pt,
                    draw,
                    anchor = center ]

\tikzstyle{vector space} = [ set,
                             minimum width  = 2.0*\unit,
                             minimum height = 1.6*\unit]

\tikzstyle{endomorphisms} = [ vector space,
                              fill=black!15,
                              minimum height = 1.6*\unit]

\tikzstyle{algebra} = [ endomorphisms,
                        fill=red!15,
                        minimum width = 1.6*\unit]

\tikzstyle{dot} = [ circle,
                    minimum width  = 0.1*\unit,
                    fill=black,
                    inner sep=0pt,
                    draw,
                    anchor = center ]

{\small

\begin{tikzpicture}[-stealth, draw = black!99, scale = \scale]


   \path (0,0) node [vector space] (M) {};
   \path (M.south) ++ (0, 0.2) node [above] {$\ccalM$};
   
   \path (M) ++ (-0.5, +0.3) node [dot] (x) {};
   \path (x.south) node [below] {$\bbx$};             
   
   \path (M) ++ (+0.5, +0.3) node [dot] (ex) {};
   \path (ex.south) node [below] {$\rho(a)\bbx$};

   \path (M.north) ++ (0, 0.5) 
         node [endomorphisms, anchor=south] (End) {};
   \path (End.north) ++ (0.0, -0.2) node [below] {$\text{End}(\ccalM)$};   

   \path (End.center) ++ (0.0, -0.33) node [dot] (e) {};      
   \path (e) node [above right] {$\rho(a)$};    
   
   \path (End.west) + (-0.1,0.15) coordinate (c1);
   \path (End.east) + (+0.1,0.15) coordinate (c2);   
   \path [draw, -stealth] (x) .. controls (c1) and (c2) .. (ex);   

   \path (End.north west) ++ (-1.0, 0) 
         node [algebra, anchor = north east] (A) {};   
   \path (A.north) ++ (0,-0.2) node [below] {$\ccalA$};   
   
   \path (e-| A.center) node [dot] (a) {};      
   \path (a) node [below left] {$a$};    

   \path (End.west) + (-0.1,0.25) coordinate (c1);
   \path (End.east) + (+0.1,0.25) coordinate (c2);   
   \path [draw, -stealth, line width = 1.0, mygreen] 
         (a) edge [bend left] node [above] {$\rho~~~$} (e);

\end{tikzpicture}

}
\caption{Algebraic Signal Processing (ASP) Model. An ASP model is made up of a vector space $\ccalM$ where signals $\bbx$ live and an Algebra $\ccalA$ where filters $a$ live. The homomorphism $\rho$ ties the algebra and the vector space together by mapping the filter $a$ to the linear function $\rho(a)$ in the space of endomorphisms of $\ccalM$. The Algebra restricts the set of linear processing maps that can be applied to signals $\bbx$.}
\label{fig_1}
\end{figure}
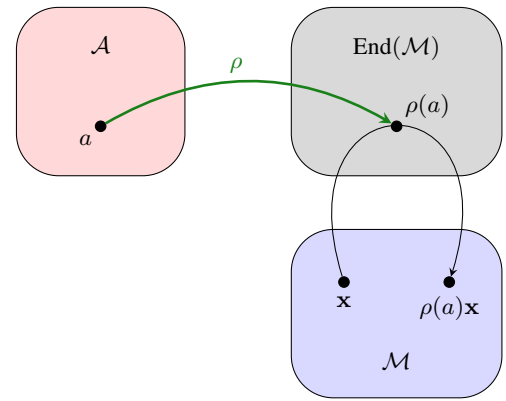


\section{Algebraic Filters} \label{sec_alg_filters}

Algebraic signal processing (ASP) provides a framework for understanding and generalizing traditional signal processing exploiting the representation theory of algebras~\cite{algSP0, algSP1, algSP2, algSP3}; see Figure \ref{fig_1}. In ASP, a signal model is defined as the triple
\begin{equation}\label{eqn_ASP_signal_model}
   (\ccalA,\ccalM,\rho),
\end{equation}   
in which $\ccalA$ is an associative algebra with unity, $\ccalM$ is a vector space with inner product, and $\rho:\ccalA\to\text{End}(\ccalM)$ is a homomorphism between the algebra $\ccalA$ and the set of endomorphisms of the vector space $\ccalM$. The elements in \eqref{eqn_ASP_signal_model} are tied together by the notion of a representation which we formally define next. 


\begin{definition}[Representation] \label{def_representation} A representation $(\ccalM,\rho)$ of the associative algebra $\ccalA$ is a vector space $\ccalM$ equipped with a homomorphism $\rho:\ccalA\to\text{End}(\ccalM)$, i.e., a linear map preserving multiplication and unit. 
\end{definition}


In an ASP model, signals  are elements of the vector space $\ccalM$, and filters are elements of the algebra $\ccalA$. Thus, the vector space $\ccalM$ determines the objects of interest and the algebra $\ccalA$ the rules of the operations that define a filter. The homomorphism $\rho$ translates the abstract operators $a\in\ccalA$ into concrete operators $\rho(a)$ that act on signals $\bbx$ to produce filter outputs 
\begin{equation}\label{eqn_ASP_filter_outputs}
   \bby = \rho(a) \bbx.
\end{equation}   
The algebraic filters in \eqref{eqn_ASP_filter_outputs} generalize the convolutional processing of time signals -- see Example \ref{ex_DTSP}. Our goal in this paper is to use them to generalize convolutional neural networks (Section \ref{sec_Algebraic_NNs}) and to study their fundamental stability properties (Section \ref{sec:perturbandstability}). Generators, which we formally define next, are important for the latter goal.


\begin{definition}[Generators] \label{def_generators} For an associative algebra with unity $\ccalA$ we say the set $\ccalG\subseteq\ccalA$ generates $\ccalA$ if all $a\in\ccalA$ can be represented as polynomial functions of the elements of $\ccalG$. We say elements $g\in\ccalG$ are generators of $\ccalA$ and we denote as $a=p_{\ccalA}(\ccalG)$ the polynomial that generates $a$.
\end{definition}


Definition \ref{def_generators} states that elements $a\in\ccalA$ can be built from the generating set as polynomials using the operations of the algebra. 

Given that representations connect the algebra $\ccalA$ to signals $\bbx$ as per Definition \ref{def_representation}, the representation $\rho(g)$ of a generator $g\in\ccalG$ will be of interest. In the context of ASP, these representations are called shift operators as we formally define next.


\begin{definition}[Shift Operators]\label{def_shift_operators} 
Let $(\ccalM,\rho)$ be a representation of the algebra $\ccalA$. Then, if $\ccalG\subseteq\ccalA$ is a generator set of $\ccalA$, the operators $\bbS = \rho(g)$ with $g\in\ccalG$ are called shift operators. The set of all admissible shift operators is denoted by $\ccalS$.
\end{definition}


Given that elements $a$ of the algebra are generated from elements $g$ of the generating set, it follows that filters $\rho(a)$ are generated from the set of shift operators $\bbS=\rho(g)$. In fact, if we have that $a=p_{\ccalA}(\ccalG)$ is the polynomial that generates $a$, the fact that $\rho$ is a homomorphism that preserves operations implies that the filter's instantiation $\rho(a)$ can be written as
\begin{equation}\label{eqn_filter_representation}
   \rho(a) = p_\ccalM \big(\rho(\ccalG)\big) 
           = p_\ccalM\big(\ccalS\big)
           = p\big(\ccalS\big),
\end{equation}
where the subindex $\ccalM$ signifies that the operations in \eqref{eqn_filter_representation} are those of the vector space $\ccalM$ -- in contrast to the polynomial $a=p_{\ccalA}(\ccalG)$ whose operations are those of the algebra $\ccalA$. In the last equality and for the rest of the paper we drop the subindices in the polynomials to simplify notation as it is generally understood from context to which set the independent variable of $p$ belongs.

We restrict attention to commutative algebras $\ccalA$. We also restrict the field $\mbF$ on which $\ccalM$ and $\ccalA$ are supported to be algebraically closed. If this doesn't hold our results apply to the corresponding algebraic extension. We point out that although not a formal requirement, our results are meaningful when the algebra $\ccalA$ has a set of generators with a small number of elements.

We present examples to clarify ideas. Readers may skip ahead since they are not needed to understand the rest of the paper.


\newcounter{dtsp}
\setcounter{dtsp}{\theexample} 
\setcounter{example}{\thedtsp} 

\begin{example}[{\bf Discrete Time Signal Processing}]\label{ex_DTSP}\normalfont Let $\ccalM = \ell^2$ be the space of square summable sequences $\bbx = \{x_n\}_{n\in\mbZ}$ and $\ccalA$ the algebra of polynomials generated by $g=t$ with elements $a = \sum_{k=0}^{K-1} h_k t^k$. Consider the time shift operator $S$ such that $S\bbx$ is the sequence with entries $(S\bbx)_n = x_{n-1}$. Define the homomorphism $\rho$ in which the generator $g=t$ is mapped to $\rho(g) = \rho(t) = S$. Then, the filter $a = \sum_{k=0}^{K-1} h_k t^k$ is mapped to the endomorphism [cf.\eqref{eqn_filter_representation}]
\begin{equation}\label{eqn_DTSP_filter_outputs0}
   \rho\left(\sum_{k=0}^{K-1} h_k t^k\right)
         \,=\, \sum_{k=0}^{K-1} h_k S^k.
\end{equation}   
Observe how the abstract polynomial $a = \sum_{k=0}^{K-1} h_k t^k$ is mapped to the polynomial $\rho(a)=\sum_{k=0}^{K-1} h_k S^k$. The latter is a concrete linear operator in the space of square summable sequences that we can use to process sequences $\bbx$ as per \eqref{eqn_ASP_filter_outputs}. This leads to the input output relationship
\begin{equation}\label{eqn_DTSP_filter_outputs}
   \bby  \,=\, \left(\sum_{k=0}^{K-1} h_k S^k\right) \bbx
         \,=\, \sum_{k=0}^{K-1} h_k S^k \bbx.
\end{equation}  
Since $(S\bbx)_n = x_{n-1}$ it follows that $(S^k\bbx)_n = x_{n-k}$ and that \eqref{eqn_DTSP_filter_outputs} representats a discrete time convolutional filter~\cite[Ch. 2]{oppenheim2013discrete}.
\end{example}


\newcounter{gsp}
\setcounter{gsp}{\theexample} 
\setcounter{example}{\thegsp} 

\begin{example}[{\bf Graph Signal Processing}]\label{ex_GSP}\normalfont We retain the algebra of polynomials as in Example \ref{ex_DTSP} but we change the space of signals to the set of complex vectors with $N$ entries, $\ccalM =\mbC^N$. We interpret components $x_n$ of $\bbx\in\ccalM = \mbC^N$ as being associated with nodes of a graph with matrix representation $\bbS\in\mbC^{N\times N}$. We consider the homomorphism $\rho$ in which the generator $g=t$ is mapped to the matrix representation $\bbS$ of the graph. Having chosen $\rho(t)=\bbS$ we
use \eqref{eqn_filter_representation} to write
\begin{equation}\label{eqn_GSP_filter_outputs0}
     \rho\left(\sum_{k=0}^{K-1} h_k t^k\right)
         \,=\, \sum_{k=0}^{K-1} h_k \bbS^k.
\end{equation}
Analogously to \eqref{eqn_DTSP_filter_outputs0}, the abstract polynomial $\sum_{k=0}^{K-1} h_k t^k$ is mapped to the concrete polynomial $\sum_{k=0}^{K-1} h_k \bbS^k$. The latter is an $N\times N$ matrix that can be applied to signals $\bbx$ to produce outputs
\begin{equation}\label{eqn_GSP_filter_outputs}
   \bby  \,=\, \left(\sum_{k=0}^{K-1} h_k \bbS^k\right) \bbx
         \,=\, \sum_{k=0}^{K-1} h_k \bbS^k \bbx.
\end{equation}
This is a representation of the graph convolutional filters used in graph signal processing (GSP)~\cite{gsp_graphfilters,ortega_gsp}. Observe that \eqref{eqn_GSP_filter_outputs} and \eqref{eqn_DTSP_filter_outputs} are similar but represent different operations. In \eqref{eqn_GSP_filter_outputs} $\bbx$ is a vector and $\bbS^k$ a matrix power. In \eqref{eqn_DTSP_filter_outputs} $\bbx$ is a sequence and $S^k$ is the composition of the time shift operator $S$. Their similarity arises from the common use of the algebra of polynomials. Their differences are because we use different vector spaces $\ccalM$ and different homomorphisms $\rho$.
\end{example}


\newcounter{dsp}
\setcounter{dsp}{\theexample} 
\setcounter{example}{\thedsp} 
\begin{example}[{\bf Discrete Signal Processing}] \label{ex_DSP}\normalfont We consider discrete time signals of length $N$ with circular convolutions. To do that we consider the vector space $\ccalM=\mbC^N$ and the algebra of polynomials modulo $t^{N}-1$. I.e., filters $a\in\ccalA$ are polynomials $a = \sum_{k=0}^{K-1} h_k t^k$ but we must have $K\leq N$ and monomial products use the rule $t^k = t^{k \mod N}$. We consider the directed cyclic matrix $\bbC$ with exactly $N$ nonzero entries $C_{n,m}=1$ for $m = (n-1) \mod N$. This matrix is such that $(\bbC \bbx)_n = x_{(n-1)\mod N}$. Using the homomorphism in which we map the generator $g=t$ to $\rho(g) = \rho(t) = \bbC$, filter instantiations take the form
\begin{equation}\label{eqn_DSP_filter_outputs0}
   \rho\left(\sum_{k=0}^{K-1} h_k t^k\right)
         \,=\, \sum_{k=0}^{K-1} h_k \bbC^k.
\end{equation}
The filter instantiation $\rho(a)=\sum_{k=0}^{K-1} h_k \bbC^k$ leads to the input output relationship
\begin{equation}\label{eqn_DSP_filter_outputs}
   \bby  \,=\, \left(\sum_{k=0}^{K-1} h_k \bbC^k\right) \bbx
         \,=\, \sum_{k=0}^{K-1} h_k \bbC^k \bbx.
\end{equation}
Since $(\bbC \bbx)_n = x_{(n-1)\mod N}$ we have that $(\bbC^k \bbx)_n = x_{(n-k)\mod N}$. Thus, \eqref{eqn_DSP_filter_outputs} is equivalent to the usual definition of circular convolutions~\cite[Ch. 8]{oppenheim2013discrete}. Observe that the homomorphism $\rho(a) = \sum_{k=0}^{K-1} h_k C^k$ is indeed a homomorphism because the cyclic matrix $\bbC$ satisfies $\bbC^k = \bbC^{k \mod N}$. This example illustrates that in some situations the choice of algebra and the choice of homomorphism are tied. \end{example}


As is clear from Examples \ref{ex_DTSP}-\ref{ex_DSP}, the effect of the operator $\rho(a)$ on a given signal $\bbx$ is determined by two factors: The filter $a\in\ccalA$ and the homomorphism $\rho$. The filter $a\in\ccalA$ indicates the \textit{laws and rules} to be used to manipulate the signal $\bbx$ and $\rho$ provides a \textit{physical realization} of the filter $a$ on the space $\ccalM$ to which $\bbx$ belongs. For instance, in these three examples the filter $a=1 + 2t$ indicates that the signal is to be added to a transformed version of the signal scaled by coefficient $2$. The homomorphism $\rho$ in Example \ref{ex_DTSP} dictates that the physical implementation of this transformation is a time shift. The homomorphism $\rho$ in Example \ref{ex_GSP} defines a transformation as a multiplication by $\bbS$ and in Example \ref{ex_DSP} the homomorphism entails a cyclic shift. We remark that in order to specify the physical effect of a filter it is always sufficient to specify the physical effect of the generators. In all three examples, the generator of the algebra is $g=t$. The respective effects of an arbitrary filter $a$ are determined once we specify that $\rho(t) = S$ in Example \ref{ex_DTSP}, $\rho(t) = \bbS$ in Example \ref{ex_GSP}, or $\rho(t) = C$ in Example \ref{ex_DSP}. 

The flexibility in the choice of algebra and homomorphism allows for a rich variety of signal processing frameworks. We highlight this richness with three more examples.


\newcounter{imgpr}
\setcounter{imgpr}{\theexample} 
\setcounter{example}{\theimgpr} 

\begin{example}[{\bf Image Processing}]\label{ex_imgpr}\normalfont 
We represent images as square summable sequences with two indexes, $\bbx=\{x_{n,m}\}_{n,m\in\mbZ}$. We define the horizontal translation operator $S_{\text{H}}$ such that $(S_{\text{H}}\bbx)_{mn} = x_{m,n-1}$ and the vertical translation operator $S_{\text{V}}$ such that $(S_{\text{V}}\bbx)_{mn} = x_{m-1,n}$. Filters to process images are elements of the algebra of polynomials of two variables $a = \sum_{k_{1}=0}^{K_{1}-1}\sum_{k_2=0}^{K_2-1} h_{k_1 k_2} t_{1}^{k_1} t_{2}^{k_2}$. This algebra has two generators $g_1=t_1$ and $g_2=t_2$ that we map to $\rho(t_1) = S_{\text{H}}$ and $\rho(t_2) = S_{\text{V}}$. This generator mapping defines the homomorphism $\rho$ in which filters are mapped to instances
\begin{align}\label{eqn_imgpr_filter_outputs}
   \rho\left(\sum_{k_{1}=0}^{K_{1}-1}\sum_{k_2=0}^{K_2-1} h_{k_1,k_2} t_{1}^{k_1} t_{2}^{k_2}\right)
      = \sum_{k_{1}=0}^{K_{1}-1}\sum_{k_2=0}^{K_2-1} h_{k_1 k_2} S_{\text{H}}^{k_1}S_{\text{V}}^{k_{2}}.
\end{align}
The composed operator $S_{\text{H}}^{k_1}S_{\text{V}}^{k_{2}}$ applied to a sequence $\bbx$ translates horizontal and vertical indexes by $k_1$ and $k_2$ indexes. Thus, applying the operator in the right hand side of \eqref{eqn_imgpr_filter_outputs} to an image $\bbx$ is equivalent to convolving the image with an 2-dimensional convolutional filter with coefficients $h_{k_1k_2}$.

\end{example}


\newcounter{groupsp}
\setcounter{groupsp}{\theexample} 
\setcounter{example}{\thegroupsp} 

\begin{example}[{\bf Signal Processing on Groups}]\normalfont \label{ex_groupsp} Let $\ccalM=\{\bbx: G\to\mathbb{C} \}=\{ \sum_{g\in G}\bbx(g)g\}$ be the set of functions defined on the group $G$ with values in $\mathbb{C}$ and $\ccalA=\ccalM$ the \textit{group algebra}. The homomorphism is given by $\rho(\bba)=L_{\bba}$, with $L_{\bba}\bbb=\bba\bbb$. Then, the action of $\rho$ on elements of $\ccalM$ is given by
\begin{equation}\label{eqn_group_filters}
\rho\left(\sum_{g\in G}\bba(g)g\right)\bbx
=
\sum_{g\in G}\bba(g)g\bbx=\sum_{g\in G}\sum_{h\in G}\bba(g)\bbx(h)gh,
\end{equation}
and making $u=gh$ we have that the filtering in~\eqref{eqn_ASP_filter_outputs} takes the form
\begin{equation}\label{eqn_group_convolution}
   \sum_{g,h\in G} \bba(g)\bbx(h)gh 
             =\sum_{u,h\in G}\bba(uh^{-1})\bbx(h)u.
\end{equation}
This is the standard representation of convolution of signals on groups~\cite{steinbergrepg,terrasFG,fulton1991representation}. We point out that \eqref{eqn_group_filters} and \eqref{eqn_group_convolution} hold for any group but that not all group algebras are commutative. Results in Section \ref{sec:stabilitytheorems} apply only when the group algebra is commutative.
\end{example}


\newcounter{wsp}
\setcounter{wsp}{\theexample} 
\setcounter{example}{\thewsp} 

\begin{example}[{\bf Graphon Signal Processing}]\normalfont \label{ex_wsp} A graphon $W(u,v):[0,1]^{2}\to [0,1]$ is a bounded symmetric measurable function and graphon signals are square summable functions $\bbx(u): [0,1] \to \mbC$. Graphons are intended to represent dense limits of graphs \cite{lovaz2012large, segarragraphon, ruiz2020graphon, alej2020graphon} and graphon signals dense limits of graph signals \cite{ruiz2020graphon, alej2020graphon}. To define graphon convolutional filters consider the algebra of polynomials of a single variable and define the graphon shift operator as
\begin{equation}\label{eqn_graphon_shift}
   \left(\ccalW\bbx\right)(u)=\int_{0}^{1}W(u,v)\bbx(v)dv,   
\end{equation}
Filters $a = \sum_{k=0}^{K-1} h_k t^k$ are mapped according to the homomorphism defined by the generator map $\rho(t) = \ccalW$ resulting on filters that define the input-output relationship
\begin{equation}\label{eqn_WSP_filter_outputs}
   \bby \,=\, \rho\left(\sum_{k=0}^{K-1} h_k t^k\right)\bbx
        \,=\, \sum_{k=0}^{K-1} h_k \ccalW^k \bbx.
\end{equation}
This is the same definition of graphon convolutional filters introduced in \cite{ruiz2020graphon} where they are shown to be limit objects of graph filters. \end{example}


The choice of $\ccalA$ and $\rho$ provides means to leverage our knowledge of the signal's domain in its processing. The convolutional filters in \eqref{eqn_DTSP_filter_outputs} leverage the shift invariance of time signals and the filters in \eqref{eqn_DSP_filter_outputs} the cyclic invariance of periodic signals. The group convolutional filters in \eqref{eqn_group_convolution} generalize shift invariance with respect to an arbitrary group action. The graph convolutional filters in \eqref{eqn_GSP_filter_outputs} engender signal processing that is independent of node labeling \cite{gamagnns} and the graphon filters in Example \ref{ex_wsp} a generalization of this notion to dense domains~\cite{ruiz2020graphon}. Leveraging this structure is instrumental in achieving scalable information processing. In the following section we explain how neural network architectures combine algebraic filters as defined in \eqref{eqn_ASP_filter_outputs} with pointwise nonlinearities to attain signal processing that inherits the invariance properties of the respective algebraic filters. 


\begin{remark}[{\bf Shift Equivariance of Algebraic Filters}]\label{rmk_shift_invariance}\normalfont In restricting the linear transformations that can be applied to signals, the Algebra $\ccalA$ reduces the complexity of the learning space. It is easier to learn coefficients of a filter than it is to learn entries of an arbitrary linear transform. In this statement, the \emph{equivariance} of algebraic filters to applications of shift operators is important. Equivariance to applications of the shift operator means that applying a shift operator at the input of an algebraic filter is equivalent to applying the same shift operator at the output. Namely, that for all filters $a=p(\ccalG)$ and shift operators $\bbS\in\ccalS$ we have
\begin{align}\label{eqn_equivariance_shift}
   \bbS p(\ccalS) \bbx = p(\ccalS) \bbS \bbx.
\end{align}
This holds true for any commutative algebra. Equivariance to application of the shift operator is important in discrete time signal processing, discrete signal processing, image processing, and group signal processing. It implies that algebraic filters are equivariant to time shifts, cyclic shifts, translations, and actions of the group, respectively. \end{remark}


\begin{remark}[{\bf Permutation Equivariance of Algebraic Filters}]\label{rmk_perm_invariance}\normalfont In learning with algebraic filters equivariance to permutations is also important. Equivariance to permutations means that a consistent permutation of the signal and the shift operator results in a consistent permutation of the output of the filter. Formally, let $\bbP\in\text{End}(\ccalM)$ be a permutation operator with adjoint $\bbP^T = \text{adj}(\bbP)$. A permutation of the signal $\bbx$ is $\tbx = \bbP^T\bbx$ and a consistent permutation of the shift operator $\bbS$ is the endomorphism $\tbS = \bbP^T \bbS \bbP$. If we let  $\tilde{\ccalS}$ denote the set of permuted shift operators we must have, 
\begin{align}\label{eqn_equivariance_permutation}
   p\big( \tilde{\ccalS} \big) \tbx = \bbP^T \big( p(\ccalS) \bbx\big).
\end{align}
I.e., the output of processing a permuted signal $\tbx$ with the filter instantiated on the set of permuted shift operators $\tilde{\ccalS}$ is equivalent to a permutation of the output signal that results from processing $\bbx$ with the filter instantiated on the shift operator $\bbS$. This is a consequence of the fact that the adjoint permutation $\bbP^T$ is the inverse of the permutation $\bbP$. Equivariance to permutations is important in graph signal processing and graphon signal processing. It implies processing that is independent of labeling. \end{remark}


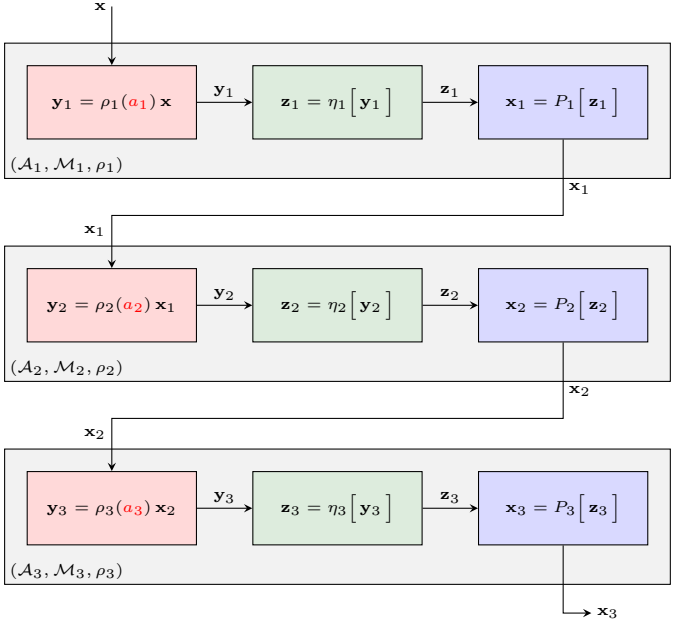
\begin{figure}
\centering

\definecolor{my_blue}{rgb}{0.0314, 0.3569, 1.0000}

\def \myfactor {0.75}
\def \unit  {\myfactor cm}

\tikzstyle{block} = [ rectangle,
                      minimum width = \unit,
                      minimum height = \unit,
                      fill = my_blue,
                      draw = black,
                      text = black]

\tikzstyle{filter} = [block,
                      minimum width  = 3.0*\unit,
                      minimum height = 1.3*\unit,
                      fill=red!15]

\tikzstyle{nonlinearity} = [ filter,
                             minimum width  = 3.0*\unit,
                             fill = mygreen!15]

\tikzstyle{pooling} = [ filter,
                             minimum width  = 3.0*\unit,
                             fill = blue!15]

\def \deltainput     {( 0.0,-1.7)}
\def \deltaoutput    {( 0.0,-1.2)}
\def \deltalayer     {3.6}
\def \deltaconnector {1.45}
\def \deltasigma     {( 4, 0.0)}

\def \one   {$\displaystyle{\mathbf{y}_{1}  = \rho_{1}(\red{a_1})\,\mathbf{x}}$}
\def \two   {$\displaystyle{\mathbf{y}_2  =  \rho_{2}(\red{a_2})\,\mathbf{x}_{1}}$}
\def \three {$\displaystyle{\mathbf{y}_3  = \rho_{3}(\red{a_3})\,\mathbf{x}_{2}}$}
\def \sigmaone   {$\displaystyle{\mathbf{z}_{1} = {\eta_{1}} \Big[\, \mathbf{y}_1 \, \Big]}$}
\def \sigmatwo   {$\displaystyle{\mathbf{z}_{2} = {\eta_{2}} \Big[\, \mathbf{y}_2 \, \Big]}$}
\def \sigmathree {$\displaystyle{\mathbf{z}_{3} = {\eta_{3}} \Big[\, \mathbf{y}_3 \, \Big]}$}
\def \proyone   {$\displaystyle{\mathbf{x}_{1} = {P_{1}} \Big[\, \mathbf{z}_1 \, \Big]}$}
\def \proytwo   {$\displaystyle{\mathbf{x}_{2} = {P_{2}} \Big[\, \mathbf{z}_2 \, \Big]}$}
\def \proythree   {$\displaystyle{\mathbf{x}_{3} = {P_{3}} \Big[\, \mathbf{z}_3 \, \Big]}$}

%
{\fontsize{6}{6}\selectfont\begin{tikzpicture}[scale = \myfactor]

  \pgfdeclarelayer{bg}     
  \pgfsetlayers{bg,main}   

  \node (input) [rectangle, minimum width = 0.1*\unit] {$\mathbf{x}$};
  \path (input.east)      ++ \deltainput node [filter]       (L1 Filter1) {\one};
  \path (L1 Filter1) ++ \deltasigma node [nonlinearity] (L1 F1)      {\sigmaone};
  \path (L1 F1) ++ \deltasigma node [pooling] (L1 F2)      {\proyone};
  \path[draw, -stealth] (L1 Filter1.east) -- node [above] {$\mathbf{y}_1$} (L1 F1.west);
  \path[draw, -stealth] (L1 F1.east) -- node [above] {$\mathbf{z}_1$} (L1 F2.west);

  \path (L1 Filter1) ++ (0,-\deltalayer) node [filter]       (L2 Filter1) {\two};
  \path (L2 Filter1) ++ \deltasigma      node [nonlinearity] (L2 F1)      {\sigmatwo};
  \path (L2 F1) ++ \deltasigma node [pooling] (L2 F2)      {\proytwo};
  \path[draw, -stealth] (L2 Filter1.east) --  node [above] {$\mathbf{y}_2$} (L2 F1.west);
  \path[draw, -stealth] (L2 F1.east) -- node [above] {$\mathbf{z}_2$} (L2 F2.west);  
  
  \path (L2 Filter1) ++ (0,-\deltalayer) node [filter]       (L3 Filter1) {\three};
  \path (L3 Filter1) ++ \deltasigma      node [nonlinearity] (L3 F1)      {\sigmathree};
  \path (L3 F1) ++ \deltasigma node [pooling] (L3 F2)      {\proythree};
  \path[draw, -stealth] (L3 Filter1.east) --  node [above] {$\mathbf{y}_3$} (L3 F1.west);
  \path[draw, -stealth] (L3 F1.east) -- node [above] {$\mathbf{z}_3$} (L3 F2.west); 

  \path[draw, -stealth] (input.east) -- (L1 Filter1.north);
  \path (L1 F2.south) ++ (0,-\deltaconnector) node [] (aux1) {};
  \path[draw, -stealth] (L1 F2.south) -- node [below right] {$\mathbf{x}_1$} (aux1.north) 
                                      --                         (aux1.north -| L2 Filter1.north) 
                                      -- node [above left]  {$\mathbf{x}_1$} (L2 Filter1.north);
  \path (L2 F2.south) ++ (0,-\deltaconnector) node [] (aux1) {};
  \path[draw, -stealth] (L2 F2.south) -- node [below right] {$\mathbf{x}_2$} (aux1.north) 
                                      --                         (aux1.north -| L2 Filter1.north) 
                                      -- node [above left]  {$\mathbf{x}_2$} (L3 Filter1.north);
  \path[draw, -stealth] (L3 F2.south) -- ++ \deltaoutput -- ++ (0.5, 0) 
                        node [right]{$\mathbf{x}_3$};

  \begin{pgfonlayer}{bg} 
      \path (L1 Filter1.west |- L1 F1.south) ++ (-0.4,-0.7)
           node [filter, anchor = south west,
                 fill = black!5, 
                 minimum width  = 11.8*\unit,
                 minimum height = 2.4*\unit,] 
        (layer)
        {}; 
       \path (layer.south west) ++ (0.0,0.0) node [above right] {$(\ccalA_1, \ccalM_1, \rho_1)$};
      \path (L1 Filter1.west |- L2 F1.south) ++ (-0.4,-0.7)
           node [filter, anchor = south west,
                 fill = black!5, 
                 minimum width  = 11.8*\unit,
                 minimum height = 2.4*\unit,] 
        (layer)
        {}; 
       \path (layer.south west) ++ (0.0,0.0) node [above right] {$(\ccalA_2, \ccalM_2, \rho_2)$};
      \path (L1 Filter1.west |- L3 F1.south)  ++ (-0.4,-0.7)
           node [filter, anchor = south west,
                 fill = black!5, 
                 minimum width  = 11.8*\unit,
                 minimum height = 2.4*\unit,] 
        (layer)
        {}; 
       \path (layer.south west) ++ (0.0,0.0) node [above right] {$(\ccalA_3, \ccalM_3, \rho_3)$};  \end{pgfonlayer}

\end{tikzpicture}} 
  \caption{Algebraic Neural Network $\Xi=\{(\ccalA_{\ell},\ccalM_{\ell},\rho_{\ell})\}_{\ell=1}^{3}$ with three layers indicating how the input signal $\bbx$ is processed by $\Xi$ and mapped into $\bbx_{3}$.}
  \label{fig_6}
\end{figure}


\section{Algebraic Neural Networks}\label{sec_Algebraic_NNs}
With the concept of algebraic filtering at hand we define an algebraic neural network (AlgNN) as a stacked layered structure (see Fig.~\ref{fig_6})
 in which each layer is composed by the triple $(\ccalA_{\ell},\ccalM_{\ell},\rho_{\ell})$, which is an algebraic signal model associated to each layer. Notice that $(\ccalM_{\ell},\rho_{\ell})$ is a representation of $\ccalA_{\ell}$. The mapping between layers is performed by the maps $\sigma_{\ell}:\ccalM_{\ell}\to\ccalM_{\ell+1}$ that perform those operations of point-wise nonlinearity and pooling. Then, the ouput from the layer $\ell$ in the AlgNN is given by
\begin{equation}\label{eq:xl}
\bbx_{\ell}=\sigma_{\ell}\left(\rho_{\ell}(a_{\ell})\bbx_{\ell-1}\right)
\end{equation}
where $a_{\ell}\in\ccalA_{\ell}$, which can be represented equivalently as
\begin{equation}\label{eq:interlayeropalgnn}
\bbx_{\ell}=\Phi (\bbx_{\ell-1},\ccalP_{\ell-1},\ccalS_{\ell-1}),
\end{equation}
where $\ccalP_{\ell}\subset\ccalA_{\ell}$ highlights the properties of the filters and $\ccalS_{\ell}$ is the set of shifts associated to $(\ccalM_{\ell},\rho_{\ell})$. Additionally, the term $\Phi\left(\bbx,\{ \ccalP_{\ell} \}_{1}^{L},\{ \ccalS_{\ell}\}_{1}^{L}\right)$ represents the total map associated to an AlgNN acting on a signal $\bbx$.


\medskip\noindent{\bf Convolutional Features.} The processing in each layer can be performed by means of several families of filters, which will lead to several \textit{features}. In particular the feature $f$ obtained in the layer $\ell$ is given by
\begin{equation}\label{eqn:alg_feat}
\bbx_{\ell}^{f}=\sigma_{\ell}\left(\sum_{g=1}^{F_{\ell}}\rho_{\ell}\left(a_{\ell}^{gf}\right)\bbx_{\ell-1}^{g}\right),
\end{equation}
where $a_{\ell}^{fg}$ is the filter in $\ccalA_{\ell}$ used to process the $g$-th feature $\bbx_{\ell-1}^{g}$ obtained from layer $\ell-1$ and $F_{\ell}$ is the number of features. 


\medskip\noindent{\bf Pooling.} As stated in~\cite{deeplearning_book} the pooling operation in CNNs helps to keep representations approximately invariant to small translations of an input signal, and also helps to improve the computational efficiency. In this work this operation is attributed  to the operator $\sigma_{\ell}$. In particular, we consider $\sigma_{\ell}=P_{\ell}\circ\eta_{\ell}$ where $P_{\ell}$ is a pooling operator and $\eta_{\ell}$ is a pointwise nonlinearity. The only property assumed from $\sigma_{\ell}$ is to be Lipschitz and to have zero as a fixed point, i.e. $\sigma_{\ell}(0)=0$. It is important to point out that $P_{\ell}$ projects elements from a given vector space into another.

We present some examples to clarify ideas.


\begin{example}[{\bf CNNs in Discrete Time}]\normalfont 
Traditional CNNs rely on the use of typical signal processing models and can be considered a particular case of an AlgNN  where the algebraic signal model is the same as in example~\ref{ex_DTSP}. Consequently, 
the $f$th feature in layer $\ell$ is given by
\begin{equation}
\bbx_{\ell}^{f}=\sigma_{\ell}\left(\sum_{g=1}^{F_{\ell}}\sum_{k=1}^{K}h_{\ell k}^{gf}S_{\ell}^{k}\bbx_{\ell-1}^{g}\right),
\label{eq:featureCNN}
\end{equation}
where $\rho_{\ell}(t) = S_{\ell}$. In this case $P_{\ell}$ is a sampling operator while typically $\eta_{\ell}(u)=\max\{0,u\}$.
\end{example}


\begin{example}[Graph Neural Networks]  \normalfont
In graph neural networks the algebraic signal model in each layer corresponds to the one discussed in example~\ref{ex_GSP}. Therefore, the $f$th feature in layer $\ell$ has the form
\begin{equation}
\bbx_{\ell}^{f}=\sigma_{\ell}\left(\sum_{g=1}^{F_{\ell}}\sum_{k=1}^{K}h_{\ell k}^{gf}\bbS_{\ell}^{k}\bbx_{\ell-1}^{g}\right),
\label{eq:featureCNN}
\end{equation}
where $\rho_{\ell}=\bbS_\ell$. Here $P_{\ell}$ can be a dimensionality reduction operator or a zeroing operator that nullify components of the signal keeping its dimensionality. A common choice of the nonlinearity function is given by $\eta_{\ell}(u)=\max\{0,u\}$.
\end{example}


\begin{example}[Group Neural Networks]\normalfont
In group neural networks the algebraic model is the same as specified in example~\ref{ex_groupsp}. Therefore, the $f$th feature in layer $\ell$ is given by
\begin{equation}
\bbx_{\ell}^{f}=\sigma_{\ell}\left(\sum_{n=1}^{F_{\ell}}\sum_{u,h\in G_{\ell}}\bba^{nf}_{(\ell)}(uh^{-1})\bbx_{\ell-1}(h)u\right).
\label{eq:featureGroupNN}
\end{equation}
Where $G_{\ell}$ is the group associated to the $\ell$th layer and $\bba^{nf}_{(\ell)}$ are the coefficients of the filter associated to the feature $f$ in layer $\ell$. In this case $P_{\ell}:L^{2}(G_{\ell})\to L^{2}(G_{\ell+1})$,  where $L^{2}(G)$ is the set of signals of finite energy defined on the group $G$. If the groups $G_{\ell}$ are finite $P_{\ell}$ can be conceived as a typical projection mapping between $\mathbb{R}^{\vert G_{\ell}\vert}\to\mathbb{R}^{\vert G_{\ell+1}\vert}$.
\end{example}

%
%

%
\section{Perturbations}\label{sec:perturbandstability}
In an ASP triple $(\ccalA,\ccalM,\rho)$, signals $\bbx\in\ccalM$ are observations of interest and the algebra $\ccalA$ defines the operations that are to be performed on signals. The homomorphism $\rho$ ties these two objects and, as such, is one we can consider as subject to model mismatch. In this paper we consider perturbations adhering to the following model.


\begin{definition}\label{def:perturbmodel}(ASP Model Perturbation)
Let $(\ccalA,\ccalM,\rho)$ be an ASP model with algebra elements generated by $g\in\ccalG$ (Definition \ref{def_generators}) and recall the definition of the shift operators $\bbS = \rho(g)$ (Definition \ref{def_shift_operators}). We say that $(\ccalA,\ccalM,\tdrho)$ is a perturbed ASP model if for all $a=p(\ccalG)$ we have that
\begin{equation}\label{eqn_def_perturbation_model_10}
   \tdrho(a) = p_\ccalM\big(\tdrho(g)\big) 
             = p_\ccalM\big(\tilde\ccalS\big)
             =p\big(\tilde\ccalS\big),
\end{equation}
where $\tilde\ccalS$ is a set of perturbed shift operators of the form
\begin{equation}\label{eqn_def_perturbation_model_20}
   \tbS = \bbS + \bbT(\bbS),
\end{equation}
for all shift operators $\bbS\in\ccalS$. \end{definition}


As per Definition \ref{def:perturbmodel}, an ASP perturbation model, is a perturbation of the homomorphism $\rho$ defined by a perturbation of the shift operators $\bbS$. Each shift operator $\bbS$ is perturbed to the shift operator $\tbS$ according to \eqref{eqn_def_perturbation_model_20} and this perturbation propagates to the filter $\rho(a)$ according to \eqref{eqn_def_perturbation_model_10}. An important technical remark is that the resulting mapping $\tdrho$ is not required to be a homomorphism -- although it can be, indeed, often is.

We point out that Definition \ref{def:perturbmodel} limits the perturbation of the homomorphism $\rho$ to perturbations of the shift operators. This is justifiable by practical considerations. In the case of graph signals a perturbation of the homomorphism models changes in the graph or errors in the measurement of edge weights. In the case of time signals, images, or groups, a perturbation of the homomorphism is an appropriate model of a diffeomorphism -- a small warping of the domain. See Section~\ref{sec:stbcases} for more details.

Of the other components of an algebraic filter, the algebra and the vector space define the choice of operations and therefore are not naturally subject to perturbation. Perturbations of the input signal $\bbx$ are possible in practice but their theoretical analysis is simple. Filters are linear functions of the input and the nonlinear operations of AlgNNs are Lipschitz. Thus, algebraic filters and AlgNNs are readily shown to be Lipschitz stable to perturbations of the input $\bbx$. 

%
%
\subsection{Perturbation Models}
In our subsequent analysis we consider perturbation models of the form
\begin{equation}\label{eqn_perturbation_model_absolute_plus_relative}
   \bbT(\bbS)=\bbT_{0} + \bbT_{1}\bbS,
\end{equation}
which is a generic model of small perturbations of a shift operator that involve an absolute perturbation $\bbT_{0}$ and a relative perturbation $\bbT_{1}\bbS$;  see \cite{fern2019stability}.  The $\bbT_{r}$ are compact normal operators with operator norm $\Vert\bbT_{r}\Vert<1$. Requiring $\Vert\bbT_{r}\Vert<1$ is a minor restriction as we are interested in small perturbations with $\Vert\bbT_{r}\Vert\ll1$. 

For the model in \eqref{eqn_perturbation_model_absolute_plus_relative} it is important to describe the commutativity of the shift operator $\bbS$ and the perturbation model operators $\bbT_r$. To that end, we write
\begin{equation}\label{eq:PrvsTr}
\mathbf{S}\mathbf{T}_{r}=\mathbf{T}_{cr}\mathbf{S}+\mathbf{S}\mathbf{P}_{r},
\end{equation}
where $\mathbf{T}_{cr}=\sum_{i}\mu_{i}\mathbf{u}_{i}\langle\mathbf{u}_{i},\cdot\rangle$, $\mu_{i}$ is the $i$th eigenvalue of $\mathbf{T}_{r}$, $\mathbf{u}_{i}$ is the $i$th eigenvector of $\mathbf{S}$, and $\langle,\rangle$ represents the inner product operation. As a consequence, we have that $\mathbf{S}\mathbf{T}_{rc}=\mathbf{T}_{rc}\mathbf{S}$ and $\Vert\mathbf{T}_{cr}\Vert=\Vert\mathbf{T}_{r}\Vert$. We define the commutation factor $\delta$ according to
\begin{equation}\label{eqn_commutation_factor}
   \| \bbP_{r}\|_{\text{F}}  \leq \delta \| \bbT_r\|,
\end{equation}
which is a measure of how far the operators $\bbS$ and $\bbT_r$ are from commuting with each other. Notice that $\delta=0$ implies $\mathbf{P}_{r}=\mathbf{0}$ and $\mathbf{T}_{r}=\mathbf{T}_{cr}$. The commutation factor $\delta$ in \eqref{eqn_commutation_factor} can be bounded as we show in Proposition \ref{prop_commutation_factor}. The specifics of this bound are not central to the results of Section \ref{sec:stabilitytheorems}.
Notice that when representations of an algebra $\ccalA$ with multiple generators $\{g_{i}\}_{i=1}^{m}$ are considered, we have that for $a\in\ccalA$ the operator 
$p(\rho(a))\in\text{End}(\mathcal{M})$ is a function of  $\rho(g_{i})=\mathbf{S}_{i}\in\text{End}(\mathcal{M})$ and therefore can be seen as the function $p: \text{End}(\mathcal{M})^{m}\to\text{End}(\mathcal{M})$, where $\text{End}(\mathcal{M})^{m}$ is the $m$-times cartesian product of $\text{End}(\mathcal{M})$.  In this scenario we use the notation $p(\mathbf{S})=p(\mathbf{S}_{1},\ldots,\mathbf{S}_{m})$  and  when considering the perturbation model in eqn.~(\ref{eqn_perturbation_model_absolute_plus_relative}) acting on $\mathbf{S}=(\mathbf{S}_{1},\ldots,\mathbf{S}_{m})$ we use the following notation $\mathbf{T}(\mathbf{S})=\left(\mathbf{T}(\mathbf{S}_{1}),\ldots,\mathbf{T}(\mathbf{S}_{m})\right)$ where $\mathbf{T}(\mathbf{S}_{i})=\mathbf{T}_{0,i}+\mathbf{T}_{1,i}\mathbf{S}_{i}$.




\section{Stability Theorems}\label{sec:stabilitytheorems}
The filters in Section \ref{sec_alg_filters} and the algebraic neural networks in Section \ref{sec_Algebraic_NNs} are operators acting on the space $\ccalM$. These operators are of the form $p(\mathbf{S})$, and their outputs depend on a filter set $\ccalP\subset\ccalA$ which is denoted as $p\in\mathcal{P}\subset\ccalA$, and the set of shift operators $\ccalS$, where $\mathbf{S}\in\mathcal{S}$. When we perturb the processing model according to Definition \ref{def:perturbmodel}, these operators are perturbed as well. The goal of this paper is to analyze these perturbations. In particular, our goal is to identify conditions for filters and algebraic neural networks to be stable in the sense of the following definition. 


\begin{definition}[Operator Stability]\label{def:stabilityoperators1} Given operators $p(\mathbf{S})$ and $p(\tilde{\mathbf{S}})$ defined on the processing models $(\ccalA,\ccalM,\rho)$ and $(\ccalA,\ccalM,\tdrho)$ (cf. Definition \ref{def:perturbmodel}) we say the operator $p(\mathbf{S})$ is Lipschitz stable if there exist constants $C_{0}, C_{1}>0$ such that 
\begin{align}\label{eq:stabilityoperators1}
   & \left \Vert p(\mathbf{S})\mathbf{x}  - p(\tilde{\mathbf{S}})\mathbf{x} \right\Vert
      \leq  \nonumber \\ & \quad~~
         \left[ C_{0} \sup_{\bbS\in\ccalS}\Vert\mathbf{T}(\mathbf{S})\Vert +   
                   C_{1}\sup_{\bbS\in\ccalS}\big\|D_{\bbT}(\bbS)\big\|
                      +\mathcal{O}\left(\Vert\mathbf{T}(\mathbf{S})\Vert^{2}\right)
         \right] 
            \big\| \bbx \big\|,
\end{align}
for all $\bbx\in\ccalM$. In \eqref{eq:stabilityoperators1} $D_{\bbT}(\bbS)$ is the Fr\'echet derivative of the perturbation operator $\bbT$. 
\end{definition}


When the perturbation value $\bbT(\bbS)$ and its derivative $D_{\bbT}(\bbS)$ are small, the inequality in \eqref{eq:stabilityoperators1} states that the operators $p(\mathbf{S})$ and $ p(\tilde{\mathbf{S}})$ are close uniformly across all inputs $\bbx$. Our stability theorems are presented in the next section, but at this point it is important to remark that algebraic filters are not always stable in the sense of \eqref{eq:stabilityoperators1}. We know that this is true because unstable counterexamples are known in the case of graph signal processing \cite{gamagnns} and the processing of time signals \cite{mallat_ginvscatt}. The best known example of an unstable filter is a high-pass filter in time when consider a dilation of the time line \cite{gamabruna_diffscattongraphs}. The same phenomenon is observed for graph signals when considering the dilation of graph shift operator \cite{gamagnns}.


\subsection{Stability of Algebraic Filters}

Taking into account that the notion of stability is meant to be satisfied by subsets of filters of the algebra and not necessarily the whole algebra, it is important to have a characterization of these subsets in simple terms. To do so, we introduce the notion of frequency representation of the elements of an algebra as follows.


\begin{definition}[{\bf Frequency Representation of a Filter}]\label{def_freq_rep}
Consider an algebra $\ccalA$ with a single generator $g$ so that for all $a\in\ccalA$ we can write $a=p(g)$. Let $\lam\in\mbF$ be a variable taking values on the field $\mbF$. We say that $p(\lam)$ is the frequency representation of the filter $a=p(g)$.
\end{definition}


Notice that the frequency representation of the elements of the algebra $\ccalA$ induces an isomorphism of algebras $\iota: \ccalA \mapsto \ccalA_{\mbF}$, where $\ccalA_{\mbF}$ is obtained when the variables of elements in $\ccalA$ are evaluated in $\mbF$. Then, we can characterize elements in $\ccalA$ by means of the properties of their frequency representations. In what follows we introduce a definition used to characterize subsets of filters in algebras with a single generator that are relevant in our analysis.


\begin{definition}
Let $p:\mbF\to\mbF$ be the frequency representation of an element in an algebra with a single generator. Then, it is said that $p$ is Lipschitz if there exists $L_{0}>0$ such that
\begin{equation}
\vert p(\lambda)-p(\mu)\vert\leq L_{0}\vert\lambda-\mu\vert
\end{equation}
for all $\lambda, \mu\in\mbF$. Additionally, it is said that $p(\lambda)$ is Lipschitz integral if there exists $L_{1}>0$ such that
\begin{equation}
\left\vert 
\lambda\frac{dp(\lambda)}{d\lambda}
\right\vert
\leq L_{1}
~\forall~\lambda.
\end{equation}
\end{definition}


In what follows, when considering subsets of a commutative algebra $\ccalA$, we denote by $\ccalA_{L_{0}}$ the subset of elements in $\ccalA$ that are Lipschitz with constant $L_{0}$ and by $\ccalA_{L_{1}}$ the subset of element of $\ccalA$ that are Lipschitz integral with constant $L_{1}$. 

We start our discussion on stability with a result for operators in algebraic models with a single generator. The result highlights the role of the Fr\'echet derivative of the map that relates the operator and its perturbed version.


\begin{theorem}\label{theorem:HvsFrechet}
Let $\ccalA$ be an algebra generated by $g$ and let $(\mathcal{M},\rho)$ be a representation of $\ccalA$ with $\rho(g)=\bbS\in\text{End}(\mathcal{M})$. Let $\tilde{\rho}(g)=\tilde{\bbS}\in\text{End}(\mathcal{M})$ where the pair
$(\mathcal{M},\tilde{\rho})$ is a perturbed version of $(\mathcal{M},\rho)$ and $\tilde{\bbS}$ is related to $\bbS$ by the perturbation model in eqn.~(\ref{eqn_def_perturbation_model_20}). Then, for any $p\in\ccalA$ we have
\begin{equation}
\left\Vert p(\bbS)\mathbf{x}-p(\tilde{\bbS})\mathbf{x}\right\Vert
\leq 
\Vert\mathbf{x}\Vert
\left(
\left\Vert D_{p}(\mathbf{S})\left\lbrace\mathbf{T}(\mathbf{S})\right\rbrace\right\Vert + \mathcal{O}\left(\Vert\mathbf{T}(\mathbf{S})\Vert^{2}
\right)\right)
\label{eq:HSoptbound}
\end{equation}
where $D_{p}(\mathbf{S})$ is the Fr\'echet derivative of $p$ on $\mathbf{S}$.
\end{theorem}
\begin{proof}
See Section~\ref{prooftheoremHvsFrechet}
\end{proof}
Theorem~\ref{theorem:HvsFrechet} highlights an important point, the difference between two operators obtained from the same elements in the algebra is bounded by the Fr\'echet derivative of $p(\bbS)$ which depends of the properties of the elements in $\ccalA$. In particular, we can see that an upper bound in the term $\left\Vert D_{p}(\mathbf{S})\bbT(\bbS)\right\Vert$ depends on how the the operator $D_{p}(\mathbf{S})$ acts on the perturbation $\bbT(\bbS)$. Then, $D_{p}(\mathbf{S})$ will determine whether $p(\bbS)$ is stable under the effect of $\bbT(\bbS)$, or in other words the properties of $p$ act on the perturbation via the operator $D_{p}(\mathbf{S})$. Additionally, notice that eqn.~(\ref{eq:HSoptbound}) is satisfied for any $\bbT(\bbS)$ if $D_{p}(\mathbf{S})$ exists.

In the following theorems we show how these terms are related to $\bbT(\bbS)$ and its Fr\'echet derivative $D_{\bbT}$.


\begin{theorem}\label{theorem:uppboundDH}
Let $\ccalA$ be an algebra with one generator element $g$ and let $(\mathcal{M},\rho)$ be a finite or countable infinite dimensional representation of $\ccalA$. Let  $(\mathcal{M},\tilde{\rho})$ be a perturbed version of $(\mathcal{M},\rho)$ associated to the perturbation model in eqn.~(\ref{eqn_perturbation_model_absolute_plus_relative}).  If $p\in\ccalA_{L_{0}}\cap\ccalA_{L_{1}}$, then
\begin{equation}
\left\Vert D_{p}\bbT(\bbS)\right\Vert
\leq(1+\delta)
\left(L_{0}\sup_{\bbS}\Vert\bbT(\bbS)\Vert +L_{1}\sup_{\bbS}\Vert D_{\bbT}(\bbS)\Vert\right)
\label{eq:DHTS}
\end{equation}
 \end{theorem}
\begin{proof}See Section~\ref{prooftheorem:uppboundDH}\end{proof}


It is worth pointing out that the constants involved in the upper bound of eqn.~(\ref{eq:DHTS}) depend on the properties of the filters and the difference between the eigenvectors of $\bbS$ and $\bbT_{r}$. Therefore, the difference between the eigenvectors of these operators do not determine if $p(\bbS)$ is stable or not, although the absolute value of the stability constants increase proportionally to $\delta$.

From theorems~\ref{theorem:HvsFrechet} and~\ref{theorem:uppboundDH} we can state the notion of stability for algebraic filters in the following corollary.


\begin{corollary}\label{corollary:DHboundsprcases}
Let $\ccalA$ be an algebra with one generator element $g$ and let $(\mathcal{M},\rho)$ be a finite or countable infinite dimensional representation of $\ccalA$. Let  $(\mathcal{M},\tilde{\rho})$ be a perturbed version of $(\mathcal{M},\rho)$ related by the perturbation model in eqn.~(\ref{eqn_perturbation_model_absolute_plus_relative}). Then, if $p\in\ccalA_{L_{0}}\cap\ccalA_{L_{1}}$ the operator $p(\mathbf{S})$ is stable in the sense of definition~\ref{def:stabilityoperators1} with $C_{0}=(1+\delta)L_{0}$ and $C_{1}=(1+\delta)L_{1}$.
\end{corollary}

\begin{proof}
Replace (\ref{eq:DHTS}) from Theorem \ref{theorem:uppboundDH} into (\ref{eq:HSoptbound}) from Theorem \ref{theorem:HvsFrechet} and reorder terms.
\end{proof}


\subsection{Algebraic Filter Stability in Algebras with Multiple Generators}

The stability results presented in previous subsection can be extended naturally to operators associated to representations of algebras with multiple generators. To do so, we introduce the notion of frequency representation of elements of algebras with multiple generators as follows.


\begin{definition}[Frequency Representation of a Filter]\label{def_freq_rep_multiple_generators}
Consider an algebra $\ccalA$ with generators $g_1,\ldots,g_m$ so that for all $a\in\ccalA$ we can write $a=p(g_1,\ldots,g_m)$. Let $\lam_i\in\mbF$ be variables taking values on the filed $\mbF$. We say that $p(\lam_1,\ldots,\lam_m)$ is the frequency representation of the filter $a=p(g_1,\ldots,g_m)$.
\end{definition}


Similar to the scenario of algebras with a single generator, the frequency representation of the elements of $\ccalA$ induces an isomorphism of algebras $\iota: \ccalA \mapsto \ccalA_{\mbF}$, where $\ccalA_{\mbF}$ is obtained when the variables of elements in $\ccalA$ are evaluated in $\mbF$. In this way we have a characterization of elements in $\ccalA$ when considering the properties of their frequency representations. 

We extend definitions introduced before to characterize frequency representations in multivariate algebras.


\begin{definition}
Let $p:\mbF^{m}\to\mbF$ be the frequency representation of an element in an algebra with $m$ generators. Then, it is said that $p$ is Lipschitz if there exists $L_{0}>0$ such that
\begin{equation}
\vert p(\boldsymbol{\lam})-p(\boldsymbol{\mu})\vert\leq L_{0}\Vert\boldsymbol{\lam}-\boldsymbol{\mu}\Vert
\end{equation}
for all $\boldsymbol{\lambda},\boldsymbol{\mu}\in\mbF^{m}$. Additionally, it is said that $p(\boldsymbol{\lambda})$ is Lipschitz integral if there exists $L_{1}>0$ such that
\begin{equation}
\left\vert \lambda_{i}\frac{\partial p(\boldsymbol{\lambda})}{\partial \lambda_{i}}\right\vert\leq L_{1}\quad\forall~i\in\{1,\ldots m \},
\end{equation}
where $\boldsymbol{\lambda}=(\lambda_{1},\ldots,\lambda_{m})$ and $\boldsymbol{\mu}=(\mu_{1},\ldots,\mu_{m})$.
\end{definition}


With these notions at hand, we are ready to  extend the stability theorems.


\begin{theorem}\label{theorem:HvsFrechetmultgen}
Let $\ccalA$ be an algebra generated by $\{g_{i}\}_{i=1}^{m}$ and let $(\mathcal{M},\rho)$ be a representation of $\ccalA$ with $\rho(g_{i})=\bbS_{i}\in\text{End}(\mathcal{M})$ for all $i$. Let $\tilde{\rho}(g_{i})=\tilde{\bbS}_{i}\in\text{End}(\mathcal{M})$ where the pair
$(\mathcal{M},\tilde{\rho})$ is a perturbed version of $(\mathcal{M},\rho)$ and $\tilde{\bbS}_{i}$ is related to $\bbS_{i}$ by the perturbation model in eqn.~(\ref{eqn_def_perturbation_model_20}). Then, for any $p\in\ccalA$ we have
\begin{multline}
\left\Vert p(\bbS)\mathbf{x}-p(\tilde{\bbS})\mathbf{x}\right\Vert\leq 
\\
\Vert\mathbf{x}\Vert\sum_{i=1}^{m}\left(\left\Vert D_{p\vert\mathbf{S}_{i}}(\mathbf{S})\mathbf{T}(\mathbf{S}_{i})\right\Vert+\mathcal{O}\left(\Vert\mathbf{T}(\mathbf{S}_{i})\Vert^{2}\right)\right)
\label{eq:HSoptboundmultgen}
\end{multline}
where $D_{p\vert\mathbf{S}_{i}}(\mathbf{S})$ is the partial Fr\'echet derivative of $p$ on $\mathbf{S}_{i}$.
\end{theorem}
\begin{proof}
See Section~\ref{prooftheoremHvsFrechet}
\end{proof}


Notice that in eqn.~(\ref{eq:HSoptboundmultgen}) we naturally add the contribution associated to each generator. Therefore, to guarantee stability we must have stability in each generator. Now, we show how the Fr\'echet derivative of $\mathbf{T}(\mathbf{S})$ is involved in the stability properties when considering multiple generators.


\begin{theorem}\label{theorem:uppboundDHmultg}
Let $\ccalA$ be an algebra with $m$ generators $\{g_{i}\}_{i=1}^{m}$ and $g_{i}g_{j}=g_{j}g_{i}$ for all $i,j\in\{1,\ldots m\}$. Let $(\mathcal{M},\rho)$ be a finite or countable infinite dimensional representation of $\ccalA$ and $(\mathcal{M},\tilde{\rho})$ a perturbed version of $(\mathcal{M},\rho)$ related by the perturbation model in eqn.~(\ref{eqn_perturbation_model_absolute_plus_relative}).  Then, if $p\in\ccalA_{L_{0}}\cap\ccalA_{L_{1}}$ it holds that
\begin{multline}
\left\Vert D_{p\vert\mathbf{S}_{i}}(\mathbf{S})\bbT(\bbS_{i})\right\Vert
\\
\leq (1+\delta)
\left(
L_{0}\sup_{\bbS_{i}\in\mathcal{S}}\Vert\bbT(\bbS_{i})\Vert +L_{1}\sup_{\bbS_{i}\in\mathcal{S}}\Vert D_{\bbT}(\bbS_{i})\Vert\right)
\label{eq:DHTSmultg}
\end{multline}
 \end{theorem}
 \begin{proof}See Section~\ref{prooftheorem:uppboundDH}\end{proof}


It is important to remark that the upper bound in eqn.~(\ref{eq:DHTSmultg}) is defined by the largest perturbation in a given generator although the constants associated are determined completely by the properties of the filters.
 
From theorems~\ref{theorem:HvsFrechetmultgen} and~\ref{theorem:uppboundDHmultg} we can state the stability results for filters in algebras with multiple generators in the following corollary.


\begin{corollary}\label{corollary:stabmultgenfilt}
Let $\ccalA$ be an algebra with generators $\{ g_{i}\}_{i=1}^{m}$ and $g_{i}g_{j}=g_{j}g_{i}$ for all $i,j$. Let $(\mathcal{M},\rho)$ be a finite or countable infinite dimensional representation of $\ccalA$ and $(\mathcal{M},\tilde{\rho})$ be a perturbed version of $(\mathcal{M},\rho)$ related by the perturbation model in eqn.~(\ref{eqn_perturbation_model_absolute_plus_relative}). Then, if $p\in\ccalA_{L_{0}}\cup\ccalA_{L_{1}}$ the operator $p(\mathbf{S})$ is stable in the sense of definition~\ref{def:stabilityoperators1} with $C_{0}=m(1+\delta)L_{0}$ and $C_{1}=m(1+\delta)L_{1}$.
\end{corollary}
\begin{proof} Replacing eqn.~(\ref{eq:DHTSmultg}) from theorem~\ref{theorem:uppboundDHmultg} into eqn.~(\ref{eq:HSoptboundmultgen}) from theorem~\ref{theorem:HvsFrechetmultgen} and organazing the terms.\end{proof}


\subsection{Stability of Algebraic Neural Networks}

The results in Theorems~\ref{theorem:HvsFrechet} to~\ref{theorem:uppboundDHmultg} and corollaries~\ref{corollary:DHboundsprcases} and~\ref{corollary:stabmultgenfilt} can be extended to operators representing AlgNNs. We say that for a given AlgNN,  $\Xi=\left\lbrace (\ccalA_{\ell},\mathcal{M}_{\ell},\rho_{\ell})\right\rbrace_{\ell=1}^{L}$, a perturbed version of $\Xi$ is given by  $\tilde{\Xi}=\left\lbrace (\ccalA_{\ell},\mathcal{M}_{\ell},\tilde{\rho}_{\ell})\right\rbrace_{\ell=1}^{L}$ where $(\ccalA_{\ell},\mathcal{M}_{\ell},\tilde{\rho}_{\ell})$ is a perturbed version of $(\ccalA_{\ell},\mathcal{M}_{\ell},\rho_{\ell})$. For the sake of simplicity we present a theorem for AlgNNs with algebras with a single generator, but notice that these results can be easily stated for 
AlgNNs with multiple generators directly from theorems~\ref{theorem:HvsFrechetmultgen} and~\ref{theorem:uppboundDHmultg}. To do so, we start highlighting in the following theorem the stability properties of the operators in the layer $\ell$ of an AlgNN.


\begin{theorem}\label{theorem:stabilityAlgNN0}
Let $\Xi=\left\lbrace (\ccalA_{\ell},\mathcal{M}_{\ell},\rho_{\ell})\right\rbrace_{\ell=1}^{L}$ be an algebraic neural network  with $L$ layers, one feature per layer and algebras $\ccalA_{\ell}$ with a single generator.  Let  $\tilde{\Xi}=\left\lbrace (\ccalA_{\ell},\mathcal{M}_{\ell},\tilde{\rho}_{\ell})\right\rbrace_{\ell=1}^{L}$ be the perturbed version of $\Xi$ by means of the perturbation model in eqn.~(\ref{eqn_perturbation_model_absolute_plus_relative}). Then, if  $\Phi\left(\mathbf{x}_{\ell-1}, \mathcal{P}_{\ell},\mathcal{S}_{\ell}\right)$ and 
$\Phi\left(\mathbf{x}_{\ell-1},\mathcal{P}_{\ell},\tilde{\mathcal{S}}_{\ell}\right)$ represent the mapping operators associated to $\Xi$ and $\tilde{\Xi}$ in the layer $\ell$ respectively, we have
\begin{multline}
\left\Vert
\Phi\left(\mathbf{x}_{\ell-1},\mathcal{P}_{\ell},\mathcal{S}_{\ell}\right)-
\Phi\left(\mathbf{x}_{\ell-1},\mathcal{P}_{\ell},\tilde{\mathcal{S}}_{\ell}\right)
\right\Vert
\leq
\\
C_{\ell}(1+\delta_{\ell})\left(L_{0}^{(\ell)} \sup_{\bbS_{\ell}}\Vert\bbT^{(\ell)}(\bbS_{\ell})\Vert 
\right.
\\
\left.
+L_{1}^{(\ell)}\sup_{\bbS_{\ell}}\Vert D_{\mathbf{T^{(\ell)}}}(\bbS_{\ell})\Vert\right)\Vert\mathbf{x}_{\ell-1}\Vert
\label{eq:theoremstabilityAlgNN0}
\end{multline}
where $C_{\ell}$ is the Lipschitz constant of $\sigma_{\ell}$, and $\mathcal{P}_{\ell}=\ccalA_{L_{0}}\cap\ccalA_{L_{1}}$ represents the domain of $\rho_{\ell}$. 
The index $(\ell)$ makes reference to quantities and constants associated to the layer $\ell$.
\end{theorem}
\begin{proof}See Section~\ref{prooftheorem:stabilityAlgNN0}\end{proof}


This result, although simple, highlights the role of the maps $\sigma_{\ell}$ when perturbations are considered in each layer. In particular, we see that the effect of $\sigma_{\ell}$  is to scale $\boldsymbol{\Delta}_{\ell}$ by a constant but it does not change the nature or mathematical form of the perturbation. Notice also that $\sigma_{\ell}$ plays the role of a mixer that allows an AlgNN to provide selectivity without affecting the stability (see Section~\ref{sec:discussion}). 

Now we present in the following theorem the stability result for a general AlgNN with commutative algebras.


\begin{theorem}\label{theorem:stabilityAlgNN1}
Let $\Xi=\left\lbrace (\ccalA_{\ell},\mathcal{M}_{\ell},\rho_{\ell})\right\rbrace_{\ell=1}^{L}$ be an algebraic neural network  with $L$ layers, one feature per layer and algebras $\ccalA_{\ell}$ with a single generator.  Let  $\tilde{\Xi}=\left\lbrace (\ccalA_{\ell},\mathcal{M}_{\ell},\tilde{\rho}_{\ell})\right\rbrace_{\ell=1}^{L}$ be the perturbed version of $\Xi$ by means of the perturbation model in eqn.~(\ref{eqn_perturbation_model_absolute_plus_relative}). Then, if  $\Phi\left(\mathbf{x},\{ \mathcal{P}_{\ell} \}_{1}^{L},\{ \mathcal{S}_{\ell}\}_{1}^{L}\right)$ and 
$\Phi\left(\mathbf{x},\{ \mathcal{P}_{\ell} \}_{1}^{L},\{ \tilde{\mathcal{S}}_{\ell}\}_{1}^{L}\right)$ represent the mapping operators associated to $\Xi$ and $\tilde{\Xi}$ respectively, we have
\begin{multline}
\left\Vert
\Phi\left(\mathbf{x},\{ \mathcal{P}_{\ell} \}_{1}^{L},\{ \mathcal{S}_{\ell}\}_{1}^{L}\right)-
\Phi\left(\mathbf{x},\{ \mathcal{P}_{\ell} \}_{1}^{L},\{ \tilde{\mathcal{S}}_{\ell}\}_{1}^{L}\right)
\right\Vert
\\
\leq
\sum_{\ell=1}^{L}\boldsymbol{\Delta}_{\ell}\left(\prod_{r=\ell}^{L}C_{r}\right)\left(\prod_{r=\ell+1}^{L}B_{r}\right)
\left(\prod_{r=1}^{\ell-1}C_{r}B_{r}\right)\left\Vert\mathbf{x}\right\Vert
\label{eq:theoremstabilityAlgNN1}
\end{multline}
where $C_{\ell}$ is the Lipschitz constant of $\sigma_{\ell}$ and $B_\ell$ is a bound on the filter's norm, $\Vert\rho_{\ell}(a)\Vert\leq B_{\ell}$. The functions $\boldsymbol{\Delta}_{\ell}$ are given by
\begin{equation}\label{eq:varepsilonl}
\boldsymbol{\Delta}_{\ell}=(1+\delta_{\ell})\left(L_{0}^{(\ell)} \sup_{\bbS_{\ell}}\Vert\bbT^{(\ell)}(\bbS_{\ell})\Vert 
+L_{1}^{(\ell)}\sup_{\bbS_{\ell}}\Vert D_{\mathbf{T^{(\ell)}}}(\bbS_{\ell})\Vert\right)
\end{equation}
with the index $(\ell)$ indicating quantities and constants associated to the layer $\ell$.
\end{theorem}
\begin{proof}See Section~\ref{prooftheorem:stabilityAlgNN1}\end{proof}


Theorem~\ref{theorem:stabilityAlgNN1} states how an AlgNN can be made stable by the selection of an appropriate subset of filters in the algebra, for a given perturbation model. It is worth pointing out that conditions like the ones obtained in~\cite{fern2019stability} for GNNs can be considered particular instantiations of the conditions in Theorem~\ref{theorem:stabilityAlgNN1}. Additionally, notice that Theorem~\ref{theorem:stabilityAlgNN1} can be easily extended to consider several features per layer, the reader can check the details of the proof of the theorem in Section~\ref{prooftheorem:stabilityAlgNN1} where the analysis is performed considering multiple features. 

The bound in \eqref{eq:theoremstabilityAlgNN1} exhibits an exponential dependency on the Lipschitz constants $C_\ell$ and the maximum filter norms $B_\ell$. This dependency can be avoided if we normalize the nonlinearities so that $C_\ell=1$ and the filters so that $B_\ell=1$. Their presence in \eqref{eq:theoremstabilityAlgNN1} highlights that if the filters and nonlinearities amplify signals, they may amplify errors as well. 


\begin{remark}\normalfont
It is important to highlight the fact that the perturbation model in eqn.~(\ref{eqn_perturbation_model_absolute_plus_relative}) is \textit{smooth} in the space of admissible $\bbS$ and this smoothness allows a consistent calculation of the Fr\'echet derivative of $\bbT(\bbS)$. This can be considered as a consequence of the fact that deformations between arbitrary spaces can be measured according to the topology of the space. In particular, if a diffeomorphism is used to produce deformation in the signal models of interest, it is possible to find an equivalent associated map that produces deformation of the set of operators acting on the signal. If notions of differentiability are used to measure the size of the original diffeomorphism, it is natural to find similar notions involved on the map acting on the operators, but with the difference that the differentiability is measured according to the topology of the new space. 
\end{remark}


\begin{remark}\normalfont
It is worth noticing that the role of the Fr\'echet derivative $\bbD_{p\vert\bbS}(\bbS)$ raises naturally when the the norm of the difference between an operator $p(\bbS)$ and its perturbed version $p(\tilde{\bbS})$ is considered (see Section~\ref{prooftheoremHvsFrechet}), and this is a direct consequence of the definition of such derivative and the type of perturbation considered (see eqn.~(\ref{eq:FreDerdef1}) and eqn.~(\ref{eq:frechetDdef})).  Then, as long as the perturbed shifts considered can be modeled as $\tilde{\bbS}=\bbS+\bbT(\bbS)$, i.e. the perturbation is added to the unperturbed shift,  the operator that acts on $\bbT(\bbS)$ is always the Fr\'echet derivative of the filter.
\end{remark}


\subsection{Implications for particular signal models}\label{sec:stbcases}

In this subsection we show the implication of the stability results for particular signal models.

\medskip \noindent {\bf Graph Neural Networks (GNNs)} In graph signal processing the shift operator $\bbS$ is a matrix representation of a graph. The perturbation model in \eqref{def:perturbmodel} simply states that $\tbS$ is a matrix representation of a different graph. Definition \ref{def:stabilityoperators1} defines a stable operator $p(\mathbf{S})$ as one that doesn't change much when run on graphs that are close and related by perturbations that are sufficiently smooth in the space of matrix representations of graphs. 

The absolute perturbation model considered in~\cite{fern2019stability} is the perturbation model where $\bbT(\bbS)=\bbT_{0}$. Therefore the stability bound for graph filters translates into
\begin{equation}
\left\Vert 
               p(\bbS)-p(\tilde{\bbS})
\right\Vert
                \leq
                      (1+\delta)L_{0}\Vert\bbT_{0}\Vert +\ccalO \left( \Vert\bbT_{0}\Vert^{2}\right),
\end{equation}
which is a scaled version of the result in~\cite{fern2019stability} (Theorem~1). Additionally, $\delta\leq\hat{\delta}\sqrt{N}$ where $\hat{\delta}$ is the non commutativity constant used in~\cite{fern2019stability} which depends on the difference between the eigenvectors of $\bbS$ and $\bbT_r$ -- please see Appendix~\ref{appendix_upperboundPr} where the formal connection between $\delta$ and $\hat{\delta}$ is stated. Notice that $\sup_{\bbS}\Vert\bbT(\bbS)\Vert=\Vert\bbT_{0}\Vert$.

A relative model can be obtained considering $\bbT(\bbS)=\bbT_{1}\bbS$, and in that case the stability bounds are given according to
\begin{equation}
\left\Vert 
               p(\bbS)-p(\tilde{\bbS})
\right\Vert
                \leq
                      (1+\delta)L_{1}\Vert\bbT_{1}\Vert +\ccalO \left( \Vert\bbT(\bbS)\Vert^{2}\right),
\end{equation}
which is a scaled version of the bound obtained in~\cite{fern2019stability}.  Notice that $\sup_{\bbS}\Vert\bbD_{\bbT}(\bbS)\Vert=\Vert\bbT_{1}\Vert$. Like in the previous scenario $\delta\leq\sqrt{N}\hat{\delta}$ where $\hat{\delta}$ is the non commutativity constant used in~\cite{fern2019stability} -- please see Appendix~\ref{appendix_upperboundPr}. It is also important to remark that the stability of bounds derived in~\cite{Gama2019StabilityOG} for graph scattering transforms are rooted in the fact that wavelet graph filters are stable, and as a consequence the stability bounds are scaled versions of the ones derived for graph filters.


\medskip \noindent {\bf CNNs with DTSP} In discrete time signal processing (DTSP) the shift operator is the discrete time shift $S$. 
The processing induced by \eqref{eqn_DTSP_filter_outputs} is invariant to shifts and therefore adequate to processing signals that are shift invariant. In general, signals are close to shift invariant but not exactly so. That is, a given signal $\bbx$ is invariant with respect to a shift operator $\tdS$ that is close to the time shift $S$. If the stability property in \eqref{eq:stabilityoperators1} holds we can guarantee that processing the signal $\bbx$ with the operator $\tdS$ is not far from processing the signal with the operator $S$. The latter represents the operations we perform -- since we choose to use $S$ in the processing of time signals. The former represents the processing we should undertake to respect the actual invariance properties of the signal $\bbx$ -- which are characterized by $\tdS$. The stability bound in this scenario is given by
\begin{multline}\label{eq:stablity_bound_DSPcnns}
\left\Vert p(S)\mathbf{x}  - p(\tdS)\mathbf{x}\right\Vert
\leq
\left[
L_{0}(1+\delta) \sup_{S\in\ccalS}\Vert\mathbf{T}(S)\Vert
\right.
\\
\left.
+ L_{1}(1+\delta)\sup_{S\in\ccalS}\big\|D_{\bbT}(S)\big\|
+\mathcal{O}\left(\Vert\mathbf{T}(S)\Vert^{2}\right)
\right] \big\| \bbx \big\|.
\end{multline}
Notice that although results in eqn.~(\ref{eq:stablity_bound_DSPcnns}) are different from those in~\cite{mallat_ginvscatt}, they exhibit similarities. This is expected since the right hand side of eqn.~(\ref{eq:stablity_bound_DSPcnns}) measures the size of $\bbT(S)$, which is a diffeomorphism acting on the space of admissible shift operators. The bounds derived in~\cite{mallat_ginvscatt} consider diffeomorphisms acting on $\mathbb{R}^{n}$ which is the domain of the signals. Additionally, notice that since the operators considered in~\cite{mallat_ginvscatt} are shift invariant the term associated to the absolute norm of the deformation is not present in the bounds. It is also worth pointing out that the convolutions we consider in the DTSP model are attributed to a polynomial algebra. While the convolutions considered in~\cite{mallat_ginvscatt} are defined considering functions in $L_{2}(\mathbb{R}^n)$ and filters in $L_{1}(\mathbb{R})$, a scenario that requires the use of a non polynomial algebra.


\medskip \noindent {\bf Graphon Neural Networks} Similar to the case of GNNs the graphon $W(u,v)$ is a limit object that represents a family of random graphs. The perturbed graphon $\tdW(u,v)$ represents a different family of random graphs. The perturbation of the graphon generates a corresponding perturbation of the shift operator defined in \eqref{eqn_graphon_shift}. If the condition stated in \eqref{eq:stabilityoperators1} is satisfied, then filtering graphon signals using the perturbed shift operator associated to $\tdW(u,v)$  will lead to similar results to the ones obtained with the unperturbed operator and the differences are proportional to the size of the perturbation acting on $W(u,v)$. For instance, if the perturbation considered is additive we have
\begin{equation}
\left\Vert 
               p(S)-p(\tilde{S})
\right\Vert
                \leq
                      (1+C)L_{0}\Vert\bbT_{0}\Vert
                      +
                      \ccalO\left( \Vert\bbT_0\Vert^{2}\right)
                      , 
\end{equation}
where $S$ is the graphon shift operator indicated in eqn.~(\ref{ex_wsp}) and $C$ is a constant associated to the eigenvalue and eigenvector spreading of the graphon operator. 


\medskip \noindent {\bf Group Neural Network} Similar to the case of DSP, the filters in \eqref{eqn_group_filters} are invariant to the action of the group. Actual signals $\bbx$ are invariant to actions of operators that are close to actions of the group -- e.g., a signal is close to invariant to rotations and symmetries. If \eqref{eq:stabilityoperators1} is true, processing the signal with operators $g$ -- as we choose to do -- is not far from processing the signal with operators $\tdg$ -- as we should do to leverage the actual invariance of the signal $\bbx$. We remark that when we perturb $g$ to $\tdg$ the resulting shift operators will not, in general, be representations of a homomorphism. Notice that when considering the representations of finite commutative groups the analysis of stability is the same as in the case of an architecture based on a DSP model, therefore the stability bounds to perturbations are given by eqn.~(\ref{eq:stablity_bound_DSPcnns}).

%
%
%
%
%

\section{Spectral Operators}\label{sec:spectopt}

Part of our proofs on the stability of AlgNN rely on the notion of spectral or Fourier decompositions associated to the realization of algebraic filters. In this section we discuss the notion of spectrum for general operators associated to algebraic signal models. Such notions of spectral decompositions are a natural generalization of the well established notions of spectrum used in GNNs and CNNs. To do so we elaborate about the concepts of irreducible and indecomposable subrepresentations, which generalize the notions of decompositions in terms of eigenvectors and eigenvalues~\cite{algSP0,folland2016course,deitmar2014principles}. 

We will highlight specially the role of the filters when a representation is compared to its perturbed version. In particular, we will show that there are essentially two factors that can cause  differences between operators and their perturbed versions, the eigenvalues\footnote{As we will show later, this is indeed a particular case of a general notion of homomorphism between the algebra and an irreducible subrepresentation of $\mathcal{A}$.} and eigenvectors. Additionally, we show how the algebra can only affect one of those sources. This is consistent with the fact that differences in the eigenvectors of the operators only affect the constants that are associated to the stability bounds.

We start introducing the notion of subrepresentation.

\begin{definition}
Let $(\mathcal{M},\rho)$ be a representation of $\mathcal{A}$. Then, a representation $(\mathcal{U},\rho)$ of $\mathcal{A}$ is a \textbf{subrepresentation} of  $(\mathcal{M},\rho)$ if $\mathcal{U}\subseteq\mathcal{M}$ and $\mathcal{U}$ is invariant under all operators $\rho(a)~\forall~a\in\mathcal{A}$, i.e. $\rho(a)u\in\mathcal{U}$ for all $u\in\mathcal{U}$ and $a\in\mathcal{A}$. A representation $(\mathcal{M}\neq 0,\rho)$ is \textbf{irreducible} or simple if the only subrepresentations of $(\mathcal{M}\neq 0,\rho)$ are $(0,\rho)$ and $(\mathcal{M},\rho)$.
\end{definition}
The class of irreducible representations of an algebra $\mathcal{A}$ is denoted by $\text{Irr}\{\mathcal{A}\}$. Notice that the zero vector space and $\mathcal{M}$ induce themselves subrepresentations of $(\mathcal{M},\rho)$. In order to state a comparison between representations the concept of \textit{homomorphism between representations} is introduced in the following definition.
\begin{definition}\label{def:homrep}
Let $(\mathcal{M}_{1},\rho_{1})$ and $(\mathcal{M}_{2},\rho_{2})$ be two representations of an algebra $\mathcal{A}$. A homomorphism or \textbf{interwining operator} $\phi:\mathcal{M}_{1}\rightarrow\mathcal{M}_{2}$ is a linear operator which commutes with the action of $\mathcal{A}$, i.e.
\begin{equation}
\phi(\rho_{1}(a)v)=\rho_{2}(a)\phi(v).
\label{eq:interwinop}
\end{equation}
A homomorphism $\phi$ is said to be an isomorphism of representations if it is an isomorphism of vectors spaces. 
\end{definition}
Notice from definition~\ref{def:homrep}  a substantial difference between the concepts of isomorphism of vector spaces and  isomorphism of representations. In the first case we can consider that two arbitrary vector spaces of the same dimension (finite) are isomorphic, while for representations that condition is required but still the condition in eqn.~(\ref{eq:interwinop}) must be satisfied. For instance, as pointed out in~\cite{repthysmbook} all the irreducible 1-dimensional representations of the polynomial algebra $\mathbb{C}[t]$ are non isomorphic.

As we have discussed before, the vector space $\mathcal{M}$ associated to $(\mathcal{M},\rho)$  provides the space where the signals are modeled. Therefore, it is of central interest to determine whether it is possible or not to \textit{decompose} $\mathcal{M}$ in terms of simpler or smaller spaces consistent with the action of $\rho$.  We remark that for any two representations $(\mathcal{M}_{1},\rho_{1})$ and $(\mathcal{M}_{2},\rho_{2})$ of an algebra $\mathcal{A}$, their direct sum is given by the representation $(\mathcal{M}_{1}\oplus\mathcal{M}_{2},\rho)$ where $\rho(a)(\mathbf{x}_{1}\oplus\mathbf{x}_{2})=(\rho_{1}(a)\mathbf{x}_{1}\oplus\rho_{2}(a)\mathbf{x}_{2})$. We introduce the concept of indecomposability in the following definition.
\begin{definition}\label{def:indecomprep}
A nonzero representation $(\mathcal{M},\rho)$ of an algebra $\mathcal{A}$ is said to be \textbf{indecomposable} if it is not isomorphic to a direct sum of two nonzero representations.
\end{definition}
Indecomposable representations provide the \textit{minimum units of information} that can be extracted from signals in a given space when the filters have a specific structure (defined by the algebra)~\cite{barot2014introduction}. The following theorem provides the basic building block for the decomposition of finite dimensional representations.
\begin{theorem}[Krull-Schmit, \cite{repthybigbook}]\label{theorem:krullschmitdecomp}
Any finite dimensional representation of an algebra can be decomposed into a finite direct sum of indecomposable subrepresentations and this decomposition is unique up to the order of the summands and up to isomorphism. 
\end{theorem}
The uniqueness in this result means that if $(\oplus_{r=1}^{r}V_{i},\rho)\cong(\oplus_{j=1}^{s}W_{j},\gamma)$ for indecomposable representations $(V_{j},\rho_{j}),(W_{j},\gamma_{j})$, then $r=s$ and there is a permutation $\pi$ of the indices such that $(V_{i},\rho_{i})\cong (W_{\pi(j)},\gamma_{\pi(j)})$~\cite{repthybigbook}. Although theorem~\ref{theorem:krullschmitdecomp} provides the guarantees for the decomposition of representation in terms of indecomposable representations, it is not applicable when infinite dimensional representations are considered. However, it is possible to overcome this obstacle taking into account that irreducible representations are indecomposable~\cite{repthysmbook,repthybigbook}, and they can be used then to build representations that are indecomposable. In particular, \textit{irreducibility} plays a central role to decompose the invariance properties of the images of $\rho$ on $\text{End}(\mathcal{M})$~\cite{repthybigbook}. Representations that allow a decomposition in terms of subrepresentations that are irreducible are called \textit{completely reducible} and its formal description is presented in the following definition.
\begin{definition}[~\cite{repthybigbook}]\label{def:complredu}
 A representation $(\mathcal{M},\rho)$ of the algebra $\mathcal{A}$ is said to be \textbf{completely reducible} if $(\mathcal{M},\rho)=\bigoplus_{i\in I}(\mathcal{U}_{i},\rho_{i})$ with irreducible subrepresentations $(\mathcal{U}_{i},\rho_{i})$. The \textbf{length} of $(\mathcal{M},\rho)$ is given by $\vert I\vert$.
\end{definition}

For a given $(\mathcal{U},\rho_{\mathcal{U}})\in\text{Irr}\{\mathcal{A}\}$ the sum of all irreducible subrepresentations of $(V,\rho_{V})$ that are equivalent (isomorphic) to $(\mathcal{U},\rho_{\mathcal{U}})$ is represented by $V(\mathcal{U})$ and it is called the $\mathcal{U}$-homogeneous component of $(V,\rho_{V})$. This sum is a direct sum, therefore it has a length that is well defined and whose value is called the \textit{multiplicity} of $(\mathcal{U},\rho_{\mathcal{U}})$ and is represented by $m(\mathcal{U},V)$~\cite{repthybigbook}. Additionally, the sum of all irreducible subrepresentations of $(V,\rho_{V})$ will be denoted as $\text{soc}\{V\}$. It is possible to see that a given representation $(V,\rho_{V})$ is completely reducible if and only if $(V,\rho_{V})=\text{soc}\{S\}$~\cite{repthybigbook}. The connection between $\text{soc}\{V\}$ and $V(\mathcal{U})$ is given by the following proposition. 
\begin{proposition}[Proposition 1.31~\cite{repthybigbook}]\label{prop:repdirectsum1}
Let $(V,\rho_{V})\in\mathsf{Rep}\{A\}$. Then $\text{soc}\{V\}=\bigoplus_{S\in\text{Irr}\{\mathcal{A}\}}V(S)$.
\end{proposition}
Now, taking into account that any homogeneous component $V(\mathcal{U})$ is itself a direct sum we have that 
\begin{equation}
\text{soc}\{V\}\cong\bigoplus_{S\in\text{Irr}\{\mathcal{A}\}}S^{\oplus m(\mathcal{U},V)}.
\label{eq:repdirectsumtot}
\end{equation}
Equation~(\ref{eq:repdirectsumtot}) provides the building block for the definition of Fourier decompositions in algebraic signal processing~\cite{algSP1}. With all these concepts at hand we are ready to introduce the following definition.
\begin{definition}[Fourier Decomposition]\label{def:foudecomp}
For an algebraic signal model $(\mathcal{A},\mathcal{M},\rho)$ we say that there is a spectral or Fourier decomposition of $(\mathcal{M},\rho)$ if
\begin{equation}
(\mathcal{M},\rho)\cong\bigoplus_{(\mathcal{U}_{i},\phi_{i})\in\text{Irr}\{\mathcal{A}\}}(\mathcal{U}_{i},\phi_{i})^{\oplus m(\mathcal{U}_{i},\mathcal{M})}
 \label{eq:foudecomp1}
\end{equation}
where the $(\mathcal{U}_{i},\phi_{i})$ are irreducible subrepresentations of $(\mathcal{M},\rho)$. Any signal $\mathbf{x}\in\mathcal{M}$ can be therefore represented by the map $\Delta$ given by 
\begin{equation}
\begin{matrix} 
\Delta: & \mathcal{M} \to \bigoplus_{(\mathcal{U}_{i},\phi_{i})\in\text{Irr}\{\mathcal{A}\}}(\mathcal{U}_{i},\phi_{i})^{\oplus m(\mathcal{U}_{i},\mathcal{M})} \\ & \mathbf{x}\mapsto \hat{\mathbf{x}}
 \end{matrix}
 \label{eq:foudecomp2}
\end{equation}
known as the Fourier decomposition of $\mathbf{x}$ and the projection of $\hat{\mathbf{x}}$ in each $\mathcal{U}_{i}$ are the Fourier components represented by $\hat{\mathbf{x}}(i)$.
\end{definition}
Notice that in eqn.~(\ref{eq:foudecomp1}) there are two sums, one dedicated to the non isomorphic subrepresentations (external) and another one (internal) dedicated to subrepresentations that are isomorphic. In this context, the sum for non isomorphic representations indicates the sum on the \textit{frequencies} of the representation while the sum for isomorphic representations a sum of components associated to a given frequency. It is also worth pointing out that $\Delta$ is an interwining operator, therefore, we have that $\Delta(\rho(a)\mathbf{x})=\rho(a)\Delta(\mathbf{x})$. As pointed out in~\cite{algSP0} this can be used to define a convolution operator as $\rho(a)\mathbf{x}=\Delta^{-1}(\rho(a)\Delta(\mathbf{x}))$. The projection of a filtered signal $\rho(a)\mathbf{x}$ on each $\mathcal{U}_{i}$ is given by $\phi_{i}(a)\hat{\mathbf{x}}(i)$ and the collection of all this projections is known as the \textit{spectral representation} of the operator $\rho(a)$. Notice that $\phi_{i}(a)\hat{\mathbf{x}}(i)$ translates to different operations depending on the dimension of $\mathcal{U}_{i}$. For instance, if $\text{dim}(\mathcal{U}_{i})=1$, $\hat{\mathbf{x}}(i)$ and $\phi_{i}(a)$ are scalars while if $\text{dim}(\mathcal{U}_{i})>1$ and finite $\phi_{i}(a)\hat{\mathbf{x}}(i)$ is obtained as a matrix product.
\begin{remark}\label{remark:interpSpecFilt}\normalfont
The spectral representation of an operator indicated as $\phi_{i}(a)\hat{\mathbf{x}}(i)$ and eqns.~(\ref{eq:foudecomp1}) and~(\ref{eq:foudecomp2}) highlight one important fact that is essential for the discussion of the results in Section~\ref{sec:proofofTheorems}. For a completely reducible representation $(\mathcal{M},\rho)\in\mathsf{Rep}\{\mathcal{A}\}$ the connection between the algebra $\mathcal{A}$ and the spectral representation is \textit{exclusively} given by $\phi_{i}(a)$ which is acting on $\hat{\mathbf{x}}(i)$, therefore, it is not possible by the selection of elements or subsets of the algebra to do any modification on the spaces $\mathcal{U}_{i}$ associated to the irreducible components in eqn.(\ref{eq:foudecomp1}). As a consequence, when measuring the similarities between two operators $\rho(a)$ and $\tilde{\rho}(a)$ associated to $(\mathcal{M},\rho)$ and $(\mathcal{M},\tilde{\rho})$, respectively, there will be two sources of error. One source of error that can be modified by the selection of $a\in\mathcal{A}$ and another one that will be associated with the differences between spaces $\mathcal{U}_{i}$ and $\tilde{\mathcal{U}}_{i}$, which are associated to the direct sum decomposition of $(\mathcal{M},\rho)$ and $(\mathcal{M},\tilde{\rho})$, respectively. This point was first elucidated in~\cite{fern2019stability} for the particular case of GNNs, but it is part of a much more general statement that becomes more clear in the language of algebraic signal processing. 
\end{remark}
\begin{example}[Discrete signal processing]\normalfont
In CNNs the filtering is defined by the polynomial algebra $\mathcal{A}=\mathbb{C}[t]/(t^{N}-1)$, therefore, in a given layer the spectral representation of the filters is given by
\begin{multline*}
\rho(a)\mathbf{x}=\sum_{i=1}^{N}\phi_{i}\left(\sum_{k=0}^{K-1}h_{k}t^{k}\right)\hat{\mathbf{x}}(i)\mathbf{u}_{i}\\
=\sum_{i=1}^{N}\sum_{k=0}^{K-1}h_{k}\phi_{i}(t)^{k}\hat{\mathbf{x}}(i)\mathbf{u}_{i}
=\sum_{i=1}^{N}\sum_{k=0}^{K-1}h_{k}\left(e^{-\frac{2\pi ij}{N}}\right)^{k}\hat{\mathbf{x}}(i)\mathbf{u}_{i},
\end{multline*}
with $a=\sum_{k=0}^{K-1}h_{k}t^{k}$ and where the $\mathbf{u}_{i}(v)=\frac{1}{\sqrt{N}}e^{\frac{2\pi jvi}{N}}$ are the column vectors of the traditional DFT matrix, while $\phi_{i}(t)=e^{-\frac{2\pi ji}{N}}$ is the eigenvalue associated to $\mathbf{u}_{i}$. Here $\hat{\mathbf{x}}$ represents the DFT of $\mathbf{x}$.
\end{example}
\begin{example}[Graph signal processing]\normalfont
Taking into account that the filtering in each layer of a GNN is defined by a polynomial algebra, the spectral representation of the filter is given by
\begin{multline}
\rho(a)\mathbf{x}=\sum_{i=1}^{N}\phi_{i}\left(\sum_{k=0}^{K-1}h_{k}t^{k}\right)\hat{\mathbf{x}}(i)\mathbf{u}_{i}\\
=\sum_{i=1}^{N}\sum_{k=0}^{K-1}h_{k}\phi_{i}(t)^{k}\hat{\mathbf{x}}(i)\mathbf{u}_{i}=
\sum_{i=1}^{N}\sum_{k=0}^{K-1}h_{k}\lambda_{i}^{k}\hat{\mathbf{x}}(i)\mathbf{u}_{i}
\label{eq:specopgnn}
\end{multline}
with $a=\sum_{k=0}^{K-1}h_{k}t^{k}$, and where the $\mathbf{u}_{i}$ are given by the eigenvector decomposition of $\rho(t)=\mathbf{S}$, where $\mathbf{S}$ could be the adjacency matrix or the Laplacian of the graph, while $\phi_{i}(t)=\lambda_{i}$ being $\lambda_{i}$ the eigenvalue associated to $\mathbf{u}_{i}$. The projection of $\mathbf{x}$ in each subspace $\mathcal{U}_{i}$ is given by $\hat{\mathbf{x}}(i)=\langle\mathbf{u}_{i}, \mathbf{x}\rangle$, and if $\mathbf{U}$ is the matrix of eigenvectors of $\mathbf{S}$ we have the widely known representation $\hat{\mathbf{x}}=\mathbf{U}^{T}\mathbf{x}$~\cite{gamagnns}.
\end{example}
\begin{example}[Group signal processing]\label{example:specopGroupCNNs}
\normalfont
Considering the Fourier decomposition on  general groups~\cite{steinbergrepg,terrasFG,fulton1991representation}, we obtain the spectral representation of the algebraic filters as
\begin{equation*}
\boldsymbol{a}\ast\mathbf{x}=\sum_{u,h\in G}\boldsymbol{a}(uh^{-1})\sum_{i,j,k}\frac{d_{k}}{\vert G\vert}\hat{\mathbf{x}}\left(\boldsymbol{\varphi}^{(k)}\right)_{i,j}\boldsymbol{\varphi}_{i,j}^{(k)}(h)hu,
\end{equation*}
where $\hat{\mathbf{x}}(\boldsymbol{\varphi}^{(k)})$ represents the Fourier components associated to the $k$th irreducible representation with dimension $d_{k}$ and $\boldsymbol{\varphi}^{(k)}$ is the associated unitary element. We can see that the $k$th element in this decomposition is $\sum_{i,j}\mathbf{x}(\varphi^{(k)})_{i,j}\sum_{u,h}\frac{d_{k}}{\vert G\vert}\boldsymbol{a}(uh^{-1})\boldsymbol{\varphi}_{i,j}^{(k)}(h)hu$.
\end{example}
\begin{example}[Graphon signal processing]\normalfont
According to the spectral theorem~\cite{conway1994course,aliprantis2002invitation}, it is possible to represent the action of a compact normal operator $S$ as $S\mathbf{x}=\sum_{i}\lambda_{i}\langle\boldsymbol{\varphi}_{i},\mathbf{x}\rangle\boldsymbol{\varphi}_{i}$ where $\lambda_{i}$ and $\boldsymbol{\varphi}_{i}$ are the eigenvalues and eigenvectors of $S$, respectively, and $\langle\cdot\rangle$ indicates an inner product. Then, the spectral representation of the filtering of a signal in the layer $\ell$ is given by
\begin{equation*}
\rho_{\ell}\left(p(t)\right)\mathbf{x}
=
\sum_{i}p(\lambda_{i})\langle\mathbf{x},\boldsymbol{\varphi}_{i}\rangle\boldsymbol{\varphi}_{i}=\sum_{i}\phi_{i}(p(t))\hat{\mathbf{x}}_{i}\boldsymbol{\varphi}_{i},
\end{equation*}
where $\phi_{i}(p(t))=p(\lambda_{i})$.
\end{example}

%
%

\section{Proof of Theorems}\label{sec:proofofTheorems}

Let us start defining some notation. Let $\boldsymbol{\pi}_{a_{1},\ldots,a_{r}}(\mathbf{A}_{1},\ldots,\mathbf{A}_{r})$ be the operator that represents the sum of all the products of the operators $\mathbf{A}_{1},\ldots,\mathbf{A}_{r}$ that appear $a_{1}, a_{2}, \ldots, a_{r}$ times respectively. For instance, $\boldsymbol{\pi}_{2,1}(\mathbf{\mathbf{A}},\mathbf{B})=\mathbf{AAB}+\mathbf{ABA}+\mathbf{BAA}$. Additionally, when considering all summation and product symbols the following convention is used $\sum_{i=a}^{b}F(i)=0$ if $b<a$, and $\prod_{i=a}^{b}F(i)=0$ if $b<a$. In what follows $\Vert\cdot\Vert$ represents the $\ell_2$ norm and $\Vert\cdot\Vert_{F}$ the Frobenius norm.
%
%
\subsection{Proof of Theorems~\ref{theorem:HvsFrechet} and~\ref{theorem:HvsFrechetmultgen}}
\label{prooftheoremHvsFrechet}
\begin{proof}
We say that $p(\mathbf{S})$ as a function of $\mathbf{S}$ is Fr\'echet differentiable at $\mathbf{S}$ if there exists a bounded linear operator $D_{p}:\text{End}(\mathcal{M})^{m}\rightarrow\text{End}(\mathcal{M})$  such that~\cite{benyamini2000geometric,lindenstrauss2012frechet}
\begin{equation}\label{eq:FreDerdef1}
\lim_{\Vert\boldsymbol{\xi}\Vert\to 0}\frac{\left\Vert p(\mathbf{S}+\boldsymbol{\xi})-p(\mathbf{S})-D_{p}(\mathbf{S})\left\lbrace\boldsymbol{\xi}\right\rbrace\right\Vert}{\Vert\boldsymbol{\xi}\Vert}=0
\end{equation}
which in Landau notation can be written as
\begin{equation}
p(\mathbf{S}+\boldsymbol{\xi})=p(\mathbf{S})+D_{p}(\mathbf{S})\left\lbrace\boldsymbol{\xi}\right\rbrace+o(\Vert\boldsymbol{\xi}\Vert).
\label{eq:frechetDdef}
\end{equation}
Calculating the norm in eqn.~(\ref{eq:frechetDdef}) and applying the triangle inequality we have:
\begin{equation}
\left\Vert p(\mathbf{S}+\boldsymbol{\xi})-
p(\mathbf{S})\right\Vert\leq
\left\Vert D_{p}(\mathbf{S})\left\lbrace\boldsymbol{\xi}\right\rbrace\right\Vert+\mathcal{O}\left(\Vert\boldsymbol{\xi}\Vert^{2}\right)    
\end{equation}
for all $\boldsymbol{\xi}=(\boldsymbol{\xi}_{1},\ldots,\boldsymbol{\xi}_{m})\in\text{End}(\mathcal{M})^{m}$. Now, taking into account the properties of a Fr\'echet derivative for a function of multiple variables (see~\cite{berger1977nonlinearity} pages 69-70) we have
\begin{equation}
\Vert D_{p}(\mathbf{S})\left\lbrace\boldsymbol{\xi}\right\rbrace\Vert\leq 
\sum_{i=1}^{m}\left\Vert D_{p\vert\mathbf{S}_{i}}(\mathbf{S})\left\lbrace\boldsymbol{\xi}_{i}\right\rbrace\right\Vert
\end{equation}
and therefore
\begin{equation*}
\left\Vert p(\mathbf{S}+\boldsymbol{\xi})-
p(\mathbf{S})\right\Vert\leq
\sum_{i=1}^{m}\left\Vert D_{p\vert\mathbf{S}_{i}}(\mathbf{S})\left\lbrace\boldsymbol{\xi}_{i}\right\rbrace\right\Vert+\mathcal{O}\left(\Vert\boldsymbol{\xi}\Vert^{2}\right),
\end{equation*}
where $D_{p\vert\mathbf{S}_{i}}(\mathbf{S})$ is the partial Frechet derivative of $p(\mathbf{S})$ on $\mathbf{S}_{i}$. Then, taking into account that
\begin{equation}
    \left\Vert p(\mathbf{S}+\boldsymbol{\xi})\mathbf{x}-
p(\mathbf{S})\mathbf{x}\right\Vert\leq\Vert\mathbf{x}\Vert\left\Vert p(\mathbf{S}+\boldsymbol{\xi})-
p(\mathbf{S})\right\Vert
\end{equation}
%
%
%
and selecting $\boldsymbol{\xi}_{i}=\mathbf{T}(\mathbf{S}_{i})$ we complete the proof. 
\end{proof} 
%
\subsection{Proof of Theorem~\ref{theorem:uppboundDH} and Theorem~\ref{theorem:uppboundDHmultg}}
\label{prooftheorem:uppboundDH}
\begin{proof}
Taking into account the definition of the Fr\'echet derivative of $p$ on $\mathbf{S}_{i}$ (see Appendix~\ref{appendix:howtofindDH}) we have
\begin{equation*}
\left\Vert D_{p\vert\mathbf{S}_{i}}(\mathbf{S})\left\lbrace\mathbf{T}(\mathbf{S}_{i})\right\rbrace\right\Vert
                  =
                     \left\Vert\sum_{k_{i}=1}^{\infty}\mathbf{A}_{k_{i}}\boldsymbol{\pi}_{1,k_{i}-1}\left(\mathbf{T}(\mathbf{S}_{i}),\mathbf{S}_{i}\right)
                     \right\Vert, 
\end{equation*}
and re-organizating terms we have
\begin{equation}
\left\Vert D_{p\vert\mathbf{S}_{i}}(\mathbf{S})\left\lbrace\mathbf{T}(\mathbf{S}_{i})\right\rbrace\right\Vert=\left\Vert\sum_{\ell=1}^{\infty}\mathbf{S}_{i}^{\ell-1}\mathbf{T}(\mathbf{S}_{i})\sum_{k_{i}=\ell}^{\infty}\mathbf{A}_{k_{i}}\mathbf{S}_{i}^{k_{i}-\ell}\right\Vert.
\end{equation}
Taking into account eqn.~(\ref{eq:PrvsTr}), it follows that
\begin{multline}
\left\Vert D_{p\vert\mathbf{S}_{i}}(\mathbf{S})\left\lbrace\mathbf{T}(\mathbf{S}_{i})\right\rbrace\right\Vert=
\\
\left\Vert\sum_{\ell=1}^{\infty}\left(\mathbf{T}_{0c,i}\mathbf{S}_{i}^{\ell-1}+\mathbf{S}^{\ell-1}_{i}\mathbf{P}_{0,i}\right)\sum_{k_{i}=\ell}^{\infty}\mathbf{A}_{k_{i}}\mathbf{S}_{i}^{k_{i}-\ell}\right.
\\
\left.+\sum_{\ell=1}^{\infty}\left(\mathbf{T}_{1c,i}\mathbf{S}_{i}^{\ell}
+\mathbf{S}_{i}^{\ell-1}\mathbf{P}_{1,i}\mathbf{S}_{i}\right)
\sum_{k=\ell}^{\infty}\mathbf{A}_{k_{i}}\mathbf{S}^{k_{i}-\ell}\right\Vert.
\label{eq:auxDHT1}
\end{multline}
Applying the triangle inequality and distribuiting the sum we have
\begin{multline}\label{eq:DpDTboundeqn1}
\left\Vert D_{p\vert\mathbf{S}_{i}}(\mathbf{S})\left\lbrace\mathbf{T}(\mathbf{S}_{i})\right\rbrace\right\Vert
\leq
\left\Vert\mathbf{T}_{0c,i}\sum_{\ell=1}^{\infty}\sum_{k_{i}=\ell}^{\infty}\mathbf{S}_{i}^{k_{i}-1}\mathbf{A}_{k_{i}}\right\Vert
\\
+\left\Vert D_{p\vert\mathbf{S}_{i}}(\mathbf{S})\left\lbrace\mathbf{P}_{0,i}\right\rbrace\right\Vert
+\left\Vert\mathbf{T}_{1c,i}\sum_{\ell=1}^{\infty}\sum_{k_{i}=\ell}^{\infty}\mathbf{S}^{k_{i}}\mathbf{A}_{k_{i}}\right\Vert
\\
+\left\Vert D_{p\vert\mathbf{S}_{i}}(\mathbf{S})\left\lbrace\mathbf{P}_{1,i}\mathbf{S}_{i}\right\rbrace\right\Vert
\end{multline}
Now, we analyze term by term in eqn.~(\ref{eq:DpDTboundeqn1}). For the first term we take into account that

\begin{equation}
\sum_{\ell=1}^{\infty}\sum_{k_{i}=\ell}^{\infty}\mathbf{S}_{i}^{k_{i}-1}\mathbf{A}_{k_{i}}=\sum_{k_{i}=1}^{\infty}k_{i}\mathbf{A}_{k_{i}}\mathbf{S}_{i}^{k_{i}-1}    
\end{equation}
and we apply the product norm property taking into account that the filters belong to $\mathcal{A}_{L_{0}}$, which leads to

\begin{multline}
\left\Vert\mathbf{T}_{0c,i}\sum_{\ell=1}^{\infty}\sum_{k_{i}=\ell}^{\infty}\mathbf{S}_{i}^{k_{i}-1}\mathbf{A}_{k_{i}}\right\Vert
\\
\leq
\left\Vert\mathbf{T}_{0c,i}\right\Vert\left\Vert\sum_{k_{i}=1}^{\infty}k_{i}\mathbf{A}_{k_{i}}\mathbf{S}_{i}^{k_{i}-1}\right\Vert\leq L_{0}\Vert\mathbf{T}_{0,i}\Vert.
\end{multline}

For the second term in eqn.~(\ref{eq:DpDTboundeqn1}) we take into account that the Fr\'echet derivative acting on $\bbP_{0,i}$ can be equivalently expressed as a linear operator acting on the left of a vectorized version of $\bbP_{0,i}$ (see~\cite{higham2008functions} pages 61 and 331). Then,

\begin{equation}
\left\Vert D_{p\vert\mathbf{S}_{i}}(\mathbf{S})\left\lbrace\mathbf{P}_{0,i}\right\rbrace\right\Vert
\leq L_{0}\Vert\mathbf{P}_{0,i}\Vert_{F}    
\end{equation}
and with the fact that $\Vert\mathbf{P}_{0,i}\Vert_{F}\leq\delta\Vert\mathbf{T}_{0,i}\Vert$, we have 

\begin{equation}
\left\Vert D_{p\vert\mathbf{S}_{i}}(\mathbf{S})\left\lbrace\mathbf{P}_{0,i}\right\rbrace\right\Vert \leq L_{0}\delta\Vert\mathbf{T}_{0,i}\Vert.    
\end{equation}
%
%


For the third term in eqn.~(\ref{eq:DpDTboundeqn1}), we take into account that 

\begin{equation}
\sum_{\ell=1}^{\infty}\sum_{k_{i}=\ell}^{\infty}\mathbf{S}_{i}^{k_{i}}\mathbf{A}_{k_{i}}=\sum_{k_{i}=1}^{\infty}k_{i}\mathbf{A}_{k_{i}}\mathbf{S}_{i}^{k_{i}},    
\end{equation}
and we apply the norm product property taking into account that the filters belong to $\mathcal{A}_{L_{1}}$, which leads to 

%
%
\begin{multline}
\left\Vert\mathbf{T}_{1c,i}\sum_{\ell=1}^{\infty}\sum_{k_{i}=\ell}^{\infty}\mathbf{S}_{i}^{k_{i}}\mathbf{A}_{k_{i}}\right\Vert
\\
\leq
\left\Vert\mathbf{T}_{1c,i}\right\Vert\left\Vert\sum_{k_{i}=1}^{\infty}k_{i}\mathbf{A}_{k_{i}}\mathbf{S}_{i}^{k_{i}}\right\Vert\leq L_{1}\Vert\mathbf{T}_{1,i}\Vert.
\end{multline}

Finally, for the fourth term we use the notation $\tilde{D}(\mathbf{S})\left\lbrace\mathbf{P}_{1,i}
\right\rbrace
=D_{p\vert\mathbf{S}_{i}}(\mathbf{S})\left\lbrace\mathbf{P}_{1,i}\mathbf{S}_{i}\right\rbrace$. We start pointing out that  (see~\cite{higham2008functions} pages 61 and 331) the eigenvalues of the operator $\tilde{D}(\mathbf{S})$ represented as $\zeta_{pq}$ are given by

\begin{equation}
\zeta_{pq}=\left\lbrace
\begin{array}{ccc}
\frac{p(\lambda_{p})-p(\lambda_{q})}{\lambda_{p}-\lambda_{q}}\lambda_{q} & \text{if} & \lambda_{p}\neq\lambda_{q}\\
\lambda_{p}p^{'}(\lambda_{p}) & \text{if} & \lambda_{p}=\lambda_{q}
\end{array}
\right. .
\end{equation}

Then, taking into account that the filters belong to $\mathcal{A}_{L_{1}}$ we have $\Vert\tilde{D}(\mathbf{S})\Vert\leq L_{1}$, and therefore

\begin{equation}
 \left\Vert
D_{p\vert\mathbf{S}_{i}}(\mathbf{S})\left\lbrace\mathbf{P}_{1,i}\mathbf{S}_{i}\right\rbrace\right\Vert=\left\Vert\tilde{D}(\mathbf{S})\left\lbrace\mathbf{P}_{1,i}
\right\rbrace\right\Vert\leq L_{1}\Vert\mathbf{P}_{1,i}\Vert_{F}   
\end{equation}
%
%


Additionally, with $\Vert\mathbf{P}_{1,i}\Vert_{F}\leq\delta\Vert\mathbf{T}_{1,i}\Vert$ it follows that

\begin{equation}
\left\Vert
D_{p\vert\mathbf{S}_{i}}(\mathbf{S})\left\lbrace\mathbf{P}_{1,i}\mathbf{S}_{i}\right\rbrace\right\Vert\leq L_{1}\delta\Vert\mathbf{T}_{1,i}\Vert    
\end{equation}
%
%


Putting all these results together into eqn.~(\ref{eq:DpDTboundeqn1}) we reach
\begin{multline*}
\left\Vert D_{p\vert\mathbf{S}_{i}}(\mathbf{S})\left\lbrace\mathbf{T}(\mathbf{S}_{i})\right\rbrace\right\Vert
\leq
(1+\delta)L_{0}\Vert\mathbf{T}_{0,i}\Vert+(1+\delta)L_{1}\Vert\mathbf{T}_{1,i}\Vert
\\
\leq
(1+\delta)\left(L_{0}\sup_{\mathbf{S}_{i}\in\mathcal{S}}\Vert\mathbf{T}(\mathbf{S}_{i})\Vert+L_{1}\sup_{\mathbf{S}_{i}\in\mathcal{S}}\Vert D_{\mathbf{T}}(\mathbf{S}_{i})\Vert\right)
\end{multline*}
\end{proof}
%
%
\subsection{Proof of Theorems~\ref{theorem:stabilityAlgNN0} and~\ref{theorem:stabilityAlgNN1}}
\subsubsection{Proof of Theorem~\ref{theorem:stabilityAlgNN0}}
\label{prooftheorem:stabilityAlgNN0}
\begin{proof}
Taking into account eqns.~(\ref{eq:xl}), and~(\ref{eq:interlayeropalgnn}) and the fact that the maps $\sigma_{\ell}$ are Lipschitz with constant $C_{\ell}$ we have that
\begin{equation}
\left\Vert 
\sigma_{\ell}\left(p(\mathbf{S}_{\ell})\mathbf{x}_{\ell-1}\right)-\sigma_{\ell}\left(p(\tilde{\mathbf{S}}_{\ell})\mathbf{x}_{\ell-1}\right)
\right\Vert
\leq
C_{\ell}\boldsymbol{\Delta}_{\ell}\Vert\mathbf{x}_{\ell-1}\Vert,
\end{equation}
where $\boldsymbol{\Delta}_{\ell}=\Vert p(\mathbf{S}_{\ell})-p(\tilde{\mathbf{S}}_{\ell})\Vert$, and whose value is determined by theorems~\ref{theorem:HvsFrechet} and~\ref{theorem:uppboundDH}. 
\end{proof}

\subsubsection{Proof of Theorem~\ref{theorem:stabilityAlgNN1}}
\label{prooftheorem:stabilityAlgNN1}
\begin{proof}
Before starting the calculations let us introduce some notation. Let $\varphi_{f}(\ell,g)=\rho_{\ell}\left(\xi^{fg}\right)$ denote the image of the filter $\xi^{fg}\in\mathcal{A}_{\ell}$ that process the $g$th feature coming from the layer $\ell-1$  and that is associated to $f$th feature in layer $\ell$. As indicated before, $\sigma_{\ell}$ indicates the Lipschitz mapping from layer $\ell$ to layer $\ell+1$. The term $\mathbf{x}_{\ell-1}^{g}$ indicates the $g$th feature in the layer $\ell-1$. Then, we have that: 
\begin{multline}
\left\Vert\mathbf{x}_{\ell}^{f}-\tilde{\mathbf{x}}_{\ell}^{f}\right\Vert\leq
\left\Vert\sigma_{\ell-1}\sum_{g_{\ell-1}}\varphi_{g_{\ell}}(\ell-1,g_{\ell-1})\sigma_{\ell-2}\right.\\
\left.\sum_{g_{\ell-2}}\varphi_{g_{\ell-1}}(\ell-2,g_{\ell-2})
\cdots\sigma_{1}\sum_{g_{1}}\varphi_{g_{2}}\mathbf{x}\right.-\\
\left. \sigma_{\ell-1}\sum_{g_{\ell-1}}\tilde{\varphi}_{g_{\ell}}(\ell-1,g_{\ell-1})\sigma_{\ell-2}\sum_{g_{\ell-2}}\tilde{\varphi}_{g_{\ell-1}}(\ell-2,g_{\ell-2})\right.\\
\left.\cdots\sigma_{1}\sum_{g_{1}}\tilde{\varphi}_{g_{2}}(1,g_{1})\mathbf{x}
\right\Vert.
\label{eq:longalgNNexp1}
\end{multline}
In order to exapand eqn.~(\ref{eq:longalgNNexp1}) we start pointing out that:
\begin{multline}
A_{\ell+1}\sigma_{\ell}(a)-\tilde{A}_{\ell+1}\sigma_{\ell}(\tilde{a})=
\\
(A_{\ell+1}-\tilde{A}_{\ell+1})\sigma_{\ell}(a)+\tilde{A}_{\ell+1}(\sigma_{\ell}(a)-\sigma_{\ell}(\tilde{a}))
\end{multline}
where $A_{\ell+1}$ and $\tilde{A}_{\ell+1}$ indicate filter operators and their perturbed versions, respectively. Now, noticing that $\Vert\sigma_{\ell}(a)-\sigma_{\ell}(b)\Vert\leq C_{\ell}\Vert a-b\Vert$, $\Vert A_{\ell+1}-\tilde{A}_{\ell+1}\Vert\leq\boldsymbol{\Delta}_{\ell}$ and $\Vert A_{\ell+1}\Vert\leq B_{\ell+1}$ we have the following relations
\begin{equation}
\sum_{g_{k}}\Vert\alpha-\tilde{\alpha}\Vert\leq\sum_{g_{k}}\left( 
\boldsymbol{\Delta}_{k}\Vert\sigma_{k-1}(\alpha)\Vert+B_{k}C_{k-1}\Vert\beta-\tilde{\beta}\Vert
\right)
\label{eq:aux_longeq1}
\end{equation}
\begin{equation}
\sum_{g_{k}}\Vert\beta-\tilde{\beta}\Vert\leq\sum_{g_{k}}\sum_{g_{k-1}}\Vert\alpha-\tilde{\alpha}\Vert
\label{eq:aux_longeq2}
\end{equation}
\begin{equation}
\sum_{g_{k}}\Vert\sigma_{k-1}(\alpha)\Vert\leq\left(\prod_{r=1}^{k}F_{r}\right)\left(\prod_{r=1}^{k-1}C_{r}B_{r}\right)\Vert\mathbf{x}\Vert
\label{eq:aux_longeq3}
\end{equation}
where $\alpha$ and $\tilde{\alpha}$ represent sequences of symbols in eqn.~(\ref{eq:longalgNNexp1}) that start with a symbol of the type $\varphi$, while $\beta$ and $\tilde{\beta}$ indicate a sequence of symbols that start with a summation  symbol, and the tilde makes reference to symbols that are associated to the perturbed representations. The term $\boldsymbol{\Delta}_{\ell}$ is associated to the difference between the operators and their perturbed versions (see definition~\ref{def:stabilityoperators1}) in the layer $\ell$ and whose values are given in Theorems~\ref{theorem:HvsFrechet} and~\ref{theorem:uppboundDH}. Combining eqns.~(\ref{eq:aux_longeq1}),~(\ref{eq:aux_longeq2}) and~(\ref{eq:aux_longeq3}) we have:
\begin{multline}
\left\Vert\mathbf{x}_{L}^{f}-\tilde{\mathbf{x}}_{L}^{f}
\right\Vert\leq
\sum_{\ell=1}^{L}\boldsymbol{\Delta}_{\ell}\left(\prod_{r=\ell}^{L}C_{r}\right)\left(\prod_{r=\ell+1}^{L}B_{r}\right)\\
\left(\prod_{r=\ell}^{L-1}F_{r}\right)\left(\prod_{r=1}^{\ell-1}C_{r}F_{r}B_{r}\right)\left\Vert\mathbf{x}\right\Vert,
\end{multline}
where the products $\prod_{r=a}^{b}F(r)=0$ if $b<a$. Now taking into account that 
\begin{multline*}
\left\Vert
\Phi\left(\mathbf{x},\{ \mathcal{P}_{\ell} \}_{1}^{L},\{ \mathcal{S}_{\ell}\}_{1}^{L}\right)-
\Phi\left(\mathbf{x},\{ \mathcal{P}_{\ell} \}_{1}^{L},\{ \tilde{\mathcal{S}}_{\ell}\}_{1}^{L}\right)
\right\Vert^{2} 
\\
=
\sum_{f=1}^{F_{L}}\left\Vert\mathbf{x}_{L}^{f}-\tilde{\mathbf{x}}_{L}^{f}
\right\Vert^{2}
\end{multline*}
we have
\begin{multline}
\left\Vert
\Phi\left(\mathbf{x},\{ \mathcal{P}_{\ell} \}_{1}^{L},\{ \mathcal{S}_{\ell}\}_{1}^{L}\right)-
\Phi\left(\mathbf{x},\{ \mathcal{P}_{\ell} \}_{1}^{L},\{ \tilde{\mathcal{S}}_{\ell}\}_{1}^{L}\right)
\right\Vert
\\
\leq
\sqrt{F_{L}}
\sum_{\ell=1}^{L}\boldsymbol{\Delta}_{\ell}\left(\prod_{r=\ell}^{L}C_{r}\right)\left(\prod_{r=\ell+1}^{L}B_{r}\right)
\left(\prod_{r=\ell}^{L-1}F_{r}\right)
\\
\left(\prod_{r=1}^{\ell-1}C_{r}F_{r}B_{r}\right)\left\Vert\mathbf{x}\right\Vert
\end{multline}
\end{proof}

%
%

%
\section{Discussion}\label{sec:discussion}
%
%
\begin{figure}[t]
    \centering

\colorlet{my_blue}{blue!50!cyan}
\definecolor{my_red}{rgb}{1, 0, 0.4980}

\def \thisplotscale {3.6}
\def \unit {\thisplotscale cm}

\def \frequencyresponse 
     {   0.8*exp(-(1*(x-1.2))^2) 
      -0.2*exp(-(0.5*(x-7))^2)+0.5*exp(-(1*(x))^2)
       + 0.7*exp(-(0.7*(x-4))^2) 
       + 0.8*exp(-(1.4*(x-6))^2) 
       + 0.1}

\hspace{-2.9mm}
\begin{tikzpicture}[x = 1*\unit, y=1*\unit]

\def \factorx {2.4/8}
\def \deltax  {0.5*\factorx}
\def \shadeshift  {0.05}

\path [fill=black!20, opacity = 0.5] 
              (\deltax - 0.001*\factorx - \shadeshift, 0.00) rectangle 
              (\deltax + 0.030*\factorx + \shadeshift, 1.00);
\path [fill=black!20, opacity = 0.5] 
              (\deltax + 3.393*\factorx - \shadeshift, 0.00) rectangle 
              (\deltax + 3.770*\factorx + \shadeshift, 1.00);
\path [fill=black!20, opacity = 0.5] 
              (\deltax + 6.048*\factorx - \shadeshift, 0.00) rectangle 
              (\deltax + 6.720*\factorx + \shadeshift, 1.00);

\begin{axis}[scale only axis,
             width  = 2.4*\unit,
             height = 1*\unit,
             xmin = -0.5, xmax=7.5,
             xtick = {0.03, -0.01, 3.77, 3.393, 6.72, 6.048},
             xticklabels = {\red{$\quad\quad\vert\tilde{\lam}_1\vert\phantom{\lam}$},
                            \blue{$\vert\lam_1\vert\quad\quad$}, 
                            \red{$\quad\tilde{\lam}_i\vert\phantom{\lam}$}, 
                            \blue{$\vert\lam_i\vert$},
                            \red{$\quad\vert\tilde{\lam}_{N}\vert\phantom{\lam}$},
                            \blue{$\vert\lam_N\vert$}},
             ymin = -0, ymax = 1.15,
             ytick = {-1},
             typeset ticklabels with strut,
             enlarge x limits=false]

\addplot+[samples at = {0.03, 0.91, 1.57, 
                        2.63, 3.77, 4.51, 
                        5.60, 6.72}, 
          color = my_red, 
          ycomb, 
          mark=otimes*, 
          mark options={my_red}]
         {\frequencyresponse};

\addplot+[samples at = {-0.01, 0.819, 1.413, 
                        2.367, 3.393, 4.059, 
                        5.04, 6.048}, 
          color = my_blue, 
          ycomb, 
          mark=oplus*, 
          mark options={my_blue}]
         {\frequencyresponse};

\addplot[ domain=-0.5:7.5, 
          samples = 80, 
          color = black,
          line width = 1.2]
         {\frequencyresponse};

\end{axis}
\end{tikzpicture}

\medskip

\colorlet{my_blue}{blue!50!cyan}
\definecolor{my_red}{rgb}{1, 0, 0.4980}

\def \thisplotscale {3.6}
\def \unit {\thisplotscale cm}

\def \frequencyresponse 
     { 0.9 - 0.7*exp(-(0.7*(x-1.6))^2)+0.2*exp(-(2*(x-1.8))^2) }

\hspace{-2.9mm}
\begin{tikzpicture}[x = 1*\unit, y=1*\unit]

\def \factorx {2.4/8}
\def \deltax  {0.5*\factorx}
\def \shadeshift  {0.05}

\path [fill=black!20, opacity = 0.5] 
              (\deltax - 0.001*\factorx - \shadeshift, 0.00) rectangle 
              (\deltax + 0.030*\factorx + \shadeshift, 1.00);
\path [fill=black!20, opacity = 0.5] 
              (\deltax + 3.393*\factorx - \shadeshift, 0.00) rectangle 
              (\deltax + 3.770*\factorx + \shadeshift, 1.00);
\path [fill=black!20, opacity = 0.5] 
              (\deltax + 6.048*\factorx - \shadeshift, 0.00) rectangle 
              (\deltax + 6.720*\factorx + \shadeshift, 1.00);

\begin{axis}[scale only axis,
             width  = 2.4*\unit,
             height = 1*\unit,
             xmin = -0.5, xmax=7.5,
             xtick = {0.03, -0.01, 3.77, 3.393, 6.72, 6.048},
             xticklabels = {\red{$\quad\quad\vert\tilde{\lam}_1\vert\phantom{\lam}$},
                            \blue{$\vert\lam_1\vert\quad\quad$}, 
                            \red{$\quad\quad\vert\tilde{\lam}_i\vert\phantom{\lam}$}, 
                            \blue{$\vert\lam_i\vert$},
                            \red{$\quad\quad\vert\tilde{\lam}_{N}\vert\phantom{\lam}$},
                            \blue{$\vert\lam_N\vert$}},
             ymin = -0, ymax = 1.15,
             ytick = {-1},
             typeset ticklabels with strut,
             enlarge x limits=false]

\addplot+[samples at = {0.03, 0.91, 1.57, 
                        2.63, 3.77, 4.51, 
                        5.60, 6.72}, 
          color = my_red, 
          ycomb, 
          mark=otimes*, 
          mark options={my_red}]
         {\frequencyresponse};

\addplot+[samples at = {-0.01, 0.819, 1.413, 
                        2.367, 3.393, 4.059, 
                        5.04, 6.048}, 
          color =my_blue, 
          ycomb, 
          mark=oplus*, 
          mark options={my_blue}]
         {\frequencyresponse};

\addplot[ domain=-0.5:7.5, 
          samples = 80, 
          color = black,
          line width = 1.2]
         {\frequencyresponse};

\end{axis}
\end{tikzpicture}

   \caption{Filter properties and stability for algebraic operators considering algebras with a single generator. (Top) We depict a Lipschitz filter where it is possible to see that an arbitrary degree of selectivity can be achieved in any part of the spectrum. (bottom) We depict a Lipschitz integral filter where we can see how the magnitude of the filters tends to a constant value as the size of $\vert\lambda_{i}\vert$ grows. As a consequence there is no discriminability in one portion of the spectrum.}
 \label{fig_graph_dilation}
\end{figure}
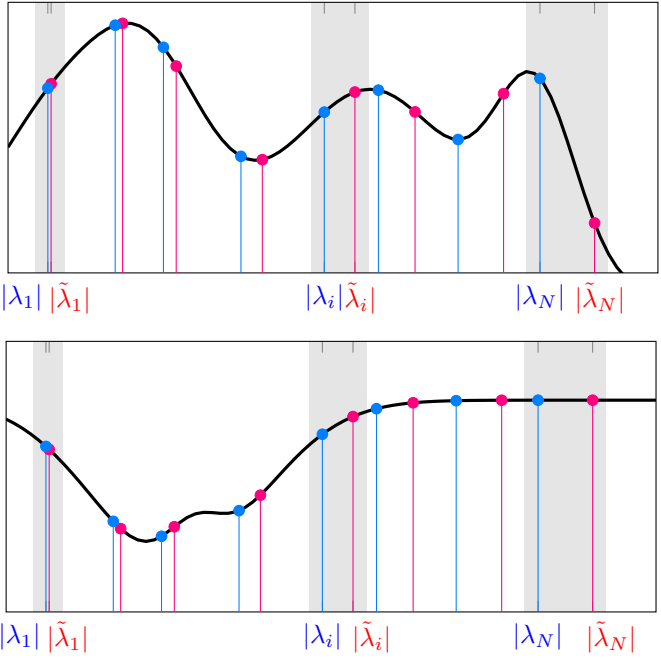
%
The mathematical form of the notion of stability introduced in definition~(\ref{def:stabilityoperators1}), eqn.~(\ref{eq:stabilityoperators1}) is uncannily similar to the expressions associated to the stability conditions stated in~\cite{mallat_ginvscatt,bruna_iscn} when the perturbation operator $\tau$  considered was affecting directly the domain of the signals.  This is consistent with the fact that  the size of the perturbation on the operators is the size of an induced diffeomorphism $\mathbf{T}$ acting on $\text{End}(\mathcal{M})$. Measuring the size of perturbations in this way, although less intuitive, provides an alternative way to handle and interpret perturbations on irregular domains. 

The nature and severity of the perturbations, imposes restrictions on the behavior of the filters needed to guarantee stability. The more complex and severe the perturbation is the more conditions  on the filters are necessary to guarantee stability. This in particular has implications regarding to the selectivity of the filters in some specific frequency bands. The trade-off between stability and selectivity in the filters of the AlgNN can be measured
by the norm of the Fr\'echet derivative of the filters $\Vert D_{p\vert\bbS}(\bbS)\Vert$. Those filters with slow variation and low selectivity will be associated with a low value of $\Vert D_{p\vert\bbS}(\bbS)\Vert$ while a filter that high variation will lead to large values of $\Vert D_{p\vert\bbS}(\bbS)\Vert$. This is also reflected in
 the size of the upper bounds in Theorems~\ref{theorem:HvsFrechet} up to~\ref{theorem:stabilityAlgNN1}. In particular, the size of $L_{0}$ and $L_{1}$ associated to the boundedness of the derivatives of the elements in $\mathcal{A}_{L0}$ and $\mathcal{A}_{L1}$. The smaller the value of $L_{0}, L_{1}$ the more stable the operators but the less selectivity we have. In Fig.~(\ref{fig_graph_dilation}) the properties in frequency of Lipschtiz and Lipschitz integral filters are depicted, where it is possible to see how the selectivity on portions of the spectrum is affected by properties that at the same time provide stability conditions for the perturbation models considered.

It is important to remark that the function $\sigma_{\ell}=P_{\ell}\circ\eta_{\ell}$ composed by the projection operator $P_{\ell}$ and the nonlinearity function $\eta_{\ell}$ relocates information from one layer to the other performing a mapping between different portions of the spectrum associated to each of the spaces $\mathcal{M}_{\ell}$.  As $\eta_{\ell}$ maps elements of $\mathcal{M}_{\ell}$ onto itself,  we can see in light of the decomposition of $\mathcal{M}_{\ell}$ in terms of irreducible representations that 
$\eta_{\ell}$ is nothing but a relocator of information from one portion of the spectrum to the other. Additionally, the simplicity of $\eta_{\ell}$ provides  a rich variety of choices that can be explored in future research. 

The notion of differentiability between metric spaces or Banach spaces can be considered also using the notion of \textit{G\^ateaux derivative} which is considered a \text{weak} notion of differentiability. Although Gateaux differentiability is in general different from Fr\'echet  differentiability, it is possible to show that when $\text{dim}\left(\text{End}(\mathcal{M})\right)<\infty$ both notions are equivalent for Lipschtiz functions, but substantial differences may exist if $\text{dim}\left(\text{End}(\mathcal{M})\right)=\infty$ even if the functions are Lipschitz~\cite{benyamini2000geometric,lindenstrauss2012frechet}.

%
%
%
%
\section{Conclusions}\label{sec:conclusions}

We considered algebraic neural networks (AlgNN) with commutative algebras as a tool to unify convolutional architectures like CNNs and GNNs, synthesizing the algebraic structure by exploiting results from the representation theory of algebras and algebraic signal processing. Within this framework, we showed that AlgNNs can, in general,  be stable to different types of perturbations, and the conditions under which the AlgNN operators are stable are determined by subsets of the algebra. We pointed out that the perturbations of the domain of the signals can be equivalently modeled as a perturbation of the representation or the signal model, and the degree of this perturbation can be measured by means of the Fr\'echet derivative of two functions, the image of the homomorhisms in $\text{End}(\mathcal{M})$ and the perturbation model $\mathbf{T}(\mathbf{S})$.  The perturbation model considered provides enough expressive power to represent a wide variety of perturbations affecting the domain of the signals or the operator themselves directly. In particular, when considering the algebraic model for GNNs, the absolute and the relative perturbation models can be considered particular cases of the perturbation model used in this work.


An interesting and relevant future research direction is to analyze stability of operators in signal models with non commutative algebras. This is important since we do not have shift invariance, and consequently there is the question of how this affects the stability properties and the constants in a stability bound if it exists. Another essential question to solve is how the ASP theory can be extended to consider stability of convolutional operators in signal models where the algebra is not of polynomial type. This has implications when considering convolutions with functions in $L_{2}(\mbR^n)$, where the algebra is $L_1(\mbR^{n})$ and the notion of generator set proposed in~\cite{algSP0} is insufficient/inadequate to capture the whole structure of the algebra.

\appendices

\section{Perturbation Model}
\label{appendix_upperboundPr}


\begin{theorem}\label{prop_commutation_factor} 
Let $\mathbf{T}_{r}, \bbT_{cr}$ as specified in eqn.~(\ref{eq:PrvsTr}) for the perturbation model with $\Vert \bbP_{r}\Vert_{F}\leq \delta\Vert\bbT_{r}\Vert$. Let $\bbe_{i}$ a orthonormal basis, $(\lambda_i,\bbv_{i})$ the eigenpairs of $\bbS$ and $(\mu_i, \bbu_{i})$ the eigenpairs of $\bbT_{r}$. Then

\begin{equation}
\delta\leq\sqrt{N}\hat{\delta}    
\end{equation}
where
\begin{equation}
 \hat{\delta}=\left(
          \Vert\bbT_{\bbu}-\bbT_{\bbv}\Vert
          +1\right)^{2}
          -1   
\end{equation}
and
\begin{equation}
\bbT_{\bbv} =
          \left(
                \sum_{i}\bbe_i \langle \bbv_i, \cdot\rangle
          \right),
          \quad
\bbT_{\bbu} =
          \left(
                \sum_{i}\bbe_i \langle \bbu_i, \cdot\rangle
          \right).          
\end{equation}
The terms $\langle \bbu_i,\cdot\rangle$ and $\langle \bbu_i,\cdot\rangle$ indicate the inner product operators with $\bbu_i$ and $\bbv_i$ respectively.
\end{theorem}

\begin{proof}

For our analysis we consider the following operators

\begin{equation}\label{eq:app_aux0}
\bbT_{\bbv^{\ast}} =
          \left(
                \sum_{i}\bbv_i \langle  \cdot, \bbe_i\rangle
          \right),
          \quad
\bbT_{\bbu^{\ast}} =
          \left(
                \sum_{i}\bbu_i \langle  \cdot, \bbe_i\rangle
          \right),          
\end{equation}
%
%
%
%
%
%
\begin{equation}\label{eq:app_aux2}
\bbT_{\boldsymbol{\mu}} =
          \left(
                \sum_{i}\mu_{i}\bbe_i \langle \bbe_i, \cdot\rangle
          \right),          
\end{equation}

and we remark that

\begin{equation}\label{eq:app_aux3}
\bbT_{r}=\bbT_{\bbu^{\ast}}\bbT_{\boldsymbol{\mu}}\bbT_{\bbu},
\quad
\bbT_{cr}=\bbT_{\bbv^{\ast}}\bbT_{\boldsymbol{\mu}}\bbT_{\bbv}
\end{equation}
with
\begin{equation}\label{app_auxaux1}
\Vert\bbT_{\bbu^{\ast}}\Vert=\Vert\bbT_{\bbu}\Vert=\Vert\bbT_{\bbv^{\ast}}\Vert=\Vert\bbT_{\bbv}\Vert=1    
\end{equation}
and
\begin{equation}\label{app_auxaux2}
\Vert\bbT_{\boldsymbol{\mu}}\Vert=\Vert\bbT_r\Vert.
\end{equation}
Now, we start taking into account that $\bbT_{r}$ can be rewritten as

\begin{multline}\label{eq:app_aux4}
\bbT_{r} =
          \bbT_{\bbv^{\ast}}\bbT_{\boldsymbol{\mu}}\bbT_{\bbv}
          +
          (\bbT_{\bbu^{*}}-\bbT_{\bbv^{*}})\bbT_{\boldsymbol{\mu}}
          (\bbT_{\bbu}-\bbT_{\bbv})
          \\
          +
          \bbT_{\bbv^{*}}\bbT_{\boldsymbol{\mu}}
          (\bbT_{\bbu}-\bbT_{\bbv})
          +
          (\bbT_{\bbu^{*}}-\bbT_{\bbv^{*}})\bbT_{\boldsymbol{\mu}}
          \bbT_{\bbv}.
\end{multline}

Then, taking into account that $\bbT_{r}=\bbT_{cr}+\bbP_{r}$ we have that

\begin{multline}\label{eq:app_aux5}
\bbP_{r}=  
(\bbT_{\bbu^{*}}-\bbT_{\bbv^{*}})\bbT_{\boldsymbol{\mu}}
          (\bbT_{\bbu}-\bbT_{\bbv})
          \\
          +
          \bbT_{\bbv^{*}}\bbT_{\boldsymbol{\mu}}
          (\bbT_{\bbu}-\bbT_{\bbv})
          +
          (\bbT_{\bbu^{*}}-\bbT_{\bbv^{*}})\bbT_{\boldsymbol{\mu}}
          \bbT_{\bbv}.
\end{multline}

Computing the norm on both sides of eqn.~(\ref{eq:app_aux5}) and applying the triangular inequality and the operator norm property it follows that

\begin{multline}\label{eq:app_aux6}
\Vert \bbP_{r}\Vert \leq   
\Vert\bbT_{\bbu^{*}}-\bbT_{\bbv^{*}}\Vert\Vert\bbT_{\boldsymbol{\mu}}\Vert
          \Vert\bbT_{\bbu}-\bbT_{\bbv}\Vert
          \\
          +
          \Vert\bbT_{\bbv^{*}}\Vert\Vert\bbT_{\boldsymbol{\mu}}\Vert
          \Vert\bbT_{\bbu}-\bbT_{\bbv}\Vert
          +
          \Vert\bbT_{\bbu^{*}}-\bbT_{\bbv^{*}}\Vert\Vert\bbT_{\boldsymbol{\mu}}\Vert
          \Vert\bbT_{\bbv}\Vert.
\end{multline}

Taking into account the expressions in eqn.~(\ref{app_auxaux1}), eqn.~(\ref{app_auxaux2}), and the fact that $\Vert\bbT_{\bbu^{*}}-\bbT_{\bbv^{*}}\Vert = \Vert\bbT_{\bbu}-\bbT_{\bbv}\Vert$, the eqn.~(\ref{eq:app_aux6}) turns into

\begin{equation}\label{eq:app_aux7}
\Vert \bbP_{r}\Vert \leq   
          \Vert\bbT_{\bbu}-\bbT_{\bbv}\Vert^{2}\Vert\bbT_{r}\Vert
          +
          2\Vert\bbT_{r}\Vert
          \Vert\bbT_{\bbu}-\bbT_{\bbv}\Vert
\end{equation}

which finally can be written as

\begin{equation}\label{eq:app_aux8}
\Vert \bbP_{r}\Vert \leq  
          \Vert\bbT_{r}\Vert
          \left(
          \left(
          \Vert\bbT_{\bbu}-\bbT_{\bbv}\Vert
          +1\right)^{2}
          -1
          \right).
\end{equation}

Now, from the relationship between the Frobenius norm and the $\ell_{2}$-norm we know that

\begin{equation}\label{eq:app_aux9}
 \frac{ \Vert \bbP_{r} \Vert_{F} }{\sqrt{N}} 
                    \leq
                    \Vert\bbP_{r}\Vert
\end{equation}

Combining eqn.~(\ref{eq:app_aux9}) and eqn.~(\ref{eq:app_aux8}) it follows that

\begin{equation}\label{eq:app_aux10}
\Vert \bbP_{r}\Vert_{F} \leq  
          \sqrt{N}
          \left(
          \left(
          \Vert\bbT_{\bbu}-\bbT_{\bbv}\Vert
          +1\right)^{2}
          -1
          \right)\Vert\bbT_{r}\Vert.
\end{equation}

\end{proof}


Notice that the term $\Vert\bbT_{\bbu}-\bbT_{\bbv}\Vert$ is a measure of the difference between the eigenvectors of $\bbS$ and the eigenvectors of $\bbT_{r}$.


\section{Frechet Derivative $D_{p\vert\mathbf{S}_{i}}(\mathbf{S})$}
\label{appendix:howtofindDH}
First, notice that $p(\mathbf{S})=\sum_{k_{1},\ldots,k_{m}=0}^{\infty}h_{k_{1}\ldots k_{m}}\mathbf{S}_{1}^{k_{1}}\ldots\mathbf{S}_{m}^{k_{m}}
=
\sum_{k_{i}=0}^{\infty}\mathbf{S}_{i}^{k_{i}}\mathbf{A}_{k_{i}}$, where $\mathbf{A}_{k_{i}}=\sum_{\substack{\{k_{j}\}=0\\ j\neq i}}^{\infty}h_{k_{1},\ldots,k_{m}}\prod_{\substack{j=1\\j\neq i}}^{m}\mathbf{S}_{j}^{k_{j}}.$
Then, it follows that
\begin{multline}
p(\mathbf{S}+\boldsymbol{\xi})-p(\mathbf{S})=
\sum_{k_{i}=0}^{\infty}\left(\mathbf{S}_{i}+\boldsymbol{\xi}_{i}\right)^{k_{i}}\mathbf{A}_{k_{i}}-
\sum_{k_{i}=0}^{\infty}\mathbf{S}_{i}^{k_{i}}\mathbf{A}_{k_{i}}
\label{eq:howtofingDH1}
\end{multline}
for $\boldsymbol{\xi}=(\mathbf{0},\ldots,\boldsymbol{\xi}_{i},\ldots,\mathbf{0})$. Considering the expansion $(\mathbf{S}_{i}+\boldsymbol{\xi}_{i})^{k_{i}}=\mathbf{S}^{k_{i}}_{i}+\boldsymbol{\xi}_{i}^{k_{i}}+\sum_{r=1}^{k-1}\boldsymbol{\pi}_{r,k_{i}-r}(\mathbf{S}_{i},\boldsymbol{\xi}_{i})$ for $k_{i}\geq 2$, eqn.~(\ref{eq:howtofingDH1}) takes the form
\begin{multline}
p(\mathbf{S}+\boldsymbol{\xi})-p(\mathbf{S})
                                                                         =
\\
\sum_{k_{i}=1}^{\infty}\sum_{r=1}^{k_{i}-1}\boldsymbol{\pi}_{r,k_{i}-r}\left(\boldsymbol{\xi}_{i},\mathbf{S}_{i}\right)\mathbf{A}_{k_{i}}
                                                                 +
\sum_{k_{i}=1}^{\infty}\boldsymbol{\xi}_{i}^{k_{i}}\mathbf{A}_{k_{i}}.
\label{eq:howtofingDH2}
\end{multline}
Separating the linear terms on $\boldsymbol{\xi}_{i}$ eqn.~(\ref{eq:howtofingDH2}) leads to
\begin{multline}
p(\mathbf{S}+\boldsymbol{\xi})-p(\mathbf{S})
                                                                          =
\sum_{k_{i}=1}^{\infty}\boldsymbol{\pi}_{1,k_{i}-1}\left(\boldsymbol{\xi}_{i},\mathbf{S}_{i}\right)\mathbf{A}_{k_{i}}
\\
+
\sum_{k_{i}=2}^{\infty}\sum_{r=2}^{k_{i}-1}\boldsymbol{\pi}_{r,k_{i}-r}\left(\boldsymbol{\xi}_{i},\mathbf{S}_{i}\right)\mathbf{A}_{k_{i}}
+
\sum_{k_{i}=2}^{\infty}\boldsymbol{\xi}^{k_{i}}\mathbf{A}_{k_{i}}.
\end{multline}
Therefore, taking into account the definition of Fr\'echet derivative (see Section~\ref{theorem:HvsFrechet}) it follows that
\begin{equation}
D_{p\vert\mathbf{S}_{i}}(\mathbf{S})\left\lbrace\boldsymbol{\xi}_{i}\right\rbrace
                                                                                                                                 =
\sum_{k_{i}=1}^{\infty}\boldsymbol{\pi}_{1,k_{i}-1}\left(\boldsymbol{\xi}_{i},\mathbf{S}_{i}\right)\mathbf{A}_{k_{i}}
\end{equation}
%
%
%


\section{Proof of Theorems: Extended Version}\label{sec:proofofTheoremsEXT}

In this appendix we provide a more detailed version of the proofs stated in Section~\ref{sec:proofofTheorems} made in a concise way due to the IEEE publication page limit. In particular, we show more detailed proofs of Theorem~\ref{theorem:uppboundDH} and Theorem~\ref{theorem:uppboundDHmultg}.

\subsection{Proof of Theorem~\ref{theorem:uppboundDH} and Theorem~\ref{theorem:uppboundDHmultg}}


\subsubsection{Proof of Theorem~\ref{theorem:uppboundDH}}

\begin{proof}

	Taking into account the definition of the Fr\'echet derivative of $p$ on $\mathbf{S}$ (see Appendix~\ref{appendix:howtofindDH}) we have
	\begin{equation}
	\left\Vert D_{p\vert\mathbf{S}}(\mathbf{S})\left\lbrace\mathbf{T}(\mathbf{S})\right\rbrace\right\Vert
	=
	\left\Vert\sum_{k=1}^{\infty}h_{k}\boldsymbol{\pi}_{1,k-1}\left(\mathbf{T}(\mathbf{S}),\mathbf{S}\right)
	\right\Vert, 
	\end{equation}
	and re-organizating terms we have
	\begin{equation}
	\left\Vert D_{p\vert\mathbf{S}}(\mathbf{S})\left\lbrace\mathbf{T}(\mathbf{S})\right\rbrace\right\Vert=\left\Vert\sum_{\ell=1}^{\infty}\mathbf{S}^{\ell-1}\mathbf{T}(\mathbf{S})\sum_{k=\ell}^{\infty}h_{k}\mathbf{S}^{k-\ell}\right\Vert.
	\end{equation}
	Taking into account eqn.~(\ref{eq:PrvsTr}), it follows that
	\begin{multline}
	\left\Vert D_{p\vert\mathbf{S}}(\mathbf{S})\left\lbrace\mathbf{T}(\mathbf{S})\right\rbrace\right\Vert=
	\\
	\left\Vert\sum_{\ell=1}^{\infty}\left(\mathbf{T}_{0c}\mathbf{S}^{\ell-1}+\mathbf{S}^{\ell-1}\mathbf{P}_{0}\right)\sum_{k=\ell}^{\infty}h_{k}\mathbf{S}^{k-\ell}\right.
	\\
	\left.+\sum_{\ell=1}^{\infty}\left(\mathbf{T}_{1c}\mathbf{S}^{\ell}
	+\mathbf{S}^{\ell-1}\mathbf{P}_{1}\mathbf{S}\right)
	\sum_{k=\ell}^{\infty}h_{k}\mathbf{S}^{k-\ell}\right\Vert.
	\label{eq:auxDHT1a}
	\end{multline}
	Applying the triangle inequality and distribuiting the sum we have
	\begin{multline}\label{eq:DpDTboundeqn1a}
	\left\Vert D_{p\vert\mathbf{S}}(\mathbf{S})\left\lbrace\mathbf{T}(\mathbf{S})\right\rbrace\right\Vert
	\leq
	\left\Vert\mathbf{T}_{0c}\sum_{\ell=1}^{\infty}\sum_{k=\ell}^{\infty}\mathbf{S}^{k-1}h_{k}\right\Vert
	\\
	+\left\Vert D_{p\vert\mathbf{S}}(\mathbf{S})\left\lbrace\mathbf{P}_{0}\right\rbrace\right\Vert
	+\left\Vert\mathbf{T}_{1c}\sum_{\ell=1}^{\infty}\sum_{k=\ell}^{\infty}\mathbf{S}^{k}h_{k}\right\Vert
	\\
	+\left\Vert D_{p\vert\mathbf{S}}(\mathbf{S})\left\lbrace\mathbf{P}_{1}\mathbf{S}\right\rbrace\right\Vert
	\end{multline}
	Now, we analyze term by term in eqn.~(\ref{eq:DpDTboundeqn1a}). For the first term we take into account that
	\begin{equation}
	\sum_{\ell=1}^{\infty}\sum_{k=\ell}^{\infty}\mathbf{S}^{k-1}h_{k}=\sum_{k=1}^{\infty}k h_{k}\mathbf{S}^{k-1}    
	\end{equation}
	and we apply the product norm property to obtain

	\begin{equation}\label{eq_thm_uppboundDH0}
	\left\Vert\mathbf{T}_{0c}\sum_{\ell=1}^{\infty}\sum_{k=\ell}^{\infty}\mathbf{S}^{k-1}h_{k}\right\Vert
	\\
	\leq
	\left\Vert
	            \mathbf{T}_{0c}
	\right\Vert
	\left\Vert
	           \sum_{k=1}^{\infty}k h_{k}\mathbf{S}^{k-1}
	\right\Vert
	.
	\end{equation}
	By means of the spectral theorem and the functional calculus of compact normal operators~\cite{conway1994course,aliprantis2002invitation} we have that the eigenvalues of $ \sum_{k=0}^{\infty}kh_{k}\bbS^{k-1}$ are given by $\sum_{k=0}^{\infty}k h_{k}\lambda_{i}^{k-1}=p^{'}(\lambda_{i})$. Therefore
	\begin{equation}
	\left\Vert
 	          \sum_{k=1}^{\infty}k h_{k}\mathbf{S}^{k-1}
	\right\Vert
	\\
	=
	\max_{i}
	\left\vert
	           \sum_{k=1}^{\infty}k h_{k}\lambda_{i}^{k-1}
	\right\vert
	.
	\end{equation}
     Since $p(\lambda)$ belongs to $\ccalA_{L_0}$ we have $\vert p^{'}(\lambda_{i}) \vert \leq L_0$. Then, taking into account that $\Vert \bbT_{0c}\Vert = \Vert \bbT_0\Vert$ eqn.~\eqref{eq_thm_uppboundDH0} leads to
	\begin{equation}
	\left\Vert
	        \mathbf{T}_{0c}\sum_{\ell=1}^{\infty}\sum_{k=\ell}^{\infty}\mathbf{S}^{k-1}h_{k}
	\right\Vert
	\\
	\leq
    L_{0}\Vert\mathbf{T}_{0}\Vert.
	\end{equation}

	For the second term in eqn.~(\ref{eq:DpDTboundeqn1a}) we start pointing out that
	the operator $ 
	D_{p\vert\mathbf{S}}(\mathbf{S})\left\lbrace\mathbf{P}_{0}\right\rbrace
	$ acting on $\bbP_{0}$ is  an operator in the space of Endomorphisms of $\text{End}({\ccalM})$, i.e. $D_{p\vert\mathbf{S}}(\mathbf{S})\left\lbrace\cdot\right\rbrace\in\text{End}(\text{End}(\ccalM))$. An eigen-pair $(\zeta,\bbV)$ of $D_{p\vert\mathbf{S}}(\mathbf{S})\left\lbrace\cdot\right\rbrace$ is composed of a scalar $\zeta$ (the eigenvalue) and a nonzero operator $\bbV\in\text{End}(\ccalM)$ (the eigenvector) with $D_{p\vert\mathbf{S}}(\mathbf{S})\left\lbrace\bbV\right\rbrace =\zeta\bbV$. If we consider the operator 
    $D_{p\vert\mathbf{S}}(\mathbf{S})\left\lbrace\cdot\right\rbrace$ inside the class of Hilbert-Schmidt operators\footnote{An operator is said to be Hilbert-Schmidt if its Frobenius or Hilbert-Schmidt norm is finite.}, which is \textbf{isometric-isomorphic} to $\text{End}(\ccalM)^{\ast} \otimes \text{End}(\ccalM)$~\cite{conway1994course}, there is a unique linear operator $\overline{D}_{p\vert\mathbf{S}}(\mathbf{S})
    \left\lbrace 
    \cdot
    \right\rbrace \in \text{End}(\ccalM)^{\ast} \otimes \text{End}(\ccalM)$ with $\Vert \overline{D}_{p\vert\mathbf{S}}(\mathbf{S})
    \left\lbrace 
    \cdot
    \right\rbrace\Vert = \Vert D_{p\vert\mathbf{S}}(\mathbf{S})\left\lbrace\cdot \right\rbrace \Vert $. This operator is given by~\cite{higham2008functions}
	\begin{equation}
	\overline{D}_{p\vert\mathbf{S}}(\mathbf{S})
	                             \left\lbrace 
	                                           \cdot
	                             \right\rbrace
	             =  
	             \left(
	 \sum_{k=1}^{\infty}h_{k} 
	                      \sum_{j=1}^{k}\left( \bbS^{k-j} \right)^{\ast} \otimes \bbS^{j-1}
	             \right)   
	             \left\lbrace 
	                        \cdot
	             \right\rbrace 
	             ,                 
	\end{equation}
	and it acts on elements of a space $\ccalW\cong \text{End}(\ccalM)$. Notice that an eigenvector of the operator $ \left( \bbS^{k-j} \right)^{\ast} \otimes \bbS^{j-1}$ is indeed given by $\bbv_{r}^{\ast}\otimes\bbv_s$, where $\bbv_{r}^{\ast}$ is an eigenvector of $\left( \bbS^{k-j} \right)^{\ast}$ and $\bbv_s$ is an eigenvector of $\bbS^{j-1}$. With this at hand, we point out that by means of the spectral theorem, we also have that
	\begin{equation}
	        \left( \bbS^{k-j} \right)^{\ast} \otimes \bbS^{j-1}
	        =
	        \sum_{r,s=1}^{\infty}\lambda_{r}^{k-j}\lambda_{s}^{j-1}
	                   \left( 
	                           \bbT_{\bbv_r^{\ast}}\otimes\bbT_{\bbv_s}
	                   \right)
	                   ,
	\end{equation}
     where $\bbT_{\bbv_r^{\ast}}$ and $\bbT_{\bbv_s}$ are the projection operators associated to the eigenvectors of $\bbS^{\ast}$ and $\bbS$ respectively. Then, it follows
      \begin{equation}
      \overline{D}_{p\vert\mathbf{S}}(\mathbf{S})
      \left\lbrace 
      \cdot
      \right\rbrace
      =  
      \left(
             \sum_{k=1}^{\infty}h_{k} 
             \sum_{j=1}^{k}
                   \sum_{r,s=1}^{\infty}\lambda_{r}^{k-j}\lambda_{s}^{j-1}
                          \left( 
                                \bbT_{\bbv_r^{\ast}}\otimes\bbT_{\bbv_s}
                          \right)
      \right)   
      \left\lbrace 
      \cdot
      \right\rbrace 
      ,                 
      \end{equation}
     and consequently the eigenvalues of $\overline{D}_{p\vert\mathbf{S}}(\mathbf{S})
     \left\lbrace 
     \cdot
     \right\rbrace$ are given by
     \begin{equation}
                    \zeta_{r,s}
                              =
                    \sum_{k=1}^{\infty}
                    h_{k}
                    \sum_{j=1}^{k}\lambda_{r}^{k-j}\lambda_{s}^{j-1}
                    =\left\lbrace
                    \begin{array}{ccc}
                    \frac{p(\lambda_{r})-p(\lambda_{s})}{\lambda_{r}-\lambda_{s}} & \text{if} & \lambda_{r}\neq\lambda_{s}\\
                    p^{'}(\lambda_{r}) & \text{if} & \lambda_{r}=\lambda_{s}
                    .
                    \end{array}
                    \right. 
     \end{equation}
     Since $p(\lambda)$ is $L_0$-Lipschitz it follows that $     \left\vert 
     \zeta_{r,s}
     \right\vert
     \leq 
     L_{0}$ and therefore
	\begin{equation}
		\Vert 
		D_{p\vert\mathbf{S}}(\mathbf{S})
		\left\lbrace
		\cdot
		\right\rbrace 
		\Vert
	=	
	\Vert 
	           \overline{D}_{p\vert\mathbf{S}}(\mathbf{S})
	                              \left\lbrace 
	                                           \cdot
	                              \right\rbrace
	\Vert 
	\leq 
	L_0
	.
	\end{equation}
Now, taking into account that~\cite{conway1994course}  $\Vert \bbA \Vert \leq \Vert \bbA\Vert_{F}$ for any bounded operator $\bbA$ we have
\begin{equation}
\Vert 
		D_{p\vert\mathbf{S}}(\mathbf{S})
		         \left\lbrace
		                \bbP_0
		         \right\rbrace 
\Vert
		\leq
\Vert 
       D_{p\vert\mathbf{S}}(\mathbf{S})
                  \left\lbrace
                             \bbP_0
                  \right\rbrace 
\Vert_{F}	
,	
\end{equation}	
and taking into account that~\cite{conway1994course} (p. 267)
\begin{equation}
\Vert 
D_{p\vert\mathbf{S}}(\mathbf{S})
\left\lbrace
\bbP_0
\right\rbrace 
\Vert_{F}
\leq
\Vert 
D_{p\vert\mathbf{S}}(\mathbf{S})
\Vert
\Vert
\bbP_0
\Vert_{F}	
,
\end{equation}	
we have
	\begin{equation}
	\left\Vert 
	          D_{p\vert\mathbf{S}}(\mathbf{S})\left\lbrace\mathbf{P}_{0}\right\rbrace
	\right\Vert
	\leq 
	L_{0}
	     \Vert
	             \mathbf{P}_{0}
	     \Vert_{F}  
	     ,  
	\end{equation}
	and with the fact that $\Vert\mathbf{P}_{0}\Vert_{F}\leq\delta\Vert\mathbf{T}_{0}\Vert$, it follows
	\begin{equation}
	\left\Vert D_{p\vert\mathbf{S}}(\mathbf{S})\left\lbrace\mathbf{P}_{0}\right\rbrace\right\Vert \leq L_{0}\delta\Vert\mathbf{T}_{0}\Vert.    
	\end{equation}

	For the third term in eqn.~(\ref{eq:DpDTboundeqn1a}), we take into account that 

	\begin{equation}
	\sum_{\ell=1}^{\infty}\sum_{k=\ell}^{\infty}\mathbf{S}^{k} h_{k}=\sum_{k=1}^{\infty}k h_{k}\mathbf{S}^{k},    
	\end{equation}
	and we apply the norm product property  to obtain
	\begin{equation}
	\left\Vert
	         \mathbf{T}_{1c}\sum_{\ell=1}^{\infty}\sum_{k=\ell}^{\infty}\mathbf{S}^{k}h_{k}
	\right\Vert
	\\
	\leq
	\left\Vert
	       \mathbf{T}_{1c}
	\right\Vert
	\left\Vert
	         \sum_{k=1}^{\infty}k h_{k}\mathbf{S}^{k}
	\right\Vert
	.
	\end{equation}
	As a consequence of the spectral theorem, the eigenvalues of $\sum_{k=1}^{\infty}k h_{k}\mathbf{S}^{k}$ are given by $\sum_{k=1}^{\infty}kh_{k}\lambda_{i}^{k}=\lambda_i p^{'}(\lambda_i)$ and therefore
	\begin{equation}
	    	\left\Vert
	    	\sum_{k=1}^{\infty}k h_{k}\mathbf{S}^{k}
	    	\right\Vert
	    	=
	    	\max_{i}
	    	\left\vert
	    	       \sum_{k=1}^{\infty}kh_{k}\lambda_{i}^{k}
	    	\right\vert
	    	.
	\end{equation}
	Since $p(\lambda)$ belongs to $\mathcal{A}_{L_{1}}$ we have $\lambda_i p^{'}(\lambda_i)\leq L_1$. Then, taking into account that $\Vert \bbT_{1c}\Vert = \Vert \bbT_{1}\Vert$, it follows that 
	\begin{equation}
	\left\Vert\mathbf{T}_{1c}\sum_{\ell=1}^{\infty}\sum_{k=\ell}^{\infty}\mathbf{S}^{k}h_{k}\right\Vert
	\leq 
	L_{1}
	\Vert
	      \mathbf{T}_{1}
    \Vert.
	\end{equation}

	Finally, for the fourth term we use the notation $\tilde{D}(\mathbf{S})\left\lbrace\mathbf{P}_{1}
	\right\rbrace
	=D_{p\vert\mathbf{S}}(\mathbf{S})\left\lbrace\mathbf{P}_{1}\mathbf{S}\right\rbrace$. We start pointing out that $\tilde{D}(\mathbf{S})\left\lbrace\mathbf{P}_{1}
	\right\rbrace$ is an operator in $\text{End}(\text{End}(\ccalM))$. Considering $\tilde{D}(\mathbf{S})\left\lbrace\cdot
	\right\rbrace$ as a Hilbert-Schmidt operator (finite Frobenius norm), we have that there is a unique linear operator $\overline{\tilde{D}}(\mathbf{S})\left\lbrace\cdot
	\right\rbrace \in \text{End}(\ccalM)^{\ast}\otimes\text{End}(\ccalM)$ with $\Vert \overline{\tilde{D}}(\mathbf{S})\left\lbrace\cdot
	\right\rbrace\Vert = \Vert\tilde{D}(\mathbf{S})\left\lbrace\cdot
	\right\rbrace\Vert$~\cite{conway1994course}. This operator is given by~\cite{higham2008functions}
	\begin{equation}
	\overline{\tilde{D}}_{p\vert\mathbf{S}}(\mathbf{S})
	\left\lbrace 
	\cdot
	\right\rbrace
	=  
	\left(
	\sum_{k=1}^{\infty}h_{k} 
	\sum_{j=1}^{k}\left( \bbS^{k-j+1} \right)^{\ast} \otimes \bbS^{j-1}
	\right)   
	\left\lbrace 
	\cdot
	\right\rbrace 
	,                 
	\end{equation}
   and it acts on a space $\ccalW \cong \text{End}(\ccalM)$. The eigenvectors of the operator $ \left( \bbS^{k-j} \right)^{\ast} \otimes \bbS^{j-1}$ are given by $\bbv_{r}^{\ast}\otimes\bbv_s$, where $\bbv_{r}^{\ast}$ is an eigenvector of $\left( \bbS^{k-j+1} \right)^{\ast}$ and $\bbv_s$ is an eigenvector of $\bbS^{j-1}$. Now, using the spectral theorem we also have
	\begin{equation}
	\left( \bbS^{k-j+1} \right)^{\ast} \otimes \bbS^{j-1}
	=
	\sum_{r,s=1}^{\infty}\lambda_{r}^{k-j+1}\lambda_{s}^{j-1}
	\left( 
	\bbT_{\bbv_r^{\ast}}\otimes\bbT_{\bbv_s}
	\right)
	,
	\end{equation}
	where $\bbT_{\bbv_r^{\ast}}$ and $\bbT_{\bbv_s}$ are the projection operators associated to the eigenvectors of $\bbS^{\ast}$ and $\bbS$ respectively. Then, it follows
	\begin{equation}
	\overline{\tilde{D}}_{p\vert\mathbf{S}}(\mathbf{S})
	\left\lbrace 
	\cdot
	\right\rbrace
	=  
	\left(
	\sum_{k=1}^{\infty}h_{k} 
	\sum_{j=1}^{k}
	\sum_{r,s=1}^{\infty}\lambda_{r}^{k-j+1}\lambda_{s}^{j-1}
	\left( 
	\bbT_{\bbv_r^{\ast}}\otimes\bbT_{\bbv_s}
	\right)
	\right)   
	\left\lbrace 
	\cdot
	\right\rbrace 
	,                 
	\end{equation}
	and consequently the eigenvalues of $\overline{\tilde{D}}_{p\vert\mathbf{S}}(\mathbf{S})
	\left\lbrace 
	\cdot
	\right\rbrace$ are given by
	\begin{equation}
	\zeta_{rs}
	=
	\sum_{k=1}^{\infty}
	h_{k}
	\sum_{j=1}^{k}\lambda_{r}^{k-j+1}\lambda_{s}^{j-1}
	=
	\left\lbrace
	\begin{array}{ccc}
	\frac{p(\lambda_{r})-p(\lambda_{s})}{\lambda_{r}-\lambda_{s}}\lambda_{r} & \text{if} & \lambda_{r}\neq\lambda_{s}\\
	\lambda_{r}p^{'}(\lambda_{r}) & \text{if} & \lambda_{r}=\lambda_{s}
	\end{array}
	\right. 
	.
	\end{equation}
	Since $p(\lambda)$ belongs to $\ccalA_{L_1}$ it follows that $\vert \zeta_{r,s}\vert\leq L_1$ and therefore
		\begin{equation}
		\left\Vert 
		\tilde{D}_{p\vert\mathbf{S}}(\mathbf{S})
		\left\lbrace
		\cdot
		\right\rbrace 
		\right\Vert
		=	
		\left\Vert 
		\overline{\tilde{D}}_{p\vert\mathbf{S}}(\mathbf{S})
		\left\lbrace 
		\cdot
		\right\rbrace
		\right\Vert 
		\leq 
		L_1
		.
		\end{equation}

Since~\cite{conway1994course} $\Vert \bbA \Vert \leq \Vert \bbA\Vert_{F}$ for any bounded operator $\bbA$ we have
\begin{equation}
\Vert 
\tilde{D}_{p\vert\mathbf{S}}(\mathbf{S})
\left\lbrace
\bbP_1
\right\rbrace 
\Vert
\leq
\Vert 
\tilde{D}_{p\vert\mathbf{S}}(\mathbf{S})
\left\lbrace
\bbP_1
\right\rbrace 
\Vert_{F}	
.	
\end{equation}	
Taking into account that~\cite{conway1994course} (p. 267)
\begin{equation}
\Vert 
\tilde{D}_{p\vert\mathbf{S}}(\mathbf{S})
\left\lbrace
\bbP_1
\right\rbrace 
\Vert_{F}
\leq
\Vert 
\tilde{D}_{p\vert\mathbf{S}}(\mathbf{S})
\Vert
\Vert
\bbP_1
\Vert_{F}	
,
\end{equation}	
we have
\begin{equation}
\left\Vert 
\tilde{D}_{p\vert\mathbf{S}}(\mathbf{S})\left\lbrace\mathbf{P}_{1}\right\rbrace
\right\Vert
\leq 
L_{1}
\Vert
\mathbf{P}_{1}
\Vert_{F}  
,  
\end{equation}
and with the fact that $\Vert\mathbf{P}_{1}\Vert_{F}\leq\delta\Vert\mathbf{T}_{1}\Vert$, it follows
\begin{equation}
\left\Vert D_{p\vert\mathbf{S}}(\mathbf{S})\left\lbrace\mathbf{P}_{1}\bbS\right\rbrace\right\Vert \leq L_{1}\delta\Vert\mathbf{T}_{1}\Vert.    
\end{equation}
	Putting all these results together into eqn.~(\ref{eq:DpDTboundeqn1a}) we reach
	\begin{equation}
	\left\Vert D_{p\vert\mathbf{S}}(\mathbf{S})\left\lbrace\mathbf{T}(\mathbf{S})\right\rbrace\right\Vert
	\leq
	(1+\delta)L_{0}\Vert\mathbf{T}_{0}\Vert+(1+\delta)L_{1}\Vert\mathbf{T}_{1}\Vert.
	\end{equation}
	Fignally, taking into account that
	\begin{equation}
	\left\Vert 
	               \bbT_{0}
	\right\Vert
	            \leq
	            \sup_{\mathbf{S}\in\mathcal{S}}\Vert\mathbf{T}(\mathbf{S})\Vert
	            ,
	            \quad
	\left\Vert
	                \bbT_1
	\right\Vert    
	               \leq     
	               \sup_{\mathbf{S}\in\mathcal{S}}\Vert D_{\mathbf{T}}(\mathbf{S}),    
	\end{equation}
	it follows that
	\begin{equation}
	\left\Vert D_{p\vert\mathbf{S}}(\mathbf{S})\left\lbrace\mathbf{T}(\mathbf{S})\right\rbrace\right\Vert
	\leq
	(1+\delta)\left(L_{0}\sup_{\mathbf{S}\in\mathcal{S}}\Vert\mathbf{T}(\mathbf{S})\Vert+L_{1}\sup_{\mathbf{S}\in\mathcal{S}}\Vert D_{\mathbf{T}}(\mathbf{S})\Vert\right)
	.
	\end{equation}

\end{proof}


\subsubsection{Proof of Theorem~\ref{theorem:uppboundDHmultg}}

\begin{proof}

Taking into account the definition of the Fr\'echet derivative of $p$ on $\mathbf{S}_{i}$ (see Appendix~\ref{appendix:howtofindDH}) we have
\begin{equation*}
\left\Vert D_{p\vert\mathbf{S}_{i}}(\mathbf{S})\left\lbrace\mathbf{T}(\mathbf{S}_{i})\right\rbrace\right\Vert
                  =
                     \left\Vert\sum_{k_{i}=1}^{\infty}\mathbf{A}_{k_{i}}\boldsymbol{\pi}_{1,k_{i}-1}\left(\mathbf{T}(\mathbf{S}_{i}),\mathbf{S}_{i}\right)
                     \right\Vert, 
\end{equation*}
and re-organizating terms we have
\begin{equation}
\left\Vert D_{p\vert\mathbf{S}_{i}}(\mathbf{S})\left\lbrace\mathbf{T}(\mathbf{S}_{i})\right\rbrace\right\Vert=\left\Vert\sum_{\ell=1}^{\infty}\mathbf{S}_{i}^{\ell-1}\mathbf{T}(\mathbf{S}_{i})\sum_{k_{i}=\ell}^{\infty}\mathbf{A}_{k_{i}}\mathbf{S}_{i}^{k_{i}-\ell}\right\Vert.
\end{equation}
Taking into account eqn.~(\ref{eq:PrvsTr}), it follows that
\begin{multline}
\left\Vert D_{p\vert\mathbf{S}_{i}}(\mathbf{S})\left\lbrace\mathbf{T}(\mathbf{S}_{i})\right\rbrace\right\Vert=
\\
\left\Vert\sum_{\ell=1}^{\infty}\left(\mathbf{T}_{0c,i}\mathbf{S}_{i}^{\ell-1}+\mathbf{S}^{\ell-1}_{i}\mathbf{P}_{0,i}\right)\sum_{k_{i}=\ell}^{\infty}\mathbf{A}_{k_{i}}\mathbf{S}_{i}^{k_{i}-\ell}\right.
\\
\left.+\sum_{\ell=1}^{\infty}\left(\mathbf{T}_{1c,i}\mathbf{S}_{i}^{\ell}
+\mathbf{S}_{i}^{\ell-1}\mathbf{P}_{1,i}\mathbf{S}_{i}\right)
\sum_{k=\ell}^{\infty}\mathbf{A}_{k_{i}}\mathbf{S}^{k_{i}-\ell}\right\Vert.
\label{eq:auxDHT1}
\end{multline}
Applying the triangle inequality and distribuiting the sum we have
\begin{multline}\label{eq:DpDTboundeqn1}
\left\Vert D_{p\vert\mathbf{S}_{i}}(\mathbf{S})\left\lbrace\mathbf{T}(\mathbf{S}_{i})\right\rbrace\right\Vert
\leq
\left\Vert\mathbf{T}_{0c,i}\sum_{\ell=1}^{\infty}\sum_{k_{i}=\ell}^{\infty}\mathbf{S}_{i}^{k_{i}-1}\mathbf{A}_{k_{i}}\right\Vert
\\
+
\left\Vert D_{p\vert\mathbf{S}_{i}}(\mathbf{S})\left\lbrace\mathbf{P}_{0,i}\right\rbrace\right\Vert
+
\left\Vert\mathbf{T}_{1c,i}\sum_{\ell=1}^{\infty}\sum_{k_{i}=\ell}^{\infty}\mathbf{S}^{k_{i}}\mathbf{A}_{k_{i}}\right\Vert
\\
+
\left\Vert D_{p\vert\mathbf{S}_{i}}(\mathbf{S})\left\lbrace\mathbf{P}_{1,i}\mathbf{S}_{i}\right\rbrace\right\Vert
\end{multline}
Now, we analyze term by term in eqn.~(\ref{eq:DpDTboundeqn1}). For the first term we take into account that
\begin{equation}
\sum_{\ell=1}^{\infty}\sum_{k_{i}=\ell}^{\infty}\mathbf{S}_{i}^{k_{i}-1}\mathbf{A}_{k_{i}}
=
\sum_{k_{i}=1}^{\infty}k_{i}\mathbf{A}_{k_{i}}\mathbf{S}_{i}^{k_{i}-1}    
\end{equation}
and we apply the product norm property to obtain

\begin{equation}
\left\Vert\mathbf{T}_{0c,i}\sum_{\ell=1}^{\infty}\sum_{k_{i}=\ell}^{\infty}\mathbf{S}_{i}^{k_{i}-1}\mathbf{A}_{k_{i}}\right\Vert
\\
\leq
\left\Vert
     \mathbf{T}_{0c,i}
\right\Vert
     \left\Vert
         \sum_{k_{i}=1}^{\infty}k_{i}\mathbf{A}_{k_{i}}\mathbf{S}_{i}^{k_{i}-1}
     \right\Vert 
.     
\end{equation}
Since all the shift operators $\bbS_i$ commute, by means of the spectral theorem we have that the $j$th eigenvalue of $\sum_{k_{i}=1}^{\infty}k_{i}\mathbf{A}_{k_{i}}\mathbf{S}_{i}^{k_{i}-1}$ is given by 
\begin{equation}
\sum_{k_{1},\ldots,k_{m}=0}^{\infty} k_i h_{k_{1}\ldots k_{m}}\lambda_{1,j}^{k_{1}}\ldots \lambda_{i,j}^{k_i -1}\ldots\lambda_{m,j}^{k_{m}}
 =
 \left.
 \frac{\partial p}{\partial\lambda_i}
 \right\vert_{\boldsymbol{\lambda}_j}
 ,
\end{equation}
where $\boldsymbol{\lambda}_j = \left( \lambda_{1,j} , \ldots, \lambda_{m,j}\right)$. Therefore 
\begin{multline}
  \left\Vert
  \sum_{k_{i}=1}^{\infty}k_{i}\mathbf{A}_{k_{i}}\mathbf{S}_{i}^{k_{i}-1}
  \right\Vert 
  \\
  =
  \max_{j}
  \left\vert
          \sum_{k_{1},\ldots,k_{m}=0}^{\infty} k_i h_{k_{1}\ldots k_{m}}\lambda_{1,j}^{k_{1}}\ldots \lambda_{i,j}^{k_i -1}\ldots\lambda_{m,j}^{k_{m}}
  \right\vert 
  .
\end{multline}
Since $p$ is $L_0$-Lipschitz we have that $\vert \partial p/\partial\lambda_i \vert \leq L_0$. Then, taking into account that $\Vert \bbT_{0c,i}\Vert = \Vert \bbT_{0,i}\Vert$ we have
\begin{equation}
\left\Vert\mathbf{T}_{0c,i}\sum_{\ell=1}^{\infty}\sum_{k_{i}=\ell}^{\infty}\mathbf{S}_{i}^{k_{i}-1}\mathbf{A}_{k_{i}}\right\Vert
\\
\leq
L_0
\Vert \bbT_{0,i}\Vert
.     
\end{equation}

	For the second term in eqn.~(\ref{eq:DpDTboundeqn1}) we start pointing out that the operator $ 
	 D_{p\vert\mathbf{S}_{i}}(\mathbf{S})\left\lbrace\mathbf{P}_{0,i}\right\rbrace
	$ acting on $\bbP_{0,i}$ is  an operator in the space of Endomorphisms of $\text{End}({\ccalM})$, i.e. $D_{p\vert\mathbf{S}_{i}}(\mathbf{S})\left\lbrace\cdot\right\rbrace\in\text{End}(\text{End}(\ccalM))$. Then, if we consider the operator 
	$D_{p\vert\mathbf{S}_{i}}(\mathbf{S})\left\lbrace\cdot\right\rbrace$ inside the class of Hilbert-Schmidt operators, which is \textbf{isometric-isomorphic} to $\text{End}(\ccalM)^{\ast} \otimes \text{End}(\ccalM)$~\cite{conway1994course}, there is a unique linear operator $\overline{D}_{p\vert\mathbf{S}_{i}}(\mathbf{S})
	\left\lbrace 
	\cdot
	\right\rbrace \in \text{End}(\ccalM)^{\ast} \otimes \text{End}(\ccalM)$ with $\Vert \overline{D}_{p\vert\mathbf{S}_i}(\mathbf{S})
	\left\lbrace 
	\cdot
	\right\rbrace\Vert = \Vert D_{p\vert\mathbf{S}_i}(\mathbf{S})\left\lbrace\cdot \right\rbrace \Vert $. This operator is given by~\cite{higham2008functions}
	\begin{equation}
	\overline{D}_{p\vert\mathbf{S}_i}(\mathbf{S})
	\left\lbrace 
	\cdot
	\right\rbrace
	=  
	\left(
	\sum_{k_i =1}^{\infty}
	\sum_{j=1}^{k_i}\left( \bbA_{k_i} \bbS^{k_i-j}_i \right)^{\ast} \otimes \bbS^{j-1}_i
	\right)   
	\left\lbrace 
	\cdot
	\right\rbrace 
	,                 
	\end{equation}
	and it acts on elements of a space $\ccalW\cong \text{End}(\ccalM)$. Notice that an eigenvector of the operator $ \left( \bbA_{k_i}\bbS^{k-j}_i \right)^{\ast} \otimes \bbS^{j-1}_i$ is given by $\bbv_{r}^{\ast}\otimes\bbv_s$, where $\bbv_{r}^{\ast}$ is an eigenvector of $\left( \bbA_{k_i}\bbS^{k_i -j}_i \right)^{\ast}$ and $\bbv_s$ is an eigenvector of $\bbS^{j-1}_i$. Taking into account the spectral theorem we have
	%
	%
	%
	\begin{equation}
	\overline{D}_{p\vert\mathbf{S}_i}(\mathbf{S})
	\left\lbrace 
	\cdot
	\right\rbrace
	=  
	\left(
	\sum_{k_i =1}^{\infty} 
	\sum_{j=1}^{k_i}
	\sum_{r,s=1}^{\infty}A_{k_i}(r)\lambda_{r,i}^{k_i-j}\lambda_{s,i}^{j-1}
	\left( 
	\bbT_{\bbv_r^{\ast}}\otimes\bbT_{\bbv_s}
	\right)
	\right)   
	\left\lbrace 
	\cdot
	\right\rbrace 
	,                 
	\end{equation}
	where $\bbT_{\bbv_r^{\ast}}$ and $\bbT_{\bbv_s}$ are the projection operators associated to the eigenvectors of $\bbS^{\ast}_i$ and $\bbS_i$ respectively, and where
\begin{equation}
A_{k_{i}}(r)
=
\sum_{\substack{\{k_{\ell}\}=0\\ \ell\neq i}}^{\infty}h_{k_{1},\ldots,k_{m}}\prod_{\substack{\ell=1\\ \ell\neq i}}^{m}
\lambda_{r,\ell}^{k_\ell}
\end{equation} 
	Then, the eigenvalues of $\overline{D}_{p\vert\mathbf{S}_i}(\mathbf{S})
	\left\lbrace 
	\cdot
	\right\rbrace$ are given by
	\begin{multline}
	\zeta_{r,s}
	=
	\sum_{k_i =1}^{\infty}
	A_{k_i}(r)
	\sum_{j=1}^{k_i}\lambda_{r,i}^{k_i-j}\lambda_{s,i}^{j-1}
	\\
	=
	\left\lbrace
	\begin{array}{ccc}
	\frac{p(\lambda_{r,i})-p(\lambda_{s,i})}{\lambda_{r,i}-\lambda_{s,i}} & \text{if} & \lambda_{r,i}\neq\lambda_{s,i}\\
	p^{'}(\lambda_{r,i}) & \text{if} & \lambda_{r,i}=\lambda_{s,i}
	,
	\end{array}
	\right. 
	\end{multline}
	where to simplify the notation we used $p(\lambda_{r,i})$ to denote the function  $ p(\lambda_{r,1},\ldots,\lambda_{r,m}) $ evaluating the $i$th position in $\lambda_{r,i}$ and $p(\lambda_{s,i})$ to denote the evaluation of $p$ in the $i$th position in $\lambda_{s,i}$. The term $p^{'}(\lambda_{r,i}) $ indicates the derivative of $p$ with respect to the variable in the $i$th position of the argument.
	
	Since $p$ is $L_1$-integral Lipschitz it follows that $     \left\vert 
	\zeta_{r,s}
	\right\vert
	\leq 
	L_{1}$ and therefore
	\begin{equation}
	    \Vert \overline{\tilde{D}}_{p\vert\mathbf{S}_i}(\mathbf{S})
	    \left\lbrace 
	    \cdot
	    \right\rbrace\Vert 
	    =
	     \Vert 
	             \tilde{D}_{p\vert\mathbf{S}_i}(\mathbf{S})\left\lbrace\cdot \right\rbrace 
	     \Vert
	             \leq
	             L_1
	             .
	\end{equation}

	Now, taking into account that~\cite{conway1994course}  $\Vert \bbA \Vert \leq \Vert \bbA\Vert_{F}$ for any bounded operator $\bbA$ we have
	\begin{equation}
	\Vert 
	\tilde{D}_{p\vert\mathbf{S}_i}(\mathbf{S})
	\left\lbrace
	\bbP_{0,i} \bbS_{i}
	\right\rbrace 
	\Vert
	\leq
	\Vert 
	D_{p\vert\mathbf{S}_i}(\mathbf{S})
	\left\lbrace
	\bbP_{0,i} \bbS_i
	\right\rbrace 
	\Vert_{F}	
	,	
	\end{equation}	
	and taking into account that~\cite{conway1994course} (p. 267)
	\begin{equation}
	\Vert 
	D_{p\vert\mathbf{S}_i}(\mathbf{S})
	\left\lbrace
	\bbP_{0,i}
	\right\rbrace 
	\Vert_{F}
	\leq
	\Vert 
	D_{p\vert\mathbf{S}_i}(\mathbf{S})
	\Vert
	\Vert
	\bbP_{0,i}
	\Vert_{F}	
	,
	\end{equation}	
	we have
	\begin{equation}
	\Vert 
	D_{p\vert\mathbf{S}_i}(\mathbf{S})
	\left\lbrace
	\bbP_{0,i}
	\right\rbrace 
	\Vert
	\leq
	L_0
	\Vert
	\bbP_{0,i}
	\Vert_{F}	
	.
	\end{equation}	
	Finally, since $\Vert\mathbf{P}_{0,i}\Vert_{F}\leq\delta\Vert\mathbf{T}_{0,i}\Vert$, it follows that
	\begin{equation}
	\left\Vert D_{p\vert\mathbf{S}_{i}}(\mathbf{S})\left\lbrace\mathbf{P}_{0,i}\right\rbrace\right\Vert \leq L_{0}\delta\Vert\mathbf{T}_{0,i}\Vert.    
	\end{equation}

For the third term in eqn.~(\ref{eq:DpDTboundeqn1}), we start taking into account that 
\begin{equation}
\sum_{\ell=1}^{\infty}\sum_{k_{i}=\ell}^{\infty}\mathbf{S}_{i}^{k_{i}}\mathbf{A}_{k_{i}}
=
\sum_{k_{i}=1}^{\infty}k_{i}\mathbf{A}_{k_{i}}\mathbf{S}_{i}^{k_{i}}
.   
\end{equation}
We apply the norm product property to obtain
\begin{equation}
\left\Vert\mathbf{T}_{1c,i}\sum_{\ell=1}^{\infty}\sum_{k_{i}=\ell}^{\infty}\mathbf{S}_{i}^{k_{i}}\mathbf{A}_{k_{i}}\right\Vert
\leq
\left\Vert\mathbf{T}_{1c,i}\right\Vert\left\Vert\sum_{k_{i}=1}^{\infty}k_{i}\mathbf{A}_{k_{i}}\mathbf{S}_{i}^{k_{i}}\right\Vert
.
\end{equation}
Now, taking into account that the eigenvalues of $\sum_{k_{i}=1}^{\infty}k_{i}\mathbf{A}_{k_{i}}\mathbf{S}_{i}^{k_{i}}$ are given by

\begin{equation}
\sum_{k_{1},\ldots,k_{m}=0}^{\infty} k_i h_{k_{1}\ldots k_{m}}\lambda_{1,j}^{k_{1}}\ldots \lambda_{i,j}^{k_i }\ldots\lambda_{m,j}^{k_{m}}
=
\left.
\lambda_i
\frac{\partial p}{\partial\lambda_i}
\right\vert_{\boldsymbol{\lambda}_j}
,
\end{equation}
where $\boldsymbol{\lambda}_j = \left( \lambda_{1,j} , \ldots, \lambda_{m,j}\right)$, we have
\begin{equation}
\left\Vert
         \sum_{k_{i}=1}^{\infty}k_{i}\mathbf{A}_{k_{i}}\mathbf{S}_{i}^{k_{i}}
\right\Vert
       =
       \max_{j}
\left\vert 
         \sum_{k_{1},\ldots,k_{m}=0}^{\infty} k_i h_{k_{1}\ldots k_{m}}\lambda_{1,j}^{k_{1}}\ldots \lambda_{i,j}^{k_i }\ldots\lambda_{m,j}^{k_{m}}
\right\vert     
.  
\end{equation}
Since the filters belong to $\mathcal{A}_{L_{1}}$, it follows that
\begin{equation}
\left\Vert
\sum_{k_{i}=1}^{\infty}k_{i}\mathbf{A}_{k_{i}}\mathbf{S}_{i}^{k_{i}}
\right\Vert
\leq 
L_1
.
\end{equation}
Then, with $\Vert \bbT_{1c,i}\Vert = \Vert \bbT_{1,i}\Vert$ we have
\begin{equation}
\left\Vert\mathbf{T}_{1c,i}\sum_{\ell=1}^{\infty}\sum_{k_{i}=\ell}^{\infty}\mathbf{S}_{i}^{k_{i}}\mathbf{A}_{k_{i}}\right\Vert
\leq
L_1
\left\Vert\mathbf{T}_{1,i}\right\Vert
.
\end{equation}

Finally, for the fourth term we use the notation $\tilde{D}(\mathbf{S})\left\lbrace\mathbf{P}_{1,i}
\right\rbrace
=D_{p\vert\mathbf{S}_{i}}(\mathbf{S})\left\lbrace\mathbf{P}_{1,i}\mathbf{S}_{i}\right\rbrace$. We start pointing out that $\tilde{D}(\mathbf{S})\left\lbrace\mathbf{P}_{1,i}
\right\rbrace$ is an operator in $\text{End}(\text{End}(\ccalM))$. Considering $\tilde{D}(\mathbf{S})\left\lbrace\cdot
\right\rbrace$ as a Hilbert-Schmidt operator (finite Frobenius norm), we have that there is a unique linear operator $\overline{\tilde{D}}(\mathbf{S})\left\lbrace\cdot
\right\rbrace \in \text{End}(\ccalM)^{\ast}\otimes\text{End}(\ccalM)$ with $\Vert \overline{\tilde{D}}(\mathbf{S})\left\lbrace\cdot
\right\rbrace\Vert = \Vert\tilde{D}(\mathbf{S})\left\lbrace\cdot
\right\rbrace\Vert$~\cite{conway1994course}. This operator is given by~\cite{higham2008functions}
\begin{equation}
\overline{\tilde{D}}_{p\vert\mathbf{S}_i}(\mathbf{S})
\left\lbrace 
\cdot
\right\rbrace
=  
\left(
\sum_{k_i =1}^{\infty}
\sum_{j=1}^{k_i}\left(\bbA_{k_i} \bbS^{k_i -j+1}_{i} \right)^{\ast} \otimes \bbS^{j-1}_{i}
\right)   
\left\lbrace 
\cdot
\right\rbrace 
,                 
\end{equation}
and it acts on a space $\ccalW \cong \text{End}(\ccalM)$. The eigenvectors of the operator $ \left( \bbA_{k_i}\bbS^{k_i -j}_{i} \right)^{\ast} \otimes \bbS^{j-1}_{i}$ are given by $\bbv_{r}^{\ast}\otimes\bbv_s$, where $\bbv_{r}^{\ast}$ is an eigenvector of $\left( \bbA_{k_i}\bbS^{k_i-j+1}_{i} \right)^{\ast}$ and $\bbv_s$ is an eigenvector of $\bbS^{j-1}_{i}$. Now, using the spectral theorem we also have
\begin{equation}
\overline{\tilde{D}}_{p\vert\mathbf{S}}(\mathbf{S})
\left\lbrace 
\cdot
\right\rbrace
=  
\left(
\sum_{k_i=1}^{\infty}
\sum_{j=1}^{k_i}
\sum_{r,s=1}^{\infty}
A_{k_i}(r)
\lambda_{r,i}^{k_i-j+1}\lambda_{s,i}^{j-1}
\left( 
\bbT_{\bbv_r^{\ast}}\otimes\bbT_{\bbv_s}
\right)
\right)   
\left\lbrace 
\cdot
\right\rbrace 
,                 
\end{equation}
where $\bbT_{\bbv_r^{\ast}}$ and $\bbT_{\bbv_s}$ are the projection operators associated to the eigenvectors $\bbv_r^{\ast}$ and $\bbv_s$ respectively, and where
	\begin{equation}
	A_{k_{i}}(r)
	=
	\sum_{\substack{\{k_{\ell}\}=0\\ \ell\neq i}}^{\infty}h_{k_{1},\ldots,k_{m}}\prod_{\substack{\ell=1\\ \ell\neq i}}^{m}
	\lambda_{r,\ell}^{k_\ell}
	.
	\end{equation} 
Then, the eigenvalues of $\overline{\tilde{D}}_{p\vert\mathbf{S}_i}(\mathbf{S})
\left\lbrace 
\cdot
\right\rbrace$ are given by
\begin{multline}
\zeta_{r,s}
=
\sum_{k_i =1}^{\infty}
A_{k_i}(r)
\sum_{j=1}^{k_i}\lambda_{r,i}^{k_i-j+1}\lambda_{s,i}^{j-1}
\\
=
\left\lbrace
\begin{array}{ccc}
\frac{p(\lambda_{r,i})-p(\lambda_{s,i})}{\lambda_{r,i}-\lambda_{s,i}}\lambda_{r,i} & \text{if} & \lambda_{r,i}\neq\lambda_{s,i}\\
\lambda_{r,i}p^{'}(\lambda_{r,i}) & \text{if} & \lambda_{r,i}=\lambda_{s,i}
,
\end{array}
\right. 
\end{multline}
where to simplify the notation we used $p(\lambda_{r,i})$ to denote the function  $ p(\lambda_{r,1},\ldots,\lambda_{r,m}) $ evaluating the $i$th position in $\lambda_{r,i}$ and $p(\lambda_{s,i})$ to denote the evaluation of $p$ in the $i$th position in $\lambda_{s,i}$. Additionally, $p^{'}(\lambda_{r,i}) $ indicates the derivative of $p$ with respect to the variable in the $i$th position of the argument.

%
%
%
Since $p$ belongs to $\ccalA_{L_1}$ it follows that $\vert \zeta_{r,s}\vert\leq L_1$ and therefore
\begin{equation}
\left\Vert 
\tilde{D}_{p\vert\mathbf{S}_i}(\mathbf{S})
\left\lbrace
\cdot
\right\rbrace 
\right\Vert
=	
\left\Vert 
\overline{\tilde{D}}_{p\vert\mathbf{S}_i}(\mathbf{S})
\left\lbrace 
\cdot
\right\rbrace
\right\Vert 
\leq 
L_1
.
\end{equation}
Since~\cite{conway1994course} $\Vert \bbA \Vert \leq \Vert \bbA\Vert_{F}$ for any bounded operator $\bbA$ we have
\begin{equation}
\Vert 
\tilde{D}_{p\vert\mathbf{S}_i}(\mathbf{S})
\left\lbrace
\bbP_{1,i}
\right\rbrace 
\Vert
\leq
\Vert 
\tilde{D}_{p\vert\mathbf{S}_i}(\mathbf{S})
\left\lbrace
\bbP_{1,i}
\right\rbrace 
\Vert_{F}	
.	
\end{equation}	
Taking into account that~\cite{conway1994course} (p. 267)
\begin{equation}
\Vert 
\tilde{D}_{p\vert\mathbf{S}_i}(\mathbf{S})
\left\lbrace
\bbP_{1,i}
\right\rbrace 
\Vert_{F}
\leq
\Vert 
\tilde{D}_{p\vert\mathbf{S}_i}(\mathbf{S})
\Vert
\Vert
\bbP_{1,i}
\Vert_{F}	
,
\end{equation}	
we have
\begin{equation}
\left\Vert 
\tilde{D}_{p\vert\mathbf{S}_i}(\mathbf{S})\left\lbrace\mathbf{P}_{1,i}\right\rbrace
\right\Vert
\leq 
L_{1}
\Vert
\mathbf{P}_{1,i}
\Vert_{F}  
,  
\end{equation}
and with the fact that $\Vert\mathbf{P}_{1,i}\Vert_{F}\leq\delta\Vert\mathbf{T}_{1,i}\Vert$, it follows that
\begin{equation}
\left\Vert D_{p\vert\mathbf{S}_i}(\mathbf{S})\left\lbrace\mathbf{P}_{1,i}\bbS_i \right\rbrace\right\Vert \leq L_{1}\delta\Vert\mathbf{T}_{1,i}\Vert.    
\end{equation}

Putting all these results together into eqn.~(\ref{eq:DpDTboundeqn1}) we reach
\begin{equation}
\left\Vert D_{p\vert\mathbf{S}_{i}}(\mathbf{S})\left\lbrace\mathbf{T}(\mathbf{S}_{i})\right\rbrace\right\Vert
\leq
(1+\delta)L_{0}\Vert\mathbf{T}_{0,i}\Vert+(1+\delta)L_{1}\Vert\mathbf{T}_{1,i}\Vert
,
\end{equation}
and taking into account that
\begin{equation}
\Vert\mathbf{T}_{0,i}\Vert
\leq
\sup_{\mathbf{S}_{i}\in\mathcal{S}}\Vert\mathbf{T}(\mathbf{S}_{i})\Vert
,
\quad
\Vert\mathbf{T}_{1,i}\Vert
\leq
\sup_{\mathbf{S}_{i}\in\mathcal{S}}\Vert D_{\mathbf{T}}(\mathbf{S}_{i})\Vert
\end{equation}
we finally have that
\begin{multline}
\left\Vert D_{p\vert\mathbf{S}_{i}}(\mathbf{S})\left\lbrace\mathbf{T}(\mathbf{S}_{i})\right\rbrace\right\Vert
\leq
\\
(1+\delta)\left(L_{0}\sup_{\mathbf{S}_{i}\in\mathcal{S}}\Vert\mathbf{T}(\mathbf{S}_{i})\Vert
+
L_{1}\sup_{\mathbf{S}_{i}\in\mathcal{S}}\Vert D_{\mathbf{T}}(\mathbf{S}_{i})\Vert\right)
\end{multline}

\end{proof}

\bibliography{bibliography}
\bibliographystyle{unsrt}
%
%
\ifCLASSOPTIONcaptionsoff
  \newpage
\fi

\end{document}